%% file: main.tex
\documentclass[10pt]{article} %

\usepackage{etoolbox}
\newcommand{\arxiv}[1]{\iftoggle{iclr}{}{#1}}
\newcommand{\iclr}[1]{\iftoggle{iclr}{#1}{}}
\newtoggle{iclr}
\global\toggletrue{iclr}
\global\togglefalse{iclr}

\iclr{
\PassOptionsToPackage{dvipsnames}{xcolor}
\usepackage{icml2026}
}

\usepackage[utf8]{inputenc} %
\usepackage[T1]{fontenc}    %
\usepackage{url}            %
\usepackage{booktabs}       %
\usepackage{amsfonts}       %
\usepackage{nicefrac}       %
\usepackage{microtype}      %
\usepackage{bm}
\usepackage{bbm}
\usepackage{multirow}
\usepackage{wrapfig}
\usepackage{algorithmicx}
\usepackage{algpseudocode}
\usepackage{mdframed}
\usepackage{subcaption}
\usepackage{enumitem}
\iclr{\setlist[itemize]{noitemsep, topsep=0pt}}
\iclr{\setlist[enumerate]{noitemsep, topsep=0pt}}

\newtoggle{draft}
\togglefalse{draft}

\usepackage{mathrsfs}

\usepackage{algorithm}
\usepackage{verbatim}
\usepackage{float} 

\usepackage{multicol}
\usepackage{siunitx}
\usepackage{colortbl}
\usepackage{xcolor}
\definecolor{hl}{RGB}{245,245,245}
\newcolumntype{H}{>{\columncolor{hl}}} %

\definecolor{sd}{RGB}{232, 233, 251}
\definecolor{fm}{RGB}{252, 232, 216}

\usepackage{setspace}

\usepackage{transparent}

\usepackage{inconsolata}
\usepackage[scaled=.90]{helvet}
\usepackage{xspace}
\usepackage[most]{tcolorbox}
\definecolor{rliableolive}{HTML}{BBCC33}
\tcbset{
  aibox/.style={
    width=\linewidth,
    top=10pt,
    bottom=4pt,
    colback=blue!6!white,
    colframe=black,
    colbacktitle=black,
    enhanced,
    center,
    attach boxed title to top left={yshift=-0.1in,xshift=0.15in},
    boxed title style={boxrule=0pt,colframe=white,},
  }
}
\newtcolorbox{AIbox}[2][]{aibox,title=#2,colback=rliableolive!10!white,#1}
\tcbuselibrary{listings,breakable}
\newtcolorbox{genbox}[2][]{enhanced, breakable,
  title={#2}, halign title=center,
  colback=teal!5, colframe=teal, coltext=black, coltitle=white,
  fonttitle=\bfseries, arc=1mm, boxrule=1pt, boxsep=1pt,
  left=4pt, right=4pt, top=2pt, bottom=2pt, toptitle=3pt, bottomtitle=3pt,
  #1}
\lstdefinestyle{trprose}{language={}, basicstyle=\ttfamily\scriptsize,
  columns=fullflexible, keepspaces=true, breaklines=true, breakindent=0pt,
  aboveskip=2pt, belowskip=2pt}
\lstdefinestyle{trcpp}{language=C++, basicstyle=\ttfamily\scriptsize,
  columns=fullflexible, keepspaces=true, breaklines=true, breakindent=0pt,
  showstringspaces=false, tabsize=2,
  keywordstyle=\color{blue!55!black}\bfseries,
  commentstyle=\color{green!45!black},
  stringstyle=\color{red!55!black},
  morecomment=[l][\color{violet!60!black}]{\#},
  aboveskip=3pt, belowskip=3pt}
\usepackage[most]{tcolorbox}

\tcbset{
  examplebox/.style={
    width=\linewidth,
    colback=gray!5,
    colframe=black,
    coltitle=black,
    fonttitle=\bfseries,
    boxrule=0.4pt,
    arc=3pt,
    left=6pt,
    right=6pt,
    top=6pt,
    bottom=6pt,
    enhanced
  }
}
\tcbuselibrary{breakable}

\usepackage{pifont}

\input{arxiv_style}
\input{macros}

\newcommand{\tailrl}{TailRL}                             % method name
\newcommand{\gtailrl}{g_{\mathrm{TailRL}}^{(N)}}          % finite-sample estimator (no baseline)
\newcommand{\gtailrlhat}{\hat g_{\mathrm{TailRL}}^{(N)}}  % with baseline (mean-centered)
\newcommand{\ind}{\mathbbm{1}}
\renewcommand{\comment}[1]{}  % SR: renew, not new: the verbatim package added in the sync also defines \comment

\iclr{%
\hypersetup{%
  colorlinks=true,%
  citecolor=[RGB]{50,100,170},%
  linkcolor=[RGB]{50,100,170},%
  urlcolor=[RGB]{255,102,178}%
}%
}

\usepackage{etoc}

\iclr{
\usepackage{hyperref}
}

\usepackage{color-edits}
\addauthor{ys}{MidnightBlue}

\arxiv{
\usepackage[final]{showlabels}

}

\usepackage{placeins}
\usepackage{caption}
\makeatletter
\let\OldStatex\Statex
\renewcommand{\Statex}[1][3]{%
  \setlength\@tempdima{\algorithmicindent}%
  \OldStatex\hskip\dimexpr#1\@tempdima\relax}
\makeatother

\usepackage{accents}
\usepackage{wrapfig}
\usepackage{parskip}
\usepackage{graphicx}
\usepackage{tikz}
\usetikzlibrary{decorations.pathreplacing}
\usepackage{varwidth}

\let\oldparagraph\paragraph
\arxiv{\renewcommand{\paragraph}[1]{\oldparagraph{#1.}}}
\iclr{\renewcommand{\paragraph}[1]{\textbf{#1.}}}

\newcommand{\paragraphi}[1]{\par\noindent\emph{#1.}}

\algrenewcommand\algorithmicrequire{\textbf{Input:}}
\algrenewcommand\algorithmicensure{\textbf{Return:}}

\newcommand{\algcomment}[1]{\textcolor{blue!70!black}{\footnotesize{\texttt{\textbf{//
          #1}}}}}

\iclr{
\icmltitlerunning{Tail-Likelihood Reinforcement Learning}
}

\arxiv{
    \title{Tail-Likelihood Reinforcement Learning}
  
    \author{
      \textbf{Shrinivas Ramasubramanian}$^{1}$ \quad
      \textbf{Daman Arora}$^{1}$ \quad
      \textbf{Fahim Tajwar}$^{1}$ \quad
      \textbf{Guanning Zeng}$^{1}$ \\
      \textbf{Qingyang Wu}$^{4}$ \quad
      \textbf{Zhongzhu Zhou}$^{4}$ \quad
      \textbf{Chenfeng Xu}$^{4}$ \\
      \textbf{Haiwen Feng}$^{2,3}$ \quad
      \textbf{Yuda Song}$^{1}$ \quad
      \textbf{Aarti Singh}$^{1}$ \quad
      \textbf{Ruslan Salakhutdinov}$^{1}$ \\
      \textbf{J. Andrew Bagnell}$^{5,1}$ \quad
      \textbf{Jeff Schneider}$^{1,\dagger}$ \quad
      \textbf{Andrea Zanette}$^{1,\dagger}$ \\[2mm]
      {\small
      $^{1}$Carnegie Mellon University \quad
      $^{2}$University of California, Berkeley \quad
      $^{3}$Impossible, Inc. \\
      $^{4}$Together AI \quad
      $^{5}$Aurora Innovation
      } \\
      \vspace{1mm}
      \texttt{\{shrinivr, jeff4, azanette\}@andrew.cmu.edu} \\
    }
}

\hypersetup{
  pdftitle={Tail-Likelihood Reinforcement Learning},
  pdfauthor={Shrinivas Ramasubramanian and Daman Arora and Fahim Tajwar and Guanning Zeng and Qingyang Wu and Zhongzhu Zhou and Chenfeng Xu and Haiwen Feng and Yuda Song and Aarti Singh and Ruslan Salakhutdinov and Drew Bagnell and Jeff Schneider and Andrea Zanette},
  pdfsubject={Reinforcement learning post-training beyond expected reward maximization},
  pdfkeywords={reinforcement learning, RLVR, post-training, Best-of-k, inference-time scaling, tail likelihood},
}

\begin{document}
\etocdepthtag.toc{maintext}

%% ---- arXiv-mode title + optional header logo ----
\arxiv{
\maketitle
\renewcommand{\thefootnote}{}
\footnotetext{$^{\dagger}$Joint advising.}
\footnotetext{Project website, code, and other assets: \url{https://zanette-labs.github.io/TailRL-website/}}
\renewcommand{\thefootnote}{\arabic{footnote}}
% Uncomment the two lines below if you have equal-contribution authors:
% \renewcommand{\thefootnote}{}
% \footnotetext{$^\dagger$ Joint advising.}
% \renewcommand{\thefootnote}{\arabic{footnote}}

% ---- Header logo block (remove this whole block if you don't want a logo) ----
\thispagestyle{fancy}
\fancyhead{}
\lhead{\raisebox{-0.7cm}{\includegraphics[height=0.4cm]{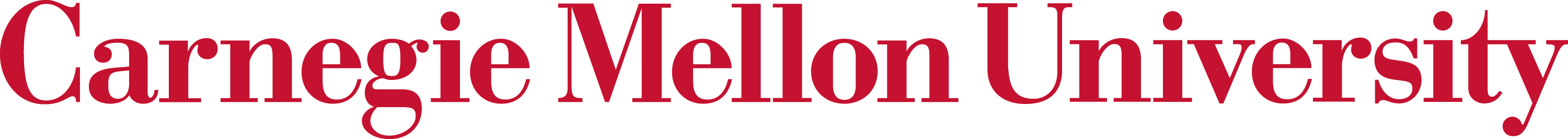}}}
\renewcommand{\headrulewidth}{0pt}
}

%% ---- ICML-mode title + author list ----
\iclr{
\twocolumn[
  \icmltitle{Your Paper Title Here}

  \icmlsetsymbol{equal}{*}

  \begin{icmlauthorlist}
    \icmlauthor{First Author}{yyy}
    \icmlauthor{Second Author}{comp}
  \end{icmlauthorlist}

  \icmlaffiliation{yyy}{Department, University, Location, Country}
  \icmlaffiliation{comp}{Company, Location, Country}

  \icmlcorrespondingauthor{First Author}{first@your.edu}

  \icmlkeywords{Machine Learning, ICML}

  \vskip 0.3in
]
\printAffiliationsAndNotice{}
}

\noindent

\begin{abstract}

Reinforcement learning typically optimizes average reward. For generative policies, the average can hide an important distinction: two policies can achieve the same mean reward while having very different chances of producing a rare but high-reward rollout. This matters as sampling increases during training and inference, since its benefit depends on retaining probability mass on high-reward outcomes.
We propose to optimize this coverage directly. Rather than considering only expected reward, we consider all of its upper tails: for each reward threshold, how likely is the policy to exceed it? This turns a continuous reward into a family of binary success events. We introduce \textbf{Tail-Likelihood Reinforcement Learning (\tailrl{})}, which maximizes the log-probability of exceeding a randomly chosen reward threshold. Its gradient gives more weight to rare, high-reward rollouts and can be interpreted as a mixture of Best-of-$k$ gradients. \tailrl{} requires only a simple modification to the advantage function, making it compatible with existing reinforcement learning pipelines. Across object localization, maze navigation, GUI grounding, and code optimization, TailRL leverages rare high-reward training samples to avoid suboptimal solutions and yields models that benefit more from additional samples at inference time.

\end{abstract}

\section{Introduction}
\label{sec:intro}

Reinforcement learning (RL) typically optimizes the expected reward of a policy
\citep{williams1992,shao2024deepseekmath,ahmadian2024back,yu2025dapo,zheng2025groupsequencepolicyoptimization}.
For generative policies, however, the mean reward does not fully characterize performance: two policies with similar mean reward can have very different probabilities of producing rare but exceptionally good rollouts.
This distinction matters whenever additional samples can be drawn, both during training and at inference time.

Recent work has exposed this problem directly: standard RL training can progressively lose coverage over rare, high-reward rollouts, often visible as a degradation in Best-of-$k$ performance
\citep{cui2025entropy,yue2025doesreinforcementlearningreally,wu2025invisibleleashrlvrescape,dang2025assessing,kirk2024understandingeffectsrlhfllm}.
Once these rollouts become sufficiently unlikely, they are rarely sampled again, making further policy improvement increasingly difficult.
The same loss of coverage limits inference-time scaling: a policy may perform well with a single sample while gaining little from drawing many
\citep{walder2025pkpo,chen2025passktrainingadaptivelybalancing,yang2025depthbreadthsynergyrlvrunlocking}.
Thus, optimizing only the mean reward can discard information about the upper tail of the reward distribution that is crucial for both training and inference scaling.

\begin{figure}[t]
\centering
\includegraphics[width=\textwidth]{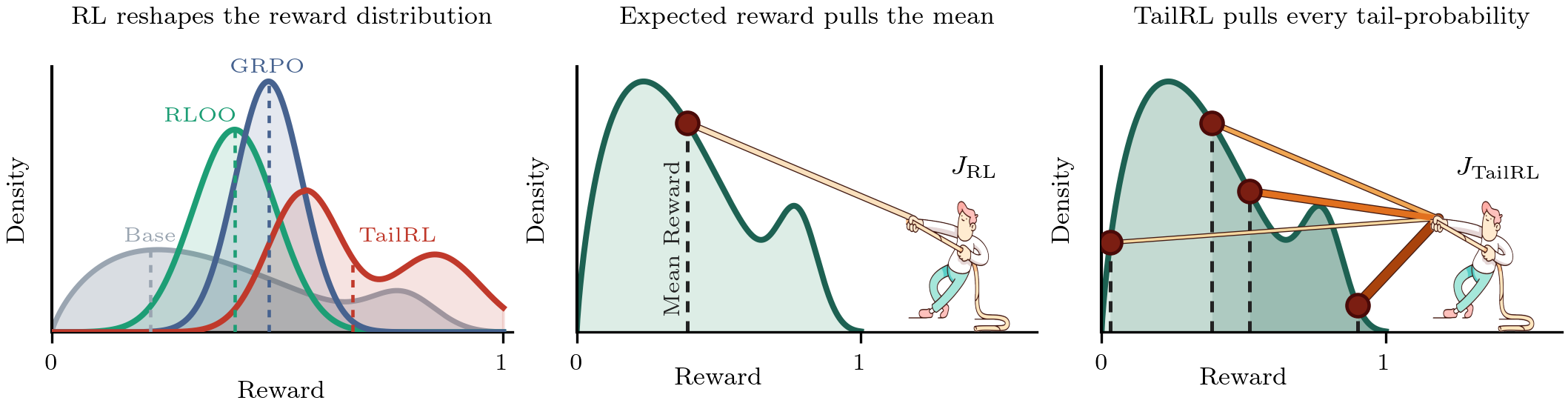}
\caption{\textbf{Expected reward pulls one rope; \tailrl{} pulls them all.}
Left: RL post-training reshapes the reward distribution of the base policy.
Expected-reward methods (GRPO, RLOO) shift the distribution and sharpen it around the mean, while \tailrl{} shifts it further and grows a heavy high-reward tail.
Middle: the expected-reward objective $J_{\mathrm{RL}}$ improves the distribution through a single handle, its mean.
Right: \tailrl{} increases the tail probability $p_\theta(x,\tau)$ across reward thresholds $\tau$, placing greater emphasis on rarer, higher-reward outcomes.}
\label{fig:tug-teaser}
\end{figure}

For binary rewards, MaxRL \citep{tajwar2026maximumlikelihoodreinforcementlearning} offers one way to address this problem.
Rather than maximizing the probability of success, MaxRL maximizes its log-probability, placing greater emphasis on rare successes.
Its gradient decomposes into a harmonic mixture of Pass@$k$ gradients, directly connecting likelihood maximization to a mixture of Pass@$k$ objectives.

How should this principle extend to \textbf{continuous rewards}?
Our starting point is simple: every reward threshold defines a binary event.
Given an input $x$, a rollout $z\sim\pi_\theta(\cdot\mid x)$, and a threshold $\tau$, we ask whether the reward exceeds $\tau$.
The corresponding \emph{tail probability} is
\begin{equation}
p_\theta(x,\tau)
:=
\Pr_{z\sim\pi_\theta(\cdot\mid x)}
\bigl(r(x,z)>\tau\bigr).
\end{equation}
A continuous reward can  be viewed as a family of binary success events, one for every threshold $\tau$. This perspective leads to \textbf{Tail-Likelihood Reinforcement Learning (\tailrl{})}.
For rewards in $[0,1]$, \tailrl{} maximizes the expected log-likelihood of exceeding a uniformly chosen reward threshold:
\begin{equation}
J_{\mathrm{\tailrl{}}}(\theta;x)
=
\int_0^1 \log p_\theta(x,\tau),d\tau.
\label{eq}
\end{equation}
Unlike expected-reward maximization, which acts on a single summary of the reward distribution, \tailrl{} explicitly optimizes upper-tail probabilities across reward levels.
We show that the gradient of \tailrl{} decomposes into a harmonic mixture of Best-of-$k$ gradients (\cref{sec:tailrl-best-of-k-expansion}), directly connecting the training objective to coverage of high-reward rollouts and inference-time scaling.
MaxRL emerges as the binary-reward special case.
Despite its different objective, \tailrl{} admits a simple critic-free policy-gradient estimator that differs from standard RL only in its advantage calculation, allowing it to be implemented by swapping the advantage function in an existing RL pipeline.
Across object localization (\cref{sec:exp-imagenet}), maze navigation (\cref{sec:exp-maze}), GUI grounding (\cref{sec:exp-gui}), and code optimization (\cref{sec:exp-pie}), \tailrl{} leverages rare high-reward samples during training to avoid suboptimal solutions and produces policies that benefit more from additional samples at inference time.

Our contributions are as follows.

\begin{enumerate}
\item \textbf{A likelihood objective for continuous rewards.} \tailrl{} maximizes the log-probability of exceeding a uniformly drawn reward threshold, which weights each reward level by the inverse of how often the policy reaches it. It introduces no threshold or weighting hyperparameter and reduces exactly to MaxRL for binary rewards.
\item \textbf{Alignment with inference-time scaling.} The gradient of \tailrl{} decomposes harmonically over Best-of-$k$ gradients (\cref{sec:tailrl-best-of-k-expansion}),
\[
\nabla_{\theta} J_{\tailrl}(\theta; x) = \sum_{k=1}^{\infty}\frac{1}{k}\,\nabla_\theta\,\text{Best-of-}k(\theta; x),
\]
so \tailrl{} improves Best-of-$k$ at every inference budget without choosing one in advance.
\item \textbf{An unbiased finite-rollout estimator.} A group of $N$ rollouts defines an order-$N$ truncation of \tailrl{} that interpolates from expected reward ($N=1$) to the population objective ($N\to\infty$), and we derive closed-form rollout weights that estimate its gradient without bias (\cref{sec:estimator}). Unlike REINFORCE, where more rollouts only reduce variance, here the rollout budget selects the objective being optimized.
\item \textbf{Strong empirical results.} \tailrl{} matches supervised objectives that observe the ground truth on object localization (\cref{sec:exp-imagenet}), learns from initial policies with $0.01\%$ success where expected-reward baselines fail on maze navigation (\cref{sec:exp-maze}), matches RLOO's Pass@1024 on GUI grounding with $128$--$256\times$ fewer inference rollouts (\cref{sec:exp-gui}), and reaches a $7.7\times$ Best-of-1024 speedup on code optimization where GRPO and RLOO collapse onto copying the input (\cref{sec:exp-pie}).
\end{enumerate}

\section{Preliminaries}
\label{sec:prelim}

We primarily focus on optimizing continuous reward reinforcement learning, where for each input $x$, a rollout $z$ is generated by the policy $\pi_{\theta}$. A rollout may be a generated response, program, or trajectory. A deterministic reward function provides a scalar feedback for an input rollout pair $r(x,z)\in[0,1]$. \cref{app:reward-range} treats a more general bounded reward range.

Training ultimately averages over $x\sim\rho$.
To keep the notation light, we write each objective for a fixed input $x$ and leave the outer average over inputs implicit.
We define the policy's score-function as
$S(x,z):=\nabla_\theta\log\pi_\theta(z\mid x)$. Standard reinforcement learning maximizes  expected reward and the score-function identity \citep{williams1992} gives its policy gradient as follows,
\begin{equation}
\begin{gathered}
J_{\mathrm{RL}}(\theta;x)
:=
\mathbb{E}_{z\sim\pi_\theta(\cdot\mid x)}
\!\left[r(x,z)\right],
\\[3pt]
\nabla_\theta J_{\mathrm{RL}}(\theta;x)
=
\mathbb{E}_{z\sim\pi_\theta(\cdot\mid x)}
\!\left[r(x,z)S(x,z)\right].
\end{gathered}
\label{eq:jrl}
\end{equation}

Thus, the policy gradient is a weighted combination of the score-function and the weights are determined by the rollout's reward. During training, critic-free rollout based policy gradient methods draw $N$ independent rollouts $z_1,\ldots,z_N\sim\pi_\theta(\cdot\mid x)$.
Critic-free methods such as GRPO \citep{shao2024deepseekmath} and RLOO \citep{ahmadian2024back} compute a finite rollout estimate of \cref{eq:jrl}. PKPO \citep{walder2025pkpo} optimizes Pass@$k$ and Best-of-$k$ for binary and continuous rewards using a similar finite-rollout based estimation framework.
\Cref{sec:estimator} shows that \tailrl{} uses the same template and changes only how these advantages are computed, and \cref{app:advantage-comparison} places the three advantage functions side by side.

\subsection{Inference-Time Selection}
\label{sec:prelim-selection}

At deployment, additional inference compute can be used to sample several rollouts and select the one with the highest reward. For $k$ independent rollouts, this performance is measured by probability of at-least one success (Pass@$k$) when rewards are binary and expected maximum reward among (Best-of-$k$) when rewards are  continuous:
\begin{align}
& \text{Pass@}k(\theta;x) := 1 - \left(\Pr_{z\sim\pi_{\theta}(\cdot\mid x)}\bigl(r(x,z)=0\bigr)\right)^{k},
&&
\text{(binary reward)}
\label{eq:passk-bestk}\\
& \text{Best-of-}k(\theta; x) := \mathbb{E}_{\{z_i\}_{i=1}^{k} \sim \pi_{\theta}(\cdot \mid x)}\!\left[\max_{1\le i\le k} r(x,z_i)\right],
&&
\text{(continuous reward)}
\label{eq:bok}
\end{align}
At $k=1$, both reduce to the mean reward, $ J_{\mathrm{RL}}(\theta;x)$. \Cref{app:passk-bestk} gives the estimators we typically use to compute both empirically. Both Pass@$k$ and Best-of-$k$ are non-decreasing in $k$, and as $k\to\infty$ each approaches the maximum reward in their support.

Unlike mean reward, $\text{Best-of-}k$ depends strongly on the upper part of the reward distribution. Two policies with the same mean can scale differently with additional training or inference compute if one assigns more probability to high-reward rollouts.

\subsection{Policy as a Generative Model of Rewards}
\label{sec:prelim-reward-dist}

A policy and reward function together define a distribution over rewards.
For an input $x$, the policy samples a rollout $z$ and the reward function assigns its reward:
\begin{equation}
z\sim\pi_\theta(\cdot\mid x),
\qquad
r=r(x,z)\in[0,1].
\label{eq:reward-generative}
\end{equation}
Because the reward function is deterministic, all randomness in the reward comes from the policy.
The induced reward distribution for any event $A$ ($A\subseteq[0,1]$) over the support of rewards is defined as, 
\begin{equation}
\Pr_{z\sim\pi_\theta(\cdot\mid x)}
\!\left(r(x,z)\in A\right) =  \mathbb{E}_{z\sim\pi_{\theta}(\cdot \mid x)}[\ind_{\{r(x,z) \in A \}}]
\label{eq:reward-mass}
\end{equation}
In this view, the policy is a generative model over rewards, and training manipulates the probability mass over reward values. Expected reward reinforcement learning uses only the mean reward to shape this reward distribution.

Different policies can have the same mean reward while assigning very different probabilities to high-reward outcomes.
This becomes visible in their difference in Best-of-$k$ and Pass@$k$ performance. This motivates objectives that act on the entirety of the reward distribution rather than only its mean. 
We begin with binary rewards, where the distribution is completely determined by a single success probability.

\subsection{MaxRL for Binary Rewards}
\label{sec:prelim-maxrl}

For a binary reward $r(x,z)\in\{0,1\}$,  expected reward equals the probability of success.
\begin{equation}
q_\theta(x)
:=
\Pr_{z\sim\pi_\theta(\cdot\mid x)}
\!\left(r(x,z)=1\right).
\label{eq:binary-p}
\end{equation}
The reward distribution is therefore a Bernoulli distribution over success and failures. The reward distribution is fully determined by $q_\theta(x)$.
Standard reinforcement learning maximizes $q_\theta(x)$, whereas MaxRL \citep{tajwar2026maximumlikelihoodreinforcementlearning} maximizes $J_{\mathrm{MaxRL}}(\theta;x) = \log q_\theta(x)$ and its gradient is,
\begin{equation}
\nabla_\theta J_{\mathrm{MaxRL}}(\theta;x)
=
\frac{1}{q_\theta(x)}
\mathbb{E}_{z\sim\pi_\theta(\cdot\mid x)}
\!\left[
\ind_{\{r(x,z)=1\}}S(x,z)
\right].
\label{eq:maxrl-gradient}
\end{equation}
The factor $1/q_\theta(x)$ gives greater weight to inputs on which success is rare. MaxRL also connects to inference-time sampling.  $\nabla_\theta J_{\mathrm{MaxRL}}(\theta;x)$ decomposes as a harmonic mixture of Pass@$k$ gradients.

\begin{equation}
\nabla_\theta  J_{\mathrm{MaxRL}}(\theta;x)
=
\sum_{k=1}^{\infty}
\frac{1}{k}
\nabla_\theta\mathrm{Pass@}k(\theta; x)
\label{eq:binary-harmonic}
\end{equation}

\section{Tail-Likelihood Reinforcement Learning}
\label{sec:tailrl-objective}

\begin{figure}[t]
\centering
\includegraphics[width=\textwidth]{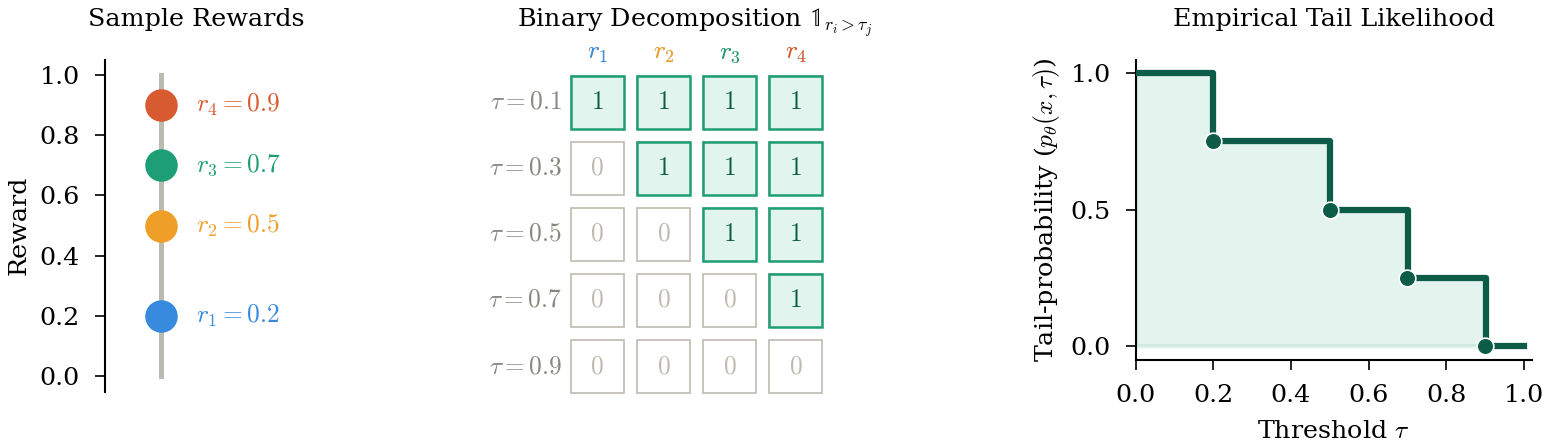}
\caption{A continuous reward decomposes into threshold events.
A rollout with reward $r_i$ clears every threshold below $r_i$.
Across a group of rollouts, these binary outcomes estimate the tail-probability $p_\theta(x,\tau)$ at every reward threshold.
\tailrl{} maximizes the average log-probability along this curve.}
\label{fig:teaser}
\end{figure}

We motivate \tailrl{} by expressing expected reward in terms of the policy's upper-tail probabilities.
For rewards in $[0,1]$, define the tail probability at threshold $\tau$ as
\begin{equation}
p_\theta(x,\tau)
:=
\Pr_{z\sim\pi_\theta(\cdot\mid x)}
\left(r(x,z)>\tau\right).
\label{eq:tailrl-tail-probability}
\end{equation}
The expected reward is exactly the area under this tail-probability curve:
\begin{equation}
J_{\mathrm{RL}}(\theta;x)
=
\mathbb E[r(x,z)]
=
\int_0^1 p_\theta(x,\tau)\,d\tau.
\label{eq:expected-reward-tail}
\end{equation}

\textbf{Tail-Likelihood Reinforcement Learning (\tailrl{})} instead applies the MaxRL likelihood principle at every reward threshold.
It maximizes the average log-probability of exceeding a uniformly sampled threshold:
\begin{equation}
J_{\mathrm{\tailrl{}}}(\theta;x)
:=
\int_0^1 \log p_\theta(x,\tau)\,d\tau
=
\mathbb E_{\tau\sim\mathrm{Unif}[0,1]}
\left[\log p_\theta(x,\tau)\right].
\label{eq:tailrl-objective}
\end{equation}
Equivalently, expected-reward RL aggregates tail probabilities arithmetically, whereas \tailrl{} aggregates them geometrically.
This makes small tail probabilities more influential, as seen directly from its gradient
\begin{equation}
\nabla_\theta J_{\mathrm{\tailrl{}}}(\theta;x)
=
\int_0^1
\frac{1}{p_\theta(x,\tau)}
\nabla_\theta p_\theta(x,\tau)
\,d\tau.
\label{eq:tailrl-threshold-gradient}
\end{equation}
The higher the reward threshold, the smaller the tail-probability and higher is the weight for the gradient. Thus reward thresholds that are difficult to reach receive larger weight during optimization.

\subsection{Harmonic Decomposition over Best-of-\texorpdfstring{$k$}{k}}
\label{sec:tailrl-best-of-k-expansion}

Although \tailrl{} is defined through reward thresholds, it has an exact interpretation in terms of inference-time sampling.
Recall that $\mathrm{Best\text{-}of\text{-}k}(\theta;x)$ is the expected maximum reward among $k$ independent rollouts from $\pi_\theta(\cdot\mid x)$.

\begin{theorem}[Best-of-$k$ decomposition]
\label{thm:tailrl-harmonic-decomposition}
The \tailrl{} objective decomposes as
\begin{equation}
J_{\mathrm{\tailrl{}}}(\theta;x)
=
\sum_{k=1}^{\infty}
\frac{\mathrm{Best\text{-}of\text{-}k}(\theta;x)-1}{k},
\label{eq:tailrl-harmonic-objective-limit}
\end{equation}
and, under the regularity conditions of \cref{lem:tailrl-regularity}, its gradient satisfies
\begin{equation}
\nabla_\theta J_{\mathrm{\tailrl{}}}(\theta;x)
=
\sum_{k=1}^{\infty}
\frac{1}{k}
\nabla_\theta \mathrm{Best\text{-}of\text{-}k}(\theta;x).
\label{eq:tailrl-harmonic-gradient-limit}
\end{equation}
\end{theorem}

Thus, \tailrl{} combines Best-of-$k$ learning signals across all inference budgets.
The harmonic weights $1/k$ arise automatically from the logarithm, so \tailrl{} does not require selecting a target inference budget in advance.
The proof is given in \cref{app:tailrl-harmonic-population}.

\subsection{Recovery of MaxRL for Binary Rewards}
\label{sec:tailrl-binary-recovery}

For binary rewards, every nontrivial reward threshold defines the same success event.
Let
\[
q_\theta(x)
:=
\Pr_{z\sim\pi_\theta(\cdot\mid x)}
\left(r(x,z)=1\right)
\]
denote the probability of success.
For every $\tau\in[0,1)$,
$p_\theta(x,\tau)=q_\theta(x)$, and therefore
\begin{equation}
J_{\mathrm{\tailrl{}}}(\theta;x)
=
\int_0^1 \log q_\theta(x)\,d\tau
=
\log q_\theta(x)
=
J_{\mathrm{MaxRL}}(\theta;x).
\label{eq:tailrl-maxrl-population}
\end{equation}
Moreover, for binary rewards $\mathrm{Best\text{-}of\text{-}k}$ coincides with $\mathrm{Pass@}k$, so the harmonic Best-of-$k$ decomposition above reduces exactly to the harmonic Pass@$k$ decomposition of MaxRL.
Thus, \tailrl{} directly generalizes MaxRL from binary to continuous rewards.

\subsection{Probabilistic Interpretation of \tailrl{}}
\label{sec:tailrl-derivation}

A direct extension of MaxRL would define only $r(x,z)=1$ as success and treat every lower reward as the same failure.
Exact success may be rare or unattainable.
Lowering the success threshold makes the event more common, but still treats all rewards on either side of the threshold as equivalent.
Once a rollout crosses the threshold, the objective has no preference for improving it further.
The resulting policy can perform well at the chosen threshold while remaining poor at higher reward levels (\cref{app:imagenet-binarization}).
We therefore need a likelihood event that preserves the continuous reward signal.

A continuous reward defines such an event at every reward threshold.
For $\tau\in[0,1)$, define the tail-probability
\begin{equation}
p_\theta(x,\tau)
:=
\Pr_{z\sim\pi_\theta(\cdot\mid x)}
\!\left(r(x,z)>\tau\right).
\label{eq:tailrl-tail-probability-restated}
\end{equation}

\paragraph{Independent quality audit}
We combine the tail events by assigning each reward threshold an independent rollout.
The audit passes only if every rollout clears its assigned threshold.
Independent rollouts are necessary: reusing one rollout would collapse the nested events to the hardest threshold.

For an audit with $L$ equally spaced thresholds, let
$\tau_\ell=(\ell-1)/L$ and draw
$z_1,\ldots,z_L\overset{\mathrm{i.i.d.}}{\sim}\pi_\theta(\cdot\mid x)$.
Define the event that the policy passes the audit as
\begin{equation}
E_L
:=
\bigcap_{\ell=1}^{L}
\left\{
r(x,z_\ell)>\tau_\ell
\right\}.
\label{eq:tailrl-audit-event}
\end{equation}
Because the rollouts are independent, the audit probability factorizes:
\begin{equation}
\Pr_{z_1...z_L \sim \pi_\theta(\cdot\mid x)}(E_L\mid x)
=
\prod_{\ell=1}^{L}
p_\theta(x,\tau_\ell).
\label{eq:tailrl-audit-probability}
\end{equation}
The number of thresholds controls only the resolution of the audit. Requiring invariance to the change in scale of the objective due to increasing resolution of the audit, while agreeing with ordinary log-likelihood when $L=1$, uniquely gives the normalization $1/L$:
\begin{equation}
\frac{1}{L}
\log\Pr_\theta(E_L\mid x)
=
\frac{1}{L}
\sum_{\ell=1}^{L}
\log p_\theta(x,\tau_\ell).
\label{eq:tailrl-audit-rate}
\end{equation}

Each midpoint $\tau_\ell$ represents an interval of width $1/L$.
The right-hand side of \cref{eq:tailrl-audit-rate} is therefore a Riemann sum over the reward range.
Letting $L\to\infty$ gives
\begin{equation}
J_{\mathrm{\tailrl{}}}(\theta;x)
:=
\lim_{L\to\infty}
\frac{1}{L}
\log\Pr_\theta(E_L\mid x)
=
\int_0^1
\log p_\theta(x,\tau)
\,d\tau,
\label{eq:tailrl-objective-audit-limit}
\end{equation}

We call \cref{eq:tailrl-objective} the population-level \tailrl{} objective.
Equivalently, \tailrl{} maximizes the expected log-probability of clearing a uniformly sampled reward threshold. 

\tailrl{} uses the full support of the reward distribution rather than selecting an arbitrary cut-off.  Uniform sampling of threshold $\tau$ assigns equal importance to the log-likelihood of their tail-events. It introduces no weighing hyperparameter.
More generally, one could sample $\tau$ from a non-uniform distribution inducing a re-weighting of the log-likelihood terms.
Any strictly positive normalized weighting is exactly equivalent to applying  \tailrl{} after a monotone transformation of the rewards (\cref{prop:tailrl-threshold-reparameterization}).
We use the uniform distribution throughout and leave task-specific reward shaping to future work.
The audit defines the population objective; it does not yet prescribe the finite-rollout estimator used for training (\cref{sec:estimator}).

\section{Estimating the \tailrl{} Gradient}
\label{sec:estimator}

Training observes only a finite group of rollouts during training. The population-level objective is inestimable from finite rollouts.
We show that a rollout budget of $N$ naturally defines an order-$N$ truncation of the \tailrl{} objective and admits a simple unbiased policy-gradient estimator.

\subsection{From Finite Rollouts to a Finite-Order Objective}
\label{sec:tailrl-finite-order}

For an order $T\geq 1$, define
\begin{equation}
J_{\mathrm{\tailrl{}}}^{(T)}(\theta;x)
:=
\sum_{k=1}^{T}
\frac{\mathrm{Best\text{-}of\text{-}k}(\theta;x)-1}{k}.
\label{eq:tailrl-harmonic-family}
\end{equation}
Its gradient is
\begin{equation}
\nabla_\theta J_{\mathrm{\tailrl{}}}^{(T)}(\theta;x)
=
\sum_{k=1}^{T}
\frac{1}{k}
\nabla_\theta
\mathrm{Best\text{-}of\text{-}k}(\theta;x).
\label{eq:tailrl-harmonic-gradient-main}
\end{equation}
At $T=1$, this has the standard expected-reward gradient, while
$J_{\mathrm{\tailrl{}}}^{(T)}(\theta;x)\to J_{\mathrm{\tailrl{}}}(\theta;x)$
as $T\to\infty$.

Equivalently, the finite-order gradient can be written directly in terms of tail probabilities:
\begin{equation}
\nabla_\theta J_{\mathrm{\tailrl{}}}^{(T)}(\theta;x)
=
\int_0^1
\frac{1-(1-p_\theta(x,\tau))^T}
     {p_\theta(x,\tau)}
\nabla_\theta p_\theta(x,\tau)
\,d\tau.
\label{eq:finite-tail-weight-gradient}
\end{equation}
The threshold weight equals $1$ at $T=1$ and approaches
$1/p_\theta(x,\tau)$ as $T\to\infty$.
Thus, increasing $T$ smoothly moves the objective from expected-reward RL toward population \tailrl{}, progressively emphasizing reward levels that are harder to reach.
\Cref{fig:weight-vs-survival} visualizes this interpolation and compares it with PKPO.

For binary rewards, the finite-order family and its estimator reduce exactly to their MaxRL counterparts; see \cref{app:tailrl-binary-recovery}.

\subsection{Finite-Rollout Estimator}
\label{sec:estimator-finite}
Critic-free methods express policy gradients as weighted combination of the score-function for different rollouts. In this section we seek to express the exact weights that allows us to give an unbiased estimate of the gradient of the order-$N$ truncated objective.

Suppose we draw $N$ independent rollouts
$z_1,\ldots,z_N\sim\pi_\theta(\cdot\mid x)$.
At each reward threshold, we divide one unit of credit equally among the sampled rollouts that exceed that threshold.
A rollout accumulates this credit over every threshold below its reward:
\begin{equation}
\omega(r(x,z_i))
:=
\int_0^{r(x,z_i)}
\frac{d\tau}
{\sum_{j=1}^{N}\ind_{\{r(x,z_j)>\tau\}}}.
\label{eq:tailrl-finite-weight}
\end{equation}
Thresholds cleared by fewer rollouts therefore contribute more weight.

The weights can be computed exactly after sorting the rewards.
Let
$r_{(1)}\leq\cdots\leq r_{(N)}$
denote the sorted rewards and set $r_{(0)}:=0$.
Then
\begin{equation}
\omega(r_{(i)})
=
\omega(r_{(i-1)})
+
\frac{r_{(i)}-r_{(i-1)}}{N-i+1},
\qquad
\omega(r_{(0)})=0.
\label{eq:tailrl-weight-recurrence}
\end{equation}

The resulting policy-gradient estimator has the standard score-function form:
\begin{equation}
\gtailrl(x)
:=
\sum_{i=1}^{N}
\omega(r(x,z_i))\,S(x,z_i).
\label{eq:tailrl-finite-estimator}
\end{equation}
Thus, \tailrl{} differs from a standard critic-free policy-gradient method only in how sampled rewards are converted into rollout weights.

\begin{theorem}[Unbiased finite-rollout estimator]
\label{thm:tailrl-finite-unbiasedness}
For $N$ independent rollouts,
\begin{equation}
\mathbb E[\gtailrl(x)]
=
\nabla_\theta
J_{\mathrm{\tailrl{}}}^{(N)}(\theta;x).
\label{eq:tailrl-finite-unbiasedness}
\end{equation}
\end{theorem}

Hence, the rollout budget determines which member of the \tailrl{} family is optimized:
one rollout recovers the expected-reward gradient, while larger rollout groups incorporate progressively higher Best-of-$k$ learning signals.
This differs from REINFORCE, where increasing the rollout count reduces estimation variance without changing the underlying expected-reward objective.

\paragraph{Centered advantages.}
In practice, we center the rollout weights within each group:
\begin{equation}
A_i
:=
\omega(r(x,z_i))-\bar{\omega},
\qquad
\bar{\omega}
:=
\frac{1}{N}\sum_{j=1}^{N}\omega(r(x,z_j)).
\label{eq:tailrl-centered-advantage}
\end{equation}
Centering reduces variance and allows \tailrl{} to be used as a drop-in replacement for the advantage calculation in standard policy-gradient implementations.
Because the baseline is estimated from the same rollout group, the centered estimator is unbiased for
$\nabla_\theta J_{\mathrm{\tailrl{}}}^{(N-1)}$
rather than
$\nabla_\theta J_{\mathrm{\tailrl{}}}^{(N)}$.
Complete derivations and proofs are given in \cref{app:tailrl-estimator-proofs}.

\section{Unifying Gradient Weight View}
\label{sec:weight-view}

\begin{wraptable}[12]{r}{0.5\textwidth}

\vspace{-10pt}
\caption{Scalar function $\phi$ and gradient weights $\phi'$. Standard RL gives equal weight to gradients while \tailrl{} weights the gradients inverse tail-probability.}
\label{tab:weight-view}
\centering
\small
\renewcommand{\arraystretch}{1.35}
\setlength{\tabcolsep}{6pt}
\begin{tabular}{lcc}
\toprule
Objective & $\phi(p)$ & $\phi'(p)$ \\
\midrule
$J_{\mathrm{RL}}$ & $p$ & $1$ \\[2pt]
$J_{\mathrm{\tailrl{}}}^{(T)}$ & $-\sum_{\ell\le T}(1-p)^{\ell}/\ell$ & $\tfrac{1-(1-p)^{T}}{p}$ \\[2pt]
$J_{\mathrm{\tailrl{}}}$  & $\log p$ & $1/p$ \\[2pt]
$J_{\mathrm{PKPO}}$ & $1-(1-p)^{k_{\mathrm{opt}}}$ & $k_{\mathrm{opt}}(1-p)^{k_{\mathrm{opt}}-1}$ \\
\bottomrule
\end{tabular}
\end{wraptable}

To compare population-level \tailrl{} with expected reward maximization, we look one step earlier. We express the gradients of  \tailrl{} and  expected reward maximization under a unified view and ask how each objective up-weights its gradients. \textbf{\tailrl{} and expected reward maximization belong to a common family of objectives}.

Consider a fixed input $x$ and a fixed parameters $\theta$ of our policy $\pi_{\theta}(\cdot \mid x)$.
Recall that we can express the tail-probability as,
\begin{equation}
p_\theta(x,\tau):=\Pr_{z\sim\pi_\theta(\cdot\mid x)}\!\left(r(x,z)>\tau\right)
\end{equation}

Its gradient $\nabla_\theta p_\theta(x,\tau)$ points in the direction that makes the tail-event more likely. 

Up to constants independent of $\theta$, they can be written in the common form they can be expressed as,
\begin{equation}
J_\phi(\theta;x):=\int_0^1\phi\!\left(p_\theta(x,\tau)\right)d\tau,
\label{eq:level-utility-view}
\end{equation}
Here, $\phi(p)$ specifies how the objective values a tail-probability $p_{\theta}(x,\tau)$.
Differentiating w.r.t $\theta$ gives
\begin{equation}
\nabla_\theta J_\phi(\theta;x)=\int_0^1\phi'\!\left(p_\theta(x,\tau)\right)\nabla_\theta p_\theta(x,\tau)\,d\tau.
\label{eq:weight-view}
\end{equation}

\begin{wrapfigure}[19]{r}{0.45\textwidth}
\centering
\vspace{-10pt}
\includegraphics[width=\linewidth]{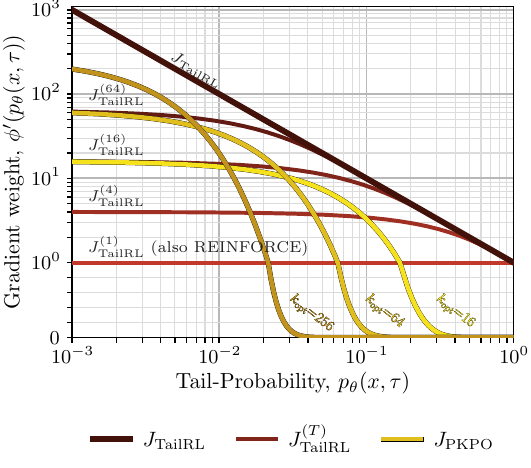}
\caption{Gradient weight assigned to tail-probability $p_{\theta}(x,\tau)$.
Order-$T$ \tailrl{} approaches the population weight.
PKPO's weight is capped at $k_{\mathrm{opt}}$ and vanishes as $p_{\theta}(x,\tau)$ grows.}
\label{fig:weight-vs-survival}
\vspace{0pt}
\end{wrapfigure}

The derivative $\phi'$ is therefore the gradient weight of the derivative of a tail-probability $p_{\theta}(x,\tau)$, and it does not introduce a new objective.
It is the marginal value that the objective assigns to making that level more likely.
The three objectives differ only in this weight (\cref{fig:weight-vs-survival,tab:weight-view}).Expected reward maximization uses $\phi(p)=p$, so $\phi'(p)=1$: every tail-probability receives the same weight, regardless of how often the policy reaches it.
\tailrl{} uses $\phi(p)=\log p$, so $\phi'(p)=1/p$.
Because $p_\theta(x,\tau)$ is non-increasing in $\tau$, this places greater weight on higher reward thresholds that the policy reaches rarely. The gradient of order-$T$ member of the truncated family is,
\begin{equation}
\nabla_\theta J_{\mathrm{\tailrl{}}}^{(T)}(\theta;x)
=
\int_0^1
\frac{1-\left(1-p_\theta(x,\tau)\right)^T}
{p_\theta(x,\tau)}
\nabla_\theta p_\theta(x,\tau)\,d\tau.
\label{eq:truncated-weight-function}
\end{equation}
It up-weights the gradient of the tail probabilities as
$\phi'(p)=\frac{1-\left(1-p\right)^T}
{p}$, i.e. the multiplier in \cref{eq:truncated-weight-function}.
At $T=1$, this multiplier equals $1$, recovering the expected reward gradient.
As $T\to\infty$, it approaches the tail-likelihood weight $1/p$. Expected reward maximization gives equal weight to the gradient of all tail-probabilities.

\section{Experiments}
\label{sec:experiments}

We organize the experiments around four questions implied by the theory. 
\begin{enumerate}
    \item First, is the population-level \tailrl{} objective a useful learning target, and does its finite-rollout estimator approach it as number of rollouts increases? (\cref{sec:tailrl-finite-order})
    \item Second, does \tailrl{} help specifically when high-reward rollouts are attainable but rare? (\cref{sec:tailrl-objective})
    \item  Third, does its alignment with Best-of-$k$ yield stronger inference-time scaling? (\cref{sec:tailrl-best-of-k-expansion})
    \item Finally, can \tailrl{} prevent a common moderate-reward behavior from displacing rarer, better outcomes? (\cref{sec:tailrl-best-of-k-expansion} and \cref{sec:tailrl-objective})
\end{enumerate}

To answer these questions, we devise 4 experimental settings. 
We test the population level objective and its finite-rollouts approximation on a localization task in the \textbf{ImageNet} dataset (\cref{sec:exp-imagenet}) . To understand the behavior of \tailrl{} under rare, high reward rollouts, we consider a \textbf{Text-Maze} setting (\cref{sec:exp-maze}). \textbf{GUI grounding} allows us to study the scaling behavior of models trained with \tailrl{} as more inference compute in poured into the problem (\cref{sec:exp-gui}) and finally, \textbf{Code Optimization} tests resistance to a safe but suboptimal shortcut (\cref{sec:exp-pie}).
Together, they test whether \tailrl{} works for the reasons predicted by the theory, from the population objective to the behavior of the learned policy.

We compare \tailrl{} with GRPO \citep{shao2024deepseekmath}, RLOO \citep{ahmadian2024back}, two popular group-based policy optimization algorithms, and PKPO \cite{walder2025pkpo} that maximizes Best-of-$k_{\mathrm{opt}}$ for an inference budget $k_{\mathrm{opt}}$ under pure on-policy policy gradient setup to avoid confounders.

\subsection{ImageNet Object Localization}
\label{sec:exp-imagenet}

\begin{figure}[t]
\centering
\includegraphics[width=\textwidth]{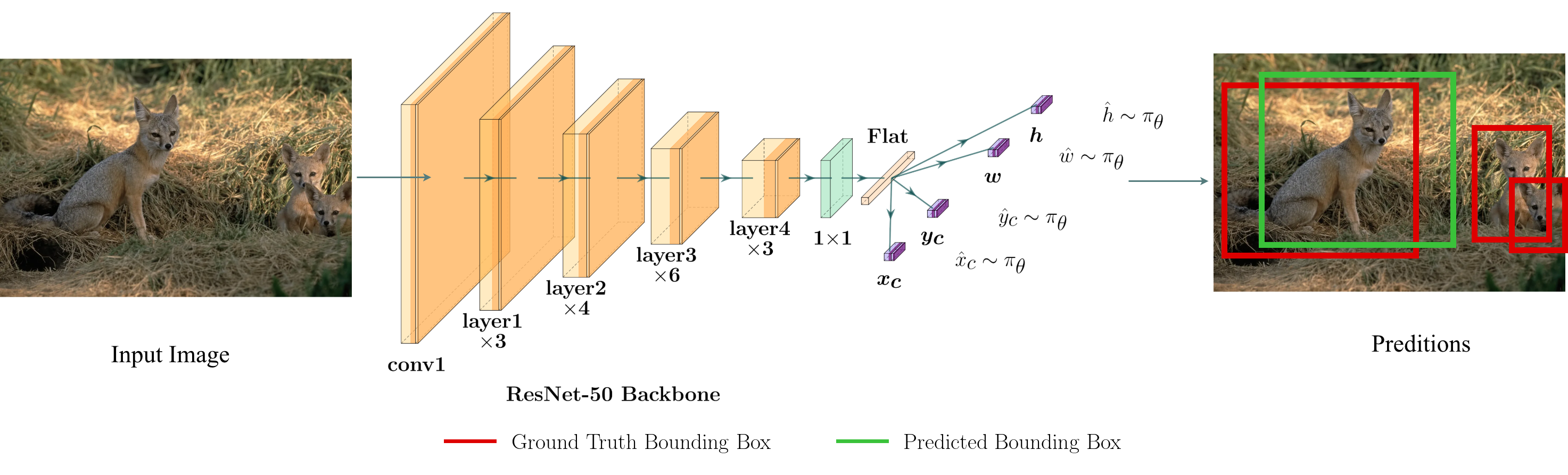}
\caption{\textbf{(ImageNet object localization)} An overview of the ImageNet object localization task.
A ResNet-50 backbone \citep{he2015deepresiduallearningimage} encodes the image, and four categorical heads parameterize the policy. A rollout samples the box center coordinates $\hat{x}_c, \hat{y}_c$, width $\hat{w}$, and height $\hat{h}$.
The sampled box is rewarded by its IoU with the matched ground-truth box.}
\label{fig:imagenet-overall}
\end{figure}

ImageNet object localization requires predicting a bounding box around an object of interest without classifying the object \citep{russakovsky2015imagenetlargescalevisual}.
This setting addresses three questions.
First, how does reinforcement learning from a scalar reward compare with supervised objectives that directly observe the ground-truth bounding box? Second, does the finite rollout \tailrl{} gradient estimator $\gtailrl$ approach the population-level gradient $\nabla_\theta J_{\mathrm{\tailrl{}}}$ as $N$ increases? Third, how does \tailrl{} compare with the expected reward baselines at matched and smaller training rollout budgets?

\paragraph{Task and setup}
For an input image $x$, a rollout $z$ is a bounding box drawn from the categorical policy of \cref{fig:imagenet-overall}, and the continuous reward $r(x,z)\in[0,1]$ is its intersection-over-union (IoU) with the ground-truth box.
We report CorLoc@$\delta$, the fraction of input images for which the greedy prediction has IoU greater than $\delta$ \citep{deselaers2010localizing}; mean IoU; and Best-of-$k$ IoU, the largest IoU among $k$ inference rollouts.

Its worthwhile to note that since the policy induces a categorical distribution over the finite set of possible bounding boxes, we can evaluate the probability and IoU reward of every box.
This gives us the exact reward distribution in closed form. Using this we can directly compute and optimize the population-level objective $J_{\mathrm{\tailrl{}}}$.
We also train supervised baselines that directly regress the ground-truth coordinates.
Training and evaluation details are provided in \cref{app:imagenet}.

\begin{figure}[b]
\centering
\includegraphics[width=\textwidth]{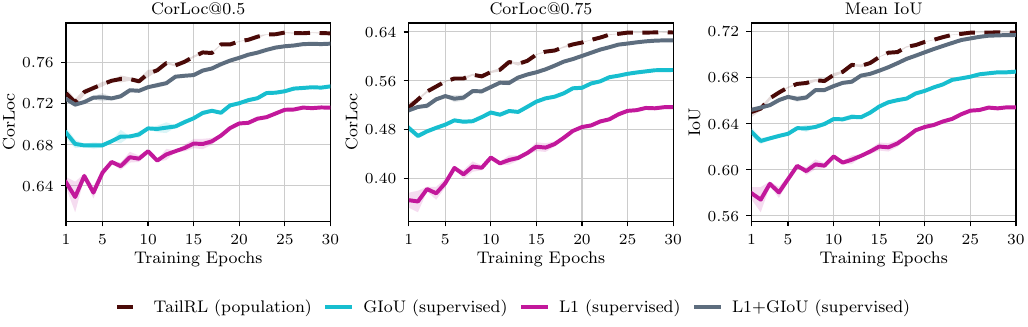}
\caption{\textbf{(ImageNet object localization)} A comparison of population-level \tailrl{} versus direct supervision for the task of ImageNet Object Localization. \tailrl{} either outperforms or is competitive against task-specific objectives.}
\label{fig:imagenet-sup-curves}

\end{figure}

\paragraph{Comparison with direct supervision}
Using only scalar IoU rewards, population-level \tailrl{} matches or exceeds objectives that directly supervise the ground-truth coordinates (\cref{fig:imagenet-sup-curves}). 
Among the supervised objectives, the combined L1+GIoU objective \citep{rezatofighi2019giou} used by DETR \citep{carion2020detr} is the strongest.
\tailrl{} achieves higher CorLoc@$0.5$ and CorLoc@$0.75$ than L1+GIoU while obtaining comparable mean IoU.
It also outperforms the individual L1 and GIoU baselines across all reported metrics throughout training.
These results show that under large training compute, optimizing a scalar continuous reward can compete with task-specific supervised objectives.

\begin{AIbox}{{Takeaway 1}}
When exact tail-probabilities are estimable, the population-level \tailrl{} 
performs as well as or better than task-specific supervised objectives that recieve ground truth information.
\end{AIbox}

\begin{wrapfigure}[17]{r}{0.5\textwidth}
\centering
\vspace{-0pt}
\includegraphics[width=\linewidth]{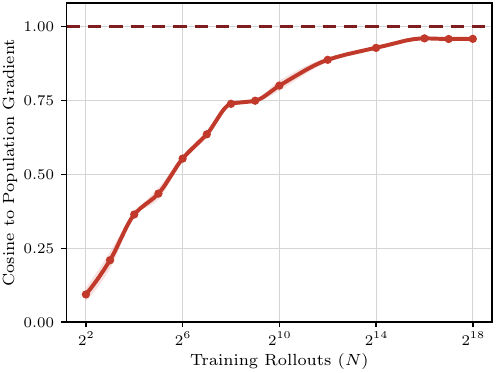}
\caption{\textbf{(ImageNet object localization)} Finite rollout \tailrl{} gradients converge toward the population-level gradient as training rollouts increase.}
\label{fig:imagenet-gradients}
% \vspace{-2pt}
\end{wrapfigure}

\paragraph{Increasing training rollout budget}
Increasing training rollouts, $N$ moves the finite rollout \tailrl{} curves toward the exact population-level objective (\cref{fig:imagenet-rl-curves}).
At $N=1024$, \tailrl{} closely tracks the population-level objective, while smaller rollout budgets remain progressively farther away.
This ordering holds across CorLoc@$0.5$, mean IoU, and Best-of-$1024$ IoU and other performance measures that we test for this experiment. We test this convergence directly at the gradient level in \cref{fig:imagenet-gradients}.
As $N$ increases, the cosine similarity between the sampled finite rollout gradient and the population-level gradient rises steadily toward exact agreement.
This increasing alignment is consistent with the convergence predicted in \cref{sec:tailrl-finite-order}. Increasing rollouts not only reduce the variance of the objective we are trying to estimate, but it also approximates a higher order truncated objective.

\begin{AIbox}{{Takeaway 2}}
As training rollouts increase, the gradient of order-$N$ truncated objective computed from finite rollout \tailrl{} converges to the gradient of population-level \tailrl{}.
\end{AIbox}

\paragraph{Comparison with RL baselines}
At a matched budget of $N=1024$, \tailrl{} outperforms GRPO and RLOO across all three metrics, despite every method receiving the same continuous IoU feedback.
The gap does not close at smaller budgets: even at $N=16$, \tailrl{} exceeds both baselines trained at $N=1024$, using $\tfrac{1}{64}$ as many training rollouts per input. \tailrl{} also pareto dominates PKPO which maximizes expected maximum reward. With just $N=16$ training rollouts, it outperforms at CorLoc@$0.5$ and mean IoU. At matched training compute, \tailrl{} matches PKPO at Best-of-$1024$ reward, while significantly outperforming at mean IoU and CorLoc@$0.5$.

\begin{figure}[t]
\centering
\includegraphics[width=\textwidth]{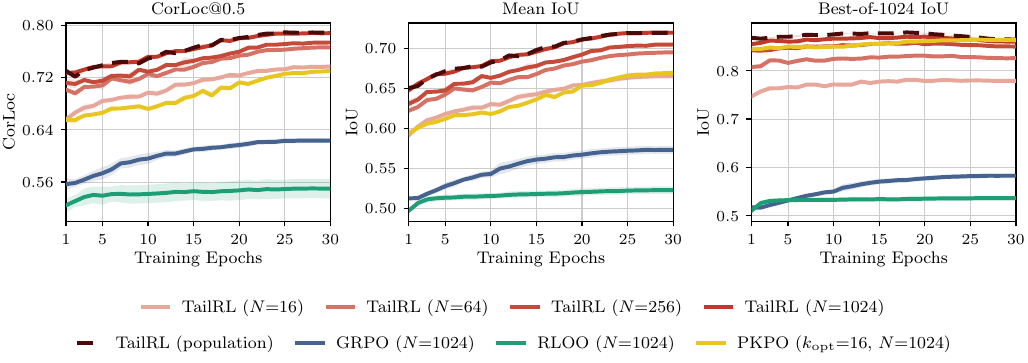}
\caption{\textbf{(ImageNet object localization)} Comparing performance of expected-reward maximization baselines with \tailrl{} on  held-out validation set. We report CorLoc@$0.5$, mean IoU, and Best-of-$1024$ IoU.
\tailrl{} uses training rollouts $N\in\{16,64,256,1024\}$, with darker curves indicating larger $N$; the dashed curve optimizes the exactly computed \tailrl{} population-level objective.
GRPO and RLOO use $N=1024$ (\cref{app:imagenet}).}
\label{fig:imagenet-rl-curves}
\end{figure}

\begin{figure}[t]
\centering
\includegraphics[width=0.9\textwidth]{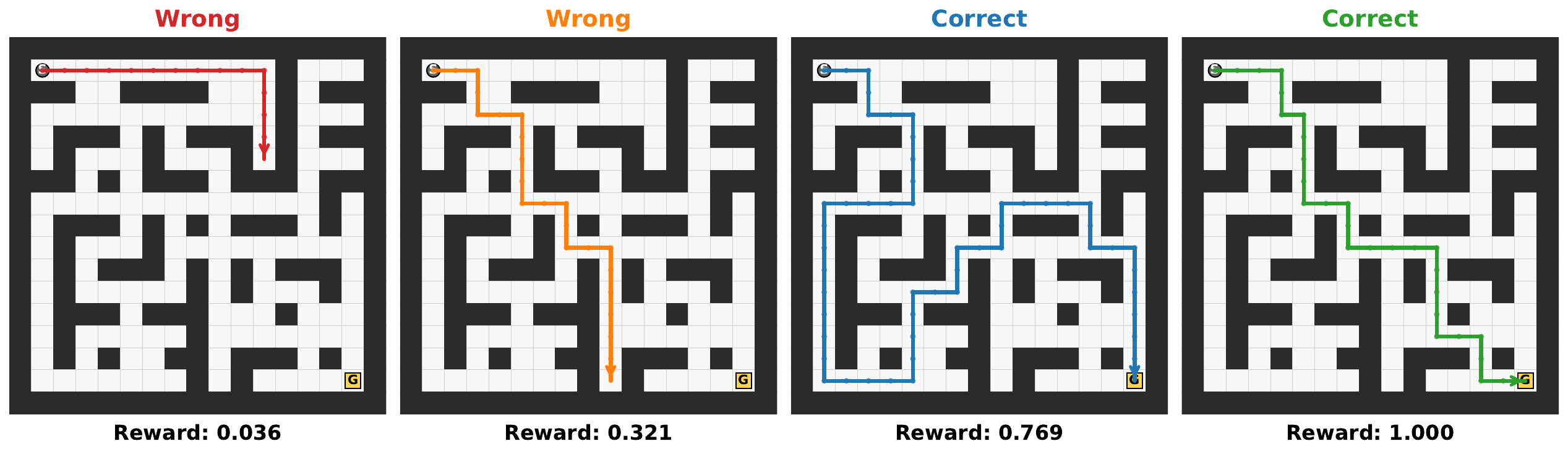}
\caption{\textbf{(Text-Maze)} Four rollouts on one Text-Maze and the continuous reward each receives.
The two left paths fail to reach the goal and still receive continuous credit for progress.
The third reaches the goal but wanders, so it gets rewarded below a shortest path.
Only the rightmost, a shortest path (right), earns reward $1$.}
\label{fig:app-maze-examples}
\end{figure}

\subsection{Text-Maze Navigation from Low-Success Initial Policies}
\label{sec:exp-maze}

Next we examine how \tailrl{} and  other baselines compare when we vary the quality of the initialization policy.
By varying the amount of supervised pretraining before doing RL post-training, we obtain a controlled range of initial policies with varying coverage over high reward attaining rollouts.
We ask whether \tailrl{} can learn from poor initialization of policies and how it compares against  expected reward maximization baselines.

\paragraph{Task and setup}

For an input $17\times17$ maze represented as text, a rollout is a token sequence describing a path.
The continuous reward $r(x,z)\in[0,1]$ measures proximity to the goal and path length relative to the shortest path (\cref{fig:app-maze-examples}).
A rollout receives reward $1$ only when it reaches the goal along a shortest path, an event we call \emph{shortest-path success}.
Unsuccessful rollouts receive partial credit for ending closer to the goal, while successful rollouts receive more reward for shorter paths.
We evaluate the post-trained policies on a held-out validation set of $1024$ mazes.

By varying the amount of supervised pretraining on goal-reaching trajectories, we obtain initial policies with shortest-path success rates ranging from approximately $1\%$ down to $0.01\%$.
Starting from each checkpoint, we separately post-train policies with \tailrl{}, GRPO, and PKPO ($k_{\mathrm{opt}}$).
All methods use the same initial policy and $N=16$ unless stated otherwise.
Model, reward, training, and evaluation details are provided in \cref{app:maze}.

\begin{wrapfigure}[21]{r}{0.5\textwidth}

\centering
\vspace{-10pt}
\includegraphics[width=\linewidth]{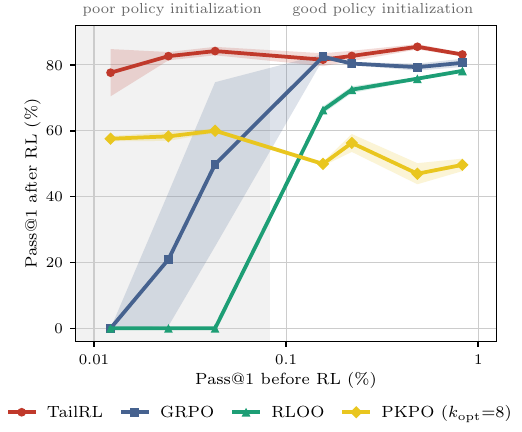}
\caption{\textbf{(Text-Maze)} Pass@$1$ before and after RL post-training for the task of Text-Maze navigation.
Shaded area marks the regime where initialization policy is poor. After RL post-training,  RLOO and GRPO fail to reliably improve Pass@$1$ in this regime.}

\label{fig:maze-init-sweep}
\vspace{-14pt}
\end{wrapfigure}

\paragraph{Learning from low-success initial policies}
\Cref{fig:maze-init-sweep} shows that RLOO, GRPO and \tailrl{} behave similarly when the initial policy produces shortest-path successes frequently, but diverge as the initial success rate falls. Despite higher coverage of initial policy,  PKPO underperforms \tailrl{} at Pass@$1$. The expected reward baselines fail to learn reliably in the low initial success regime. PKPO consistently improves Pass@$1$ across the spectrum of initialization policy yet still under-performs \tailrl{} at Pass@$1$ even at the regime of poor policy initalization. 

Once the initial success rate rises beyond this regime, \tailrl{}, RLOO and GRPO learn well and their differences narrow. GRPO's performance improves before RLOO, although the performance becomes seed dependent before it consistently starts to navigate the maze.
The advantage of \tailrl{} is therefore concentrated where high-reward rollouts are attainable but rare. Inference and training-rollout budget sweeps are reported in \cref{app:maze-additional}.

\begin{AIbox}{{Takeaway 3}}
\tailrl{} upweighs gradients for rare exceptional rollouts. This results in \tailrl{} outperforming other methods when initial policy has poor coverage over rare excellent rollouts.
\end{AIbox}

\subsection{GUI Grounding with Vision-Language Models}
\label{sec:exp-gui}

We next demonstrate the efficacy of \tailrl{} on Vision Language Models (VLM). We consider a visual grounding task with verifiable rewards. Through this task, we show that models trained with \tailrl{} benefit from inference-time sampling.

\textbf{Task and setup:}
Given a screenshot and a natural-language instruction to perform a click, the VLM policy generates the coordinates of the click location.
The continuous reward $r(x,z)\in[0,2.5]$ combines proximity to the target, a bonus for clicking inside the target element, and a format bonus \citep{yuan2025segui}.
We fine-tune Qwen2.5-VL-3B and Qwen2.5-VL-7B \citep{bai2025qwen25vl} on the GTA1 grounding corpus \citep{yang2025gta1}.
The training configuration is identical across \tailrl{}, GRPO, and RLOO except for the advantage estimator.
We evaluate the resulting policies on ScreenSpot-Pro \citep{li2025screenspotpro}.
Here, Pass@$k$ is the probability that at least one of $k$ inference rollouts clicks inside the target element \citep{chen2021codex}, while Best-of-$k$ reward is the largest continuous reward among those rollouts.
The reward construction, training protocol and evaluation procedure are detailed in \cref{app:gui}.

\textbf{Results:}
At both model scales, \tailrl{} and RLOO achieve similar Pass@$1$, but they substantially differ in how the learned policies respond to additional inference sampling (\cref{fig:gui-scales}).
As inference rollouts $k$ increases, \tailrl{}'s Pass@$k$ continues to rise, whereas RLOO plateaus considerably earlier; GRPO's Pass@$k$ is inferior to \tailrl{} and RLOO.
\tailrl{} leaves more inputs with a non-negligible probability of producing a successful click, so additional samples continue to reveal useful candidates.

We define the \emph{matching budget} as the smallest evaluated value of $k$ at which \tailrl{} reaches or exceeds a baseline's mean Pass@$1024$. On ScreenSpot-Pro, at 3B, \tailrl{} reaches RLOO's Pass@$1024$ with $8$ rather than $1024$ inference rollouts, a $128$-fold reduction; at 7B it does so with $4$ rollouts, a $256$-fold reduction. A single \tailrl{} rollout exceeds GRPO's mean Pass@$1024$ at both model scales.

\begin{AIbox}{{Takeaway 4}}

On GUI grounding, \tailrl{} shows superior scaling of inference compute when compared against expected reward maximization baselines.
\end{AIbox}

\subsection{Code Runtime Optimization}
\label{sec:exp-pie}

%: \tailrl{} trained VLMs are better at test-time scaling.

\begin{figure}[t]
\centering
\includegraphics[width=\textwidth]{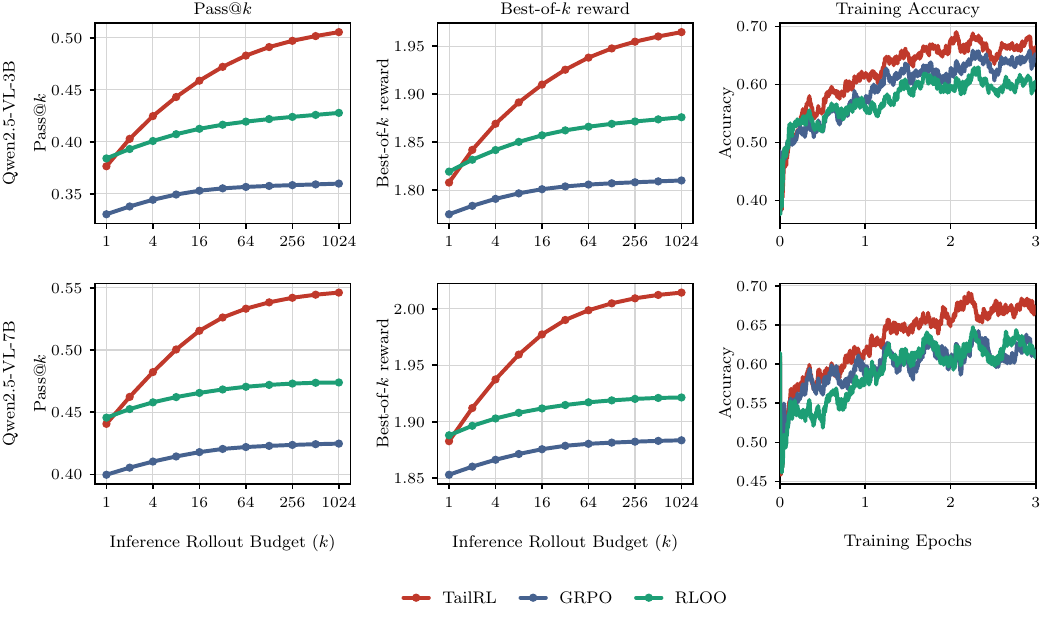}
\caption{\textbf{(GUI-grounding)} Evaluation results for GUI grounding task on ScreenSpot-Pro benchmark for Qwen2.5-VL-3B (top) and 7B (bottom).
Columns show Pass@$k$, Best-of-$k$ reward, and exponentially smoothed training-batch accuracy.
\tailrl{}  matches RLOO's Pass@$1024$ using $8$ rollouts at 3B and $4$ rollouts at 7B scale, a reduction of $128\times$ and $256\times$ in test-time compute respectively.}
\label{fig:gui-scales}
\end{figure}

We now ask a qualitatively different question: what happens when training is attracted to a safe but systematically suboptimal behavior hurting further exploration?
Such behaviors provide a reliable moderate reward and can become stable solutions, even when rarer and riskier behaviors offer substantially better outcomes \citep{skalse2022reward,baronio2025kevin}.
This creates a stress test for \tailrl{}: whether it concentrates the policy on the dependable yet suboptimal shortcuts or improves the policy's coverage over high-reward tail.

\begin{wrapfigure}{r}{0.5\textwidth}
\centering
\vspace{-10pt}
\includegraphics[width=\linewidth]{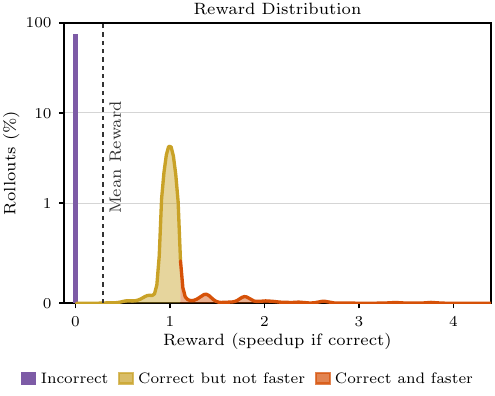}
\caption{\textbf{(Code runtime optimization)} Initial reward distribution over all test problems of PIE dataset for Qwen3-1.7B.
$74.4\%$ of rollouts are incorrect, $23.5\%$ are correct but not faster than the input, concentrating in a spike at $1\times$, and $2.1\%$ are correct and faster.}
\label{fig:pie-rollout-dist}
\vspace{-8pt}
\end{wrapfigure}

\paragraph{Task and setup}
Each input is a slow C++ program from the PIE \citep{pie_iclr_2024_spotlight} corpus of competitive-programming. The LLM is prompted to write a faster yet correct version of the input program. A rollout is a rewritten program and is compiled and executed against the problem's test suite.
Incorrect outputs receives zero reward otherwise receives a reward equal to the speedup over the input program.
We measure speedup using gem5 time so that timing noise cannot create spurious improvements \citep{binkert2011gem5,lowepower2020gem5}.
We post-train Qwen3-1.7B \citep{yang2025qwen3technicalreport} with \tailrl{}, GRPO, and RLOO using $N=16$ rollouts per program.

We report mean reward over all rollouts; density of Best-of-$k$ rewards at $k=1024$ ; and the fraction of rollouts that pass every test; and Best-of-$k$ reward density at $k=1024$.
A policy that always reproduces its input obtains a mean reward of $1$ and passes all test cases. Full task, training, and evaluation details are provided in \cref{app:pie}.

% TODO: Add standard-error bands over the three seeds.

\begin{figure}[b]
\centering
\includegraphics[width=\textwidth]{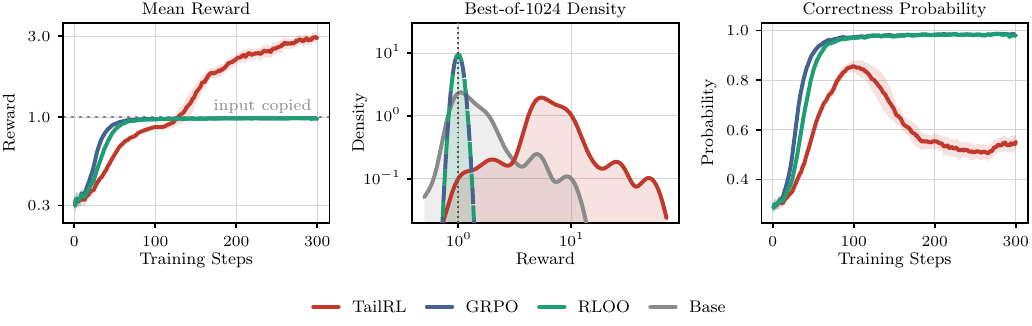}
\caption{\textbf{(Code runtime optimization)}
(Left) Average reward during training for 1 epoch, EMA over three runs per objective.
(Middle) kernel density of the per-problem Best-of-$1024$ reward on the test-set problems of PIE. GRPO and RLOO are drawn with alternating dashes because their densities coincide at a reward of 1.0 indicating that they only echo the input program. Right: fraction of training  rollouts for a given batch passing every test. In this task, it is better to explore and find the faster rewrite of the input program than to copy the input and be correct. Correctness only gets a reward of $1.0$ while correct and meaningfully faster programs receive much higher rewards. Just copying the input is a degenerate soltuion exhibited by RLOO and GRPO \cref{app:pie-samples} 
 All experiments have been performed on 3 seeds per method.}
\label{fig:pie-results}
\end{figure}

\paragraph{Results}
GRPO and RLOO rapidly converge toward the copying shortcut (\cref{fig:pie-results}) as evidenced by sample rollouts  (\cref{app:pie-samples}).
Their correctness rise above $98\%$, while their mean rewards settle just below $1.0$. Their policy entropy collapses by one to two orders of magnitude over the same interval (\cref{fig:pie-entropy}).
Expected-reward training therefore converges to the most reliable mode of the reward distribution while eliminating the behavior to maximize coverage over excellent outputs.

\tailrl{} follows a different trajectory.
Rather than collapsing onto the guaranteed reward from copying the input, it maintains substantially higher entropy and continues producing risky rewrites.
This lowers single-rollout correctness, but places more probability on programs that are both correct and meaningfully faster, raising mean reward to $2.92$, nearly three times the copying value.
All methods use the same one-epoch training budget, and \tailrl{} is still improving at step $300$; these are therefore matched-compute results rather than converged endpoints.
On the held-out test set the trained policies separate sharply: \tailrl{}'s mean Best-of-$1024$ speedup is $7.7\times$ against $0.98\times$ for GRPO and $0.96\times$ for RLOO, whose best rollouts never beat the input they reproduce (\cref{fig:pie-results}, middle; \cref{fig:pie-bok}).

\begin{AIbox}{{Takeaway 5}}
When an easy suboptimal solution exists, expected reward maximization methods settle for it.
\tailrl{} keeps searching and learns to produce rarer and better outcomes.
\end{AIbox}

\section{Related Work}
\label{sec:related}

\paragraph{Objectives beyond expected reward}
Reinforcement-learning post-training for reasoning and agentic models still largely maximizes expected reward \citep{ouyang2022traininglanguagemodelsfollow,guo2025deepseek}.
GRPO \citep{shao2024deepseekmath} and RLOO \citep{ahmadian2024back} are practical, critic-free variants of the score-function policy gradient \citep{williams1992}.
RLOO changes the baseline, while GRPO applies a common normalization to the update for each input; neither changes the relative weighting of reward levels within that input (\cref{sec:weight-view}).
\tailrl{} changes the objective instead.
Closest to our work, MaxRL maximizes the log-probability of success for binary rewards and decomposes its gradient into a harmonic sum of Pass@$k$ gradients \citep{tajwar2026maximumlikelihoodreinforcementlearning}.
PKPO derives unbiased estimators for a chosen Pass@$k$ objective and its continuous counterpart, Best-of-$k$ \citep{walder2025pkpo}; related methods directly optimize a selected inference-time metric \citep{tang2025passk,chen2025passktrainingadaptivelybalancing} or study objectives defined by monotone transforms of success probability \citep{davis2025objectivereasoningreinforcementlearning}.
These approaches select a threshold, transform, or inference budget.
\tailrl{} instead integrates log tail-probability over all reward thresholds, yielding harmonic combinations of Best-of-$k$ objectives and recovering MaxRL exactly for binary rewards.

\paragraph{Coverage and tail-sensitive reinforcement learning}
Expected-reward training can narrow the policy distribution and remove rare, high-reward behavior, an effect studied mechanistically \citep{cui2025entropy} and observed at large sampling budgets \citep{yue2025doesreinforcementlearningreally,kirk2024understandingeffectsrlhfllm}.
Rather than adding a separate diversity regularizer, \tailrl{} addresses this through the objective itself by weighting each reward level inversely by how often the policy reaches it.
\tailrl{} is also distinct from distributional and risk-sensitive reinforcement learning, which learn a return distribution through a critic or optimize a tail statistic at a chosen risk level \citep{bellemare2017distributional,rockafellar2000optimization}.
\tailrl{} uses no learned critic: its policy-gradient weights are computed directly from each rollout group, and it integrates over all reward thresholds rather than fixing one risk level, which can otherwise overlook rare successes \citep{greenberg2022cesor}.
We discuss further connections to exploration, sample allocation, threshold decompositions of continuous rewards, and the evaluated domains in \cref{app:related-extended}.

\section{Conclusion}
\label{sec:conclusion}

We introduced \tailrl{}, a likelihood objective for continuous rewards that maximizes the expected log-probability of exceeding a uniformly sampled reward threshold.
It recovers MaxRL exactly for binary rewards, decomposes harmonically over Best-of-$k$ gradients, and admits a simple, critic-free finite-rollout estimator.
Across four settings, its gains were largest when high-reward rollouts were rare or when training was drawn toward a common but suboptimal behavior.
In these regimes, \tailrl{} learned more reliably from rare outcomes, benefited more from additional training and inference samples, and avoided suboptimal collapse.
More broadly, a continuous reward is more than a scalar to average: it defines a family of success events, one at every reward level.
\tailrl{} provides a practical way to optimize their likelihoods.

\section*{Acknowledgements}

This research used the DeltaAI advanced computing and data resource \citep{deltaai}, which is supported by the National Science Foundation under award OAC-2320345 and by the State of Illinois.
 DeltaAI is a joint effort of the University of Illinois Urbana-Champaign and its National Center for Supercomputing Applications.
These resources were used through the Advanced Cyberinfrastructure Coordination Ecosystem: Services \& Support (ACCESS) program \citep{access_compute}.
Overall, this project used ACCESS allocations CIS250426, CIS260353, CIS260522, CIS260557, CIS260677, CIS260678, and CIS260679.
We are especially grateful to Brett Bode of the NCSA Delta Support team, whose assistance in effectively using the Delta cluster was critical to completing this work on schedule. We would like to extend our sincere gratitude to Sooth Labs for their generous support by granting us their compute resources.
This work was partially supported by the National Science Foundation under Grants CCF-2106778. Fahim and Shrinivas were funded in part by Stack AV and Skylark Labs

We are grateful to Sumukh Aithal, Ben Freed, Surgan Jandial, Sreyas Venkatraman, Kartik Sharma, Elton Lobo, Parv Maheshwari, Rhea Basappa, Srinath Ravi, Rishubh Parihar, Harsh Rangwani, Chaitanya Chawla, and Mayank Mishra for carefully reviewing earlier drafts and providing valuable feedback.
We also thank Rohit Sonkar, Anoushka Alavilli, Jiayu Chen, and Mineui Hong from the CMU Auton Lab, and Lawrence Jang from Russ lab for their helpful feedback that improved the quality of the draft. The authors would also like to extend their gratitude to Prof. Yaser Sheikh and Prof. Yonatan Bisk for helpful discussions and suggestions throughout this work.

\iclr{\bibliographystyle{icml2026}}
\bibliography{refs}

%% ================================ APPENDIX ==================================
\newpage
\appendix
\etocdepthtag.toc{appendix}
\newpage
\phantomsection
\section*{Appendix Contents}
\begingroup
\small
\etocsettagdepth{maintext}{none}
\etocsettagdepth{appendix}{subsubsection}
\etocsettocstyle{}{}
\tableofcontents
\endgroup
\newpage

\crefalias{section}{appendix}
\crefalias{subsection}{appendix}
\crefalias{subsubsection}{appendix}

\onecolumn

% ============================= appendix.tex =============================
%% ============================================================================
%%  APPENDIX -- notation table (the paper's single source of truth for
%%  symbols; sections adhere to it) + per-experiment hyperparameter tables.
%% ============================================================================

\section{Notation}
\label{app:notation}

\Cref{tab:notation} collects every symbol used in the paper. One symbol,
one meaning is strictly enforced.

\begin{table}[h]
  \caption{Notation.}
  \label{tab:notation}
  \centering
  \begin{tcolorbox}[enhanced, hbox,
    colback=teal!5, colframe=teal, coltext=black, coltitle=white,
    fonttitle=\bfseries, arc=1mm, boxrule=1pt, boxsep=1pt,
    left=2pt, right=2pt, top=0pt, bottom=2pt,
    toptitle=3pt, bottomtitle=3pt, center]
    \small
    \renewcommand{\arraystretch}{1.2}
    \setlength{\tabcolsep}{8pt}
\begin{tabular}{l|l}
  \multicolumn{1}{l|}{\textbf{Symbol}} & \multicolumn{1}{l}{\textbf{Meaning}} \\
  \hline
  $x$ & an input prompt \\
  $\theta$ & policy parameter \\
  $\rho$ & the input distribution \\
  $\pi_\theta(\cdot \mid x)$ & the policy with parameters $\theta$ \\
  $z \sim \pi_\theta(\cdot \mid x)$ & a rollout; $z_{1:k}$ are i.i.d.\ rollouts \\
  $r(x, z)$ & the reward of rollout $z$\\
  $S(x, z)$ & the score-function $\nabla_\theta \log \pi_\theta(z \mid x)$ \\
  \hline
  $\tau \in [0, 1)$ & a threshold on the reward range, as in \cref{sec:intro} \\

  $p_\theta(x, \tau)$ & tail-probability $\Pr_{z \sim \pi_\theta(\cdot \mid x)}\bigl[r(x,z) > \tau\bigr]$ \\
  $q_\theta(x)$ & the success probability when the reward is binary \\
  \hline
  $J_{\mathrm{RL}}(\theta; x)$ & the  expected reward on $x$ \\
  $J_{\mathrm{\tailrl{}}}(\theta; x)$ & the tail-likelihood $\int_0^1 \log p_\theta(x, \tau)\, d\tau$ \\
  $J_{\mathrm{\tailrl{}}}^{(T)}(\theta; x)$ & its order-$T$ truncation (compute-indexed family) \\
  $T$ & the truncation order; $T \to \infty$ recovers $J_{\mathrm{\tailrl{}}}$ \\
  \hline
  $N$ & rollouts per input in a training group \\
  $M$ & number of binary reward training rollouts that succeed \\
  $\gtailrl$ & the $N$-rollout estimator, no baseline \\
  $\gtailrlhat$ & the same estimator with the mean baseline \\
  $A_i$ & the advantage assigned to rollout $z_i$ \\

  $\omega(r)$ & un-centered gradient weight for rollout w/ reaward $r$  \\
  $\omega_\theta^{(N)}(r)$ & un-centered gradient weight for rollout under finite compute  \\
  $\sigma_\theta(x)$ & variance of rewards for an input $x$ \\
  $\phi(p)$ & scalar function; $\phi'(p)$ is the gradient weight (\cref{sec:weight-view}) \\
  \hline
  $k$ & the inference rollout budget at evaluation \\
  $k_{\mathrm{opt}}$ & the Pass@$k$ order that the PKPO objective optimizes \\
  Pass@$k$ & probability that at least one of $k$ samples of $z$ succeeds \\
  Best-of-$k$ & the largest attained reward among $k$ samples of $z$ \\
\end{tabular}
\end{tcolorbox}
\end{table}

\section{Extended Related Work}
\label{app:related-extended}
This section gives the full related-work discussion summarized in \cref{sec:related}, together with several literatures not covered there: inference-time selection as its own topic, exploration and sample-allocation methods, threshold decompositions of continuous rewards, and the literatures behind each evaluated domain.

\paragraph{Reinforcement learning post-training}
Reinforcement learning now drives post-training for reasoning and agentic models, and almost all of it maximizes  expected reward.
The recipe descends from learning against human preferences \citep{stiennon2020summarize,ouyang2022traininglanguagemodelsfollow} and, once verifiable rewards replaced learned reward models, produced the current generation of reasoning systems \citep{guo2025deepseek,jaech2024openai,kimiteam2025kimik15scalingreinforcement,lambert2025tulu3pushingfrontiers}.
Practical estimators have converged on critic-free group-relative updates: GRPO standardizes rewards within a group \citep{shao2024deepseekmath}, RLOO subtracts a leave-one-out baseline \citep{ahmadian2024back,kool2019buy}, and later variants adjust clipping, normalization, and the level at which the ratio is formed \citep{yu2025dapo,zheng2025groupsequencepolicyoptimization,liu2025understandingr1zeroliketrainingcritical}.
Every one of these estimators traces back to the score-function policy gradient \citep{williams1992,sutton1999policy,schulman2017proximalpolicyoptimizationalgorithms} and differs from the others in variance rather than in target.
They share one population objective, the mean of the reward distribution, and in the weight view of \cref{sec:weight-view} they share the flat gradient weight $\phi'(p)=1$.
\tailrl{} changes that target rather than the variance around it.

\paragraph{Objectives beyond the mean}
A recent line of work replaces the mean with a nonlinear functional of the policy's success probability, and \tailrl{} belongs to it.
MaxRL maximizes the log-probability of success for binary rewards and expands that logarithm into a harmonically weighted sum of Pass@$k$ gradients \citep{tajwar2026maximumlikelihoodreinforcementlearning}.
PKPO derives unbiased estimators for Pass@$k$ and for its continuous generalization, the expected maximum of $k$ rewards \citep{walder2025pkpo}, while other work optimizes a chosen inference-time metric directly \citep{tang2025passk,chen2025passktrainingadaptivelybalancing} or reweighs rollouts by problem difficulty \citep{yang2025depthbreadthsynergyrlvrunlocking}.
\citet{davis2025objectivereasoningreinforcementlearning} unify much of this by showing that popular algorithms implicitly ascend a monotone transform of the success probability, and \citet{thrampoulidis2025advantageshaping} recast the resulting Pass@$k$ advantage weights as surrogate reward maximization.
Related objectives target diversity \citep{hamid2025polychromicobjectivesreinforcementlearning,orney2026polyepotrainingexploratoryreasoning}, risk sensitivity \citep{jiang2025risksensitiverlalleviatingexploration}, and rare-success amplification \citep{osband2026delightfulpolicygradient,kaddour2026targetpolicyoptimization}.
Contemporaneous works extend two pieces of this picture: RL2ML generalizes the harmonic coefficient for binary rewards \citep{zheng2026rl2ml}, and OrderGrad estimates gradients of L-statistics over sorted rewards \citep{parmas2026ordergrad}.
Each of these methods commits to one threshold, one transform, or one order statistic.
\tailrl{} instead integrates the log tail-probability over every threshold, so its order-$T$ truncation is a harmonic combination of Best-of-$k$ objectives and reduces to MaxRL exactly when the reward is binary.

\paragraph{Inference-time scaling and selection}
Drawing many rollouts and keeping the best one has become a standard axis of deployment compute, which makes Pass@$k$ and Best-of-$k$ the metrics that matter at inference.
Repeated sampling raises coverage predictably over several orders of magnitude \citep{brown2024monkeys,schaeffer2025largelanguagemonkeyspower}, and allocating compute at test-time can beat allocating it to parameters \citep{snell2024scaling}.
Selecting among the samples requires either agreement \citep{wang2023selfconsistencyimproveschainthought}, a reward function \citep{cobbe2021training,lightman2023letsverify}, or execution \citep{li2022alphacode}, and the unbiased Pass@$k$ estimator we use comes from this literature \citep{chen2021codex}.
Selection also has limits, since optimizing hard against an imperfect reward function eventually degrades true quality \citep{gao2022scaling}.
This body of work measures or exploits the upper tail after training has finished.
\tailrl{} makes that same upper tail the training objective.

\paragraph{Distribution narrowing under expected reward training}
Expected-reward post-training reliably sharpens the policy, which is the failure mode \tailrl{} is designed to resist.
\citet{cui2025entropy} characterize this as an entropy mechanism in which the policy-gradient covariance term stays positive and entropy falls monotonically without intervention.
Measured at large sampling budgets, the effect is visible as a narrowing of what the policy can produce: base models can match or exceed post-trained models on Pass@$k$ at large $k$ \citep{yue2025doesreinforcementlearningreally}, the reachable support changes little \citep{wu2025invisibleleashrlvrescape,nguyen2025reasoningboundaryparadoxreinforcement}, and diversity drops across settings \citep{kirk2024understandingeffectsrlhfllm,dang2025assessing}.
Regularization is one reported cause, since a KL penalty to a reference policy can itself favor mode collapse \citep{gxchen2025klregularizedreinforcementlearningdesigned}.
A parallel line intervenes on entropy directly, through token-level selection \citep{wang2025beyond}, entropy-aware bonuses \citep{cheng2025reasoningexplorationentropyperspective,hao2025rethinkingentropyinterventionsrlvr}, or smoothing \citep{gai2025differentialsmoothingmitigatessharpening}, and these interventions show that coverage can be partly recovered.
\tailrl{} takes a different route and reweights reward thresholds by the inverse of how often the policy reaches them (tail-probability), so rare high-reward threshold keep influence without an added regularizer.

\paragraph{Exploration and sample allocation}
Exploration methods change which rollouts a learner sees, whereas \tailrl{} changes how a fixed group of rollouts is weighted.
Classical approaches add bonuses from visitation counts, prediction error, or information gain \citep{bellemare2016unifyingcountbasedexplorationintrinsic,pathak2017curiosity,burda2018exploration}.
Their language-model counterparts penalize repeated outcomes \citep{song2025outcomebasedexplorationllmreasoning}, reward novelty in representation space \citep{tuyls2025representationbasedexplorationlanguagemodels,dai2025cdecuriositydrivenexplorationefficient}, widen the rollout budget \citep{hu2025brorlscalingreinforcementlearning}, or schedule problems by difficulty \citep{chen2025selfevolvingcurriculumllmreasoning,qu2025pope}.
Closest to us, Reinforce-Ada allocates samples adaptively under a non-linear objective and arrives at a difficulty-prioritizing weighted estimator \citep{xiong2025reinforceadaadaptivesamplingframework}.

\paragraph{Distributional and risk-sensitive reinforcement learning}
Modeling a whole reward or return distribution is not new, so it is worth stating precisely what \tailrl{} does differently.
Distributional reinforcement learning learns the return distribution as a critic, representing it categorically \citep{bellemare2017distributional} or through quantiles \citep{dabney2018distributional,dabney2018iqn,bdr2023book}, and propagates it with a distributional Bellman operator.
Risk-sensitive methods optimize a tail functional of that distribution, most often the conditional value at risk \citep{rockafellar2000optimization,tamar2015optimizing,chow2017risk,clavier2022robustreinforcementlearningdistributional}.
Two differences separate these from \tailrl{}.
First, \tailrl{} carries no critic and performs no bootstrapping: the reward function supplies the reward in one step, and the tail probabilities appear only inside a policy-gradient weight computed from the group itself.
Second, risk measures fix a single risk level, and a fixed level can make a learner blind to rare successes \citep{greenberg2022cesor}, whereas \tailrl{} integrates the log tail-probability uniformly across every threshold.

\paragraph{Threshold decompositions of continuous rewards}
Splitting a continuous reward into a family of binary threshold events is a classical device in supervised learning, and it is the direct ancestor of the decomposition in \cref{fig:teaser}.
Cumulative-link models predict $P(y > j)$ at each threshold \citep{mccullagh1980regression}, and reductions to extended binary classification train one classifier per threshold and reassemble the prediction \citep{frank2001simple,li2006ordinal}.
Neural versions inherit the same structure, with multiple binary outputs \citep{niu2016ordinal} and rank-consistency constraints across cuts \citep{cao2020rank,shi2023corn}.
\citet{patel2026odrpo} carry the idea into policy optimization for discrete rewards.
\tailrl{} is the policy-gradient counterpart for continuous rewards: the threshold $\tau$ plays the role of the ordinal cut, the indicator $\ind_{\{r(x,z)>\tau\}}$ plays the role of the binary label, and the decomposition runs over a continuum rather than a finite set of cuts.

\paragraph{Reinforcement learning in the evaluated domains}
Our four settings each connect to an established line of work.
Bounding-box prediction is normally trained by direct supervision on coordinates, using an $L1$ term, a generalized intersection-over-union term, or the combination adopted by DETR \citep{rezatofighi2019giou,carion2020detr}, against which \cref{sec:exp-imagenet} compares a purely reward-driven policy \citep{russakovsky2015imagenetlargescalevisual,deselaers2010localizing}.
GUI grounding has recently been posed as reinforcement learning with a dense point reward on top of vision-language backbones \citep{bai2025qwen25vl,yang2025gta1,yuan2025segui,liu2025infiguig1} and evaluated on high-resolution professional interfaces \citep{li2025screenspotpro}.
Code generation has long used execution feedback as a reward signal \citep{le2022coderl,shojaee2023ppocoder,liu2023rltf,dou2024stepcoder,gehring2024rlef,zeng2025acecoder}, and optimizing runtime rather than correctness requires both a corpus of slow programs and a deterministic timing oracle \citep{pie_iclr_2024_spotlight,binkert2011gem5}.
That last setting also admits a reward-preserving shortcut, returning the input unchanged, which is an instance of the reward gaming characterized by \citet{skalse2022reward}.
\cref{sec:exp-pie} uses it to separate an objective that settles for a dependable moderate outcome from one that keeps searching for a rare better one.

\section{Supporting Results for the \tailrl{} Gradient Estimator}
\label{app:tailrl-estimator-proofs}

This section supplies the regularity conditions and proofs used in
\cref{sec:estimator}. We continue to work at a fixed input $x$ and assume
$r(x,z)\in[0,1]$.

\subsection{Regularity of the \tailrl{} Gradient}
\label{app:tailrl-regularity}

The derivations of \cref{sec:estimator} rest on three operations.
The score-function identity differentiates the tail-probability under an expectation over the policy's rollouts.
The second operation moves a gradient from outside the threshold integral to inside it.
The third swaps the order of the threshold integral and the rollout expectation.
Each operation exchanges two limits, and each exchange is valid once the quantity being moved is bounded by something with a finite integral.
We state one assumption for each operation, justify each in turn, and then record the lemma that licenses all three.
Throughout, fix an input $x$, let the rollouts take values in a countable set, and let $\Theta$ be an open set of parameters.

\begin{assumption}[Smooth policy]
\label{as:tailrl-smooth}
The support of $\pi_\theta(\cdot\mid x)$ is the same for every $\theta\in \Theta$, each probability $\pi_\theta(z\mid x)$ is differentiable in $\theta$ on $\Theta$, and
\begin{equation}
\sum_z \sup_{\theta\in \Theta}\,\bigl\|\nabla_\theta\pi_\theta(z\mid x)\bigr\| \;<\; \infty.
\label{eq:app-tailrl-smooth-policy}
\end{equation}
\end{assumption}

\Cref{as:tailrl-smooth} is the classical hypothesis behind likelihood-ratio gradient estimators \citep{williams1992,glynn1990likelihood}: it is the interchange condition of \citet{lecuyer1995interchange}, and \citet{mohamed2020montecarlo} survey it as the standing assumption of the score-function estimator class.
It holds by inspection for the policies of this paper: a softmax policy has a fixed finite support, its rollout probabilities are differentiable in $\theta$, and a finite sum of continuous gradient norms is bounded on a bounded $\Theta$.

\begin{assumption}[Finite objective]
\label{as:tailrl-finite}
The tail-likelihood is finite on $\Theta$: $\int_0^1\bigl|\log p_\theta(x,\tau)\bigr|\,d\tau<\infty$ for every $\theta\in \Theta$.
\end{assumption}

\Cref{as:tailrl-finite} is the substantive assumption, and full support is what delivers it.
A softmax policy assigns strictly positive probability to every rollout in its finite support, and the auto-regressive token policies of this paper, softmax at every step with bounded generation length, are of this form.
Let $z^\star$ be a rollout of maximal reward on $x$ and normalize that maximum to $1$.
Every threshold $\tau\in[0,1)$ is then cleared at least by $z^\star$, so the tail-probability is squeezed between two positive constants:
\begin{equation}
0
\;<\;
\pi_\theta(z^\star\mid x)
\;\le\;
p_\theta(x,\tau)
\;\le\;
1,
\qquad
\int_0^1\bigl|\log p_\theta(x,\tau)\bigr|\,d\tau
\;\le\;
-\log\pi_\theta(z^\star\mid x)
\;<\;\infty.
\label{eq:app-tailrl-softmax-check}
\end{equation}
The bound is uniform over any bounded parameter set $\Theta$, since $-\log\pi_\theta(z^\star\mid x)$ is continuous in $\theta$.
If no rollout attains reward $1$, the same squeeze holds over thresholds below the largest attainable reward, and the threshold integral is read over that range.
One consequence is used repeatedly below: the finiteness of the integral forces $p_\theta(x,\tau)>0$ for almost every $\tau$, so the logarithm and every ratio with $p_\theta(x,\tau)$ in its denominator are defined.

\begin{assumption}[Bounded weighted score]
\label{as:tailrl-bounded}
There is a constant $C<\infty$ such that, for almost every $\tau$,
\begin{equation}
\sup_{\theta\in \Theta}\;
\frac{\mathbb{E}_{z\sim\pi_\theta(\cdot\mid x)}
\!\left[\ind_{\{r(x,z)>\tau\}}\,\|S(x,z)\|_2\right]}
{p_\theta(x,\tau)}
\;\le\; C.
\label{eq:app-tailrl-dominated-score}
\end{equation}
\end{assumption}

\Cref{as:tailrl-bounded} follows from the bounded-reward assumption $\sup_{\theta\in \Theta}\sup_{z}\|S(x,z)\|\le C$ routine in policy-gradient convergence analyses \citep{papini2018stochastic,zhang2020global,agarwal2021theory}, since $\mathbb{E}_{z}\!\left[\ind_{\{r(x,z)>\tau\}}\,\|S(x,z)\|\right]\le C\,p_\theta(x,\tau)$.
Scores are bounded for the softmax policies of this paper whenever the logits have bounded gradients.

With the assumptions justified, the lemma states only its conclusion.

\begin{lemma}[Regularity for the \tailrl{} gradient]
\label{lem:tailrl-regularity}
Under \cref{as:tailrl-smooth,as:tailrl-finite,as:tailrl-bounded}, at every $\theta\in \Theta$, the population objective and every finite truncation are differentiable, and the score-function identity and every exchange of a gradient, threshold integral, and rollout expectation in \cref{sec:estimator,sec:estimator-finite} are valid.
\end{lemma}

\begin{proof}
The proof verifies the three operations in order, at an arbitrary $\theta\in \Theta$.

\paragraphi{The score-function identity}
Fix a threshold $\tau$.
The tail-probability is the sum of the policy probabilities over the rollouts that clear it,
\begin{equation}
p_\theta(x,\tau)
=
\sum_z \ind_{\{r(x,z)>\tau\}}\,\pi_\theta(z\mid x).
\label{eq:app-tailrl-pass-rate-density}
\end{equation}
By \cref{as:tailrl-smooth}, every term of the differentiated series is bounded by the summable envelope of \cref{eq:app-tailrl-smooth-policy}, so the gradient of the sum is the sum of the gradients.

\begin{equation}
\nabla_{\theta} p_\theta(x,\tau)
=
\sum_z \ind_{\{r(x,z)>\tau\}}\,\nabla_{\theta}\pi_\theta(z\mid x).
\label{eq:app-tailrl-pass-rate-density-gradient}
\end{equation}
Substituting $\nabla_\theta\pi_\theta(z\mid x)=\pi_\theta(z\mid x)\,S(x,z)$ in \cref{eq:app-tailrl-pass-rate-density-gradient}  turns the differentiated sum into the score-function identity
\begin{equation}
\nabla_\theta p_\theta(x,\tau)
=
\mathbb{E}_{z\sim\pi_\theta(\cdot\mid x)}
\!\left[\ind_{\{r(x,z)>\tau\}}\,S(x,z)\right].
\label{eq:app-tailrl-score-identity}
\end{equation}

\paragraphi{Moving the gradient inside the threshold integral}
The triangle inequality applied to \cref{eq:app-tailrl-score-identity}, followed by \cref{eq:app-tailrl-dominated-score}, bounds the log-gradient at every threshold:
\begin{equation}
\bigl\|\nabla_\theta\log p_\theta(x,\tau)\bigr\|
=
\frac{\|\nabla_\theta p_\theta(x,\tau)\|}{p_\theta(x,\tau)}
\le
\frac{\mathbb{E}_{z}\!\left[\ind_{\{r(x,z)>\tau\}}\,\|S(x,z)\|\right]}{p_\theta(x,\tau)}
\le
C.
\label{eq:app-tailrl-log-gradient-bound}
\end{equation}
The objective is finite by \cref{as:tailrl-finite}, and the gradient of its integrand is bounded by the constant $C$, so the gradient moves inside the threshold integral:
\begin{equation}
\nabla_\theta J_{\mathrm{\tailrl{}}}(\theta;x)
=
\int_0^1
\frac{\nabla_\theta p_{\theta}(x,\tau)}{p_{\theta}(x,\tau)}
\,d\tau.
\label{eq:app-tailrl-population-differentiation}
\end{equation}
The truncations need nothing new.
The order-$N$ threshold multiplier is a finite geometric sum and is therefore bounded,
\begin{equation}
\frac{1-\left(1-p_\theta(x,\tau)\right)^N}{p_\theta(x,\tau)}
=
\sum_{j=0}^{N-1}\left(1-p_\theta(x,\tau)\right)^j
\le N,
\label{eq:app-tailrl-finite-multiplier-bound}
\end{equation}
so the gradient of the truncated integrand is bounded by $NC$ and the same argument differentiates $J_{\mathrm{\tailrl{}}}^{(N)}$.

\paragraphi{Swapping the threshold integral and the rollout expectation}
Integrating \cref{eq:app-tailrl-dominated-score} over the thresholds gives
\begin{equation}
\int_0^1
\mathbb{E}_{z\sim\pi_{\theta}(\cdot\mid x)}
\!\left[
\frac{\ind_{\{r(x,z)>\tau\}}}{p_{\theta}(x,\tau)}
\,\|S(x,z)\|
\right]
d\tau
\le
C
<\infty,
\label{eq:app-tailrl-fubini-bound}
\end{equation}
so the integrand is absolutely integrable and Fubini's theorem permits taking the threshold integral and the rollout expectation in either order.
The truncated multiplier never exceeds $1/p_{\theta}(x,\tau)$, so the truncated weighted score-function is bounded by the integrand of \cref{eq:app-tailrl-fubini-bound} and the same swap applies at every order $N$.
\end{proof}

The countability of the rollout space is inessential: for a policy with densities over a continuous rollout-space, the sum in \cref{eq:app-tailrl-pass-rate-density} becomes an integral against a common reference measure and the proof is unchanged.

%%%%%%%%%%%%%%%%%%%%%%%%%%%%%%%%%%%%%
\subsection{Derivation of the Rollout-Weight Forms}
\label{app:tailrl-weight-derivation}
We derive the population rollout-weight form of the \tailrl{} gradient and its order-$N$ truncation, \cref{eq:finite-tail-weight-gradient}.
Because the tail-likelihood averages per-level log-likelihoods, its gradient is an average of per-level likelihood gradients,
\begin{equation}
\nabla_\theta J_{\mathrm{\tailrl{}}}(\theta;x)
=
\int_0^1
\frac{\nabla_\theta p_\theta(x,\tau)}
{p_\theta(x,\tau)}
\,d\tau
=
\mathbb{E}_{\tau \sim \mathrm{U}[0,1]}\left[\frac{\nabla_\theta p_\theta(x,\tau)}
{p_\theta(x,\tau)}\right].
\label{eq:tailrl-gradient-over-thresholds}
\end{equation}
At a fixed threshold, the log-derivative trick expresses the pass-rate gradient as a score-function expectation:
\begin{align}
\nabla_\theta p_\theta(x,\tau)
=
\mathbb{E}_{z\sim\pi_\theta(\cdot\mid x)}
\!\left[
\ind_{\{r(x,z_j)>\tau\}}S(x,z)
\right].
\label{eq:tailrl-pass-rate-gradient}
\end{align}
Substituting into \cref{eq:tailrl-gradient-over-thresholds} writes the population gradient as a double expectation,
\begin{equation}
\nabla_\theta J_{\mathrm{\tailrl{}}}(\theta;x)
=
\mathbb{E}_{\tau\sim\mathrm{U}[0,1]}\left[\mathbb{E}_{z\sim \pi_{\theta}(\cdot\mid x)}
\!\left[
\frac{\ind_{\{r(x,z)>\tau\}}}
     {p_\theta(x,\tau)}
S(x,z)
\right]\right],
\label{eq:tailrl-double-expectation}
\end{equation}
and the integrability condition of \cref{lem:tailrl-regularity} permits Fubini's theorem to exchange the two.
The inner threshold integral then runs over exactly the thresholds cleared by the rollout, gives us  \cref{eq:tailrl-population-gradient} as restated below.
\begin{equation}
\nabla_\theta J_{\mathrm{\tailrl{}}}(\theta;x)
=
\mathbb{E}_z
\!\left[
\left(
\int_0^{r(x,z)}
\frac{1}
     {p_\theta(x, \tau)}
\,d\tau
\right)
S(x,z)
\right],
\label{eq:tailrl-population-gradient}
\end{equation}
The order-$N$ member follows the same route.
Differentiating \cref{eq:tailrl-harmonic-family} and summing the finite geometric series gives
\begin{equation}
\nabla_\theta J_{\mathrm{\tailrl{}}}^{(N)}(\theta;x)
=
\int_0^1 \left(
\frac{1-\left(1-p_\theta(x, \tau)\right)^N}
     {p_\theta(x, \tau)} \right)
\nabla_\theta p_\theta(x, \tau)
\,d\tau,
\label{eq:tailrl-order-n-gradient}
\end{equation}
and substituting \cref{eq:tailrl-pass-rate-gradient} and exchanging the threshold integral with the rollout expectation restricts the integral to the thresholds below the rollout's reward,
\begin{equation}
\nabla_\theta J_{\mathrm{\tailrl{}}}^{(N)}(\theta;x)
=
\mathbb{E}_z
\!\left[
\left(
\int_0^{r(x,z)}
\frac{1-\left(1-p_\theta(x, \tau)\right)^N}
     {p_\theta(x, \tau)}
\,d\tau
\right)
S(x,z)
\right],
\label{eq:tailrl-truncated-weighted-gradient}
\end{equation}
which is \cref{eq:finite-tail-weight-gradient}.

\subsection{Best-of-\texorpdfstring{$k$}{k} Decompositions and Probabilistic Interpretations}
\label{app:tailrl-decompositions}

We prove the finite and population Best-of-$k$ decompositions stated in \cref{sec:tailrl-objective}.
It also shows that threshold weighting is equivalent to rescaling the reward axis.
We continue to work at a fixed input $x$ and assume $r(x,z)\in[0,1]$.
The gradient statements use the regularity conditions of \cref{lem:tailrl-regularity}.

\subsubsection{Layer-Cake Form of Best-of-\texorpdfstring{$k$}{k}}
\label{app:tailrl-bestk-layer-cake}

Let $z_1,\ldots,z_k\overset{\mathrm{i.i.d.}}{\sim}\pi_\theta(\cdot\mid x)$ and write
$R_i:=r(x,z_i)$ and $M_k:=\max_{1\leq i\leq k}R_i$.
Since $M_k\in[0,1]$, the layer-cake identity gives
\begin{equation}
\mathbb E[M_k]
=
\int_0^1 \Pr(M_k>\tau)\,d\tau.
\label{eq:app-tailrl-layer-cake-start}
\end{equation}
The maximum fails to clear $\tau$ exactly when all $k$ rollouts fail it.
By independence,
\begin{equation}
\Pr(M_k\leq\tau)
=
\left(1-p_\theta(x,\tau)\right)^k.
\label{eq:app-tailrl-max-cdf}
\end{equation}
Therefore,
\begin{equation}
\text{Best-of-}k(\theta;x)
=
\int_0^1
\left[
1-\left(1-p_\theta(x,\tau)\right)^k
\right]d\tau.
\label{eq:app-bok-c3}
\end{equation}

\subsubsection{Finite Harmonic Decomposition}
\label{app:tailrl-harmonic-finite}

\begin{proposition}[Harmonic Best-of-$k$ expansion of the finite objective]
\label{prop:harmonic-c3}
For every truncation order $T\geq1$,
\begin{equation}
J_{\mathrm{\tailrl{}}}^{(T)}(\theta;x)
=
\sum_{k=1}^{T}
\frac{\text{Best-of-}k(\theta;x)-1}{k}.
\label{eq:app-harmonic-finite-c3}
\end{equation}
\end{proposition}

\begin{proof}
Exchange the finite sum in \cref{eq:tailrl-harmonic-family} with the threshold integral:
\begin{align}
J_{\mathrm{\tailrl{}}}^{(T)}(\theta;x)
&=
-\sum_{k=1}^{T}
\frac{1}{k}
\int_0^1
\left(1-p_\theta(x,\tau)\right)^k
\,d\tau
\nonumber\\
&=
\sum_{k=1}^{T}
\frac{1}{k}
\left(
\int_0^1
\left[1-\left(1-p_\theta(x,\tau)\right)^k\right]
\,d\tau
-1
\right).
\end{align}
The integral in parentheses is $\text{Best-of-}k(\theta;x)$ by \cref{eq:app-bok-c3}.
\end{proof}

\begin{proposition}[Harmonic Best-of-$k$ expansion of the finite gradient]
\label{prop:harmonic-grad-c3}
Under \cref{lem:tailrl-regularity}, for every $T\geq1$,
\begin{equation}
\nabla_\theta J_{\mathrm{\tailrl{}}}^{(T)}(\theta;x)
=
\sum_{k=1}^{T}
\frac{1}{k}
\nabla_\theta\text{Best-of-}k(\theta;x).
\label{eq:app-harmonic-grad-finite}
\end{equation}
Equivalently,
\begin{equation}
\nabla_\theta J_{\mathrm{\tailrl{}}}^{(T)}(\theta;x)
=
\int_0^1
\frac{1-\left(1-p_\theta(x,\tau)\right)^T}
{p_\theta(x,\tau)}
\nabla_\theta p_\theta(x,\tau)
\,d\tau,
\label{eq:app-harmonic-grad-c3}
\end{equation}
where the ratio is interpreted by continuity as $T$ when $p_\theta(x,\tau)=0$.
\end{proposition}

\begin{proof}
Differentiating \cref{eq:app-harmonic-finite-c3} term by term gives \cref{eq:app-harmonic-grad-finite}.
Differentiating \cref{eq:tailrl-harmonic-family} under the integral gives
\begin{align}
\nabla_\theta J_{\mathrm{\tailrl{}}}^{(T)}(\theta;x)
&=
\int_0^1
\left[
\sum_{k=1}^{T}
\left(1-p_\theta(x,\tau)\right)^{k-1}
\right]
\nabla_\theta p_\theta(x,\tau)
\,d\tau
\nonumber\\
&=
\int_0^1
\frac{1-\left(1-p_\theta(x,\tau)\right)^T}
{p_\theta(x,\tau)}
\nabla_\theta p_\theta(x,\tau)
\,d\tau,
\end{align}
where the second equality uses the finite geometric-series identity.
\end{proof}

At $T=1$, \cref{eq:app-harmonic-finite-c3} gives
\begin{equation}
J_{\mathrm{\tailrl{}}}^{(1)}(\theta;x)
=
\text{Best-of-}1(\theta;x)-1
=
J_{\mathrm{RL}}(\theta;x)-1.
\label{eq:app-tailrl-first-order}
\end{equation}
Also, since $0\leq(1-p_\theta(x,\tau))^k\leq1$,
\begin{equation}
-H_T
\leq
J_{\mathrm{\tailrl{}}}^{(T)}(\theta;x)
\leq
0,
\qquad
H_T:=\sum_{k=1}^{T}\frac{1}{k}.
\label{eq:app-tailrl-finite-bounds}
\end{equation}

\subsubsection{Population Harmonic Decomposition}
\label{app:tailrl-harmonic-population}

\begin{proof}[Proof of \cref{thm:tailrl-harmonic-decomposition}]
For each threshold, define
\begin{equation}
s_T(\tau)
:=
\sum_{k=1}^{T}
\frac{\left(1-p_\theta(x,\tau)\right)^k}{k}.
\label{eq:app-tailrl-maclaurin-partial}
\end{equation}
The sequence $s_T(\tau)$ is nondecreasing in $T$.
The Maclaurin series gives the extended-real pointwise limit
\begin{equation}
\lim_{T\to\infty}s_T(\tau)
=
-\log p_\theta(x,\tau).
\label{eq:app-tailrl-maclaurin-limit}
\end{equation}
The monotone convergence theorem therefore gives
\begin{align}
\lim_{T\to\infty}
J_{\mathrm{\tailrl{}}}^{(T)}(\theta;x)
&=
-\int_0^1
\lim_{T\to\infty}s_T(\tau)
\,d\tau
\nonumber\\
&=
\int_0^1
\log p_\theta(x,\tau)
\,d\tau
=
J_{\mathrm{\tailrl{}}}(\theta;x).
\label{eq:app-tailrl-population-objective-limit}
\end{align}
Combining this limit with \cref{eq:app-harmonic-finite-c3} yields
\begin{equation}
J_{\mathrm{\tailrl{}}}(\theta;x)
=
\sum_{k=1}^{\infty}
\frac{\text{Best-of-}k(\theta;x)-1}{k}
\end{equation}
whenever the population objective is finite.

For the gradients, define
\begin{equation}
w_T(p)
:=
\frac{1-(1-p)^T}{p}
=
\sum_{j=0}^{T-1}(1-p)^j.
\label{eq:app-tailrl-finite-weight}
\end{equation}
For every $p>0$, $w_T(p)\uparrow1/p$, and $0\leq w_T(p)\leq1/p$.
By \cref{lem:tailrl-regularity},
\begin{equation}
\frac{\left\|\nabla_\theta p_\theta(x,\tau)\right\|}
{p_\theta(x,\tau)}
\leq C
\end{equation}
for almost every threshold and some finite $C$.
Dominated convergence applied to \cref{eq:app-harmonic-grad-c3} gives
\begin{align}
\lim_{T\to\infty}
\nabla_\theta J_{\mathrm{\tailrl{}}}^{(T)}(\theta;x)
&=
\int_0^1
\frac{\nabla_\theta p_\theta(x,\tau)}
{p_\theta(x,\tau)}
\,d\tau
\nonumber\\
&=
\nabla_\theta J_{\mathrm{\tailrl{}}}(\theta;x).
\end{align}
The finite gradient identity then identifies this limit with
$\sum_{k=1}^{\infty}k^{-1}\nabla_\theta\text{Best-of-}k(\theta;x)$.
\end{proof}

\subsubsection{Threshold Weighting as Reward Reparameterization}
\label{app:tailrl-threshold-weighting}

The uniform threshold distribution in \cref{eq:tailrl-objective} is the simplest member of a weighted family.
Let $w:[0,1]\to(0,\infty)$ be a fixed density satisfying $\int_0^1w(\tau)\,d\tau=1$, and define
\begin{equation}
J_w(\theta;x)
:=
\int_0^1
w(\tau)
\log p_\theta(x,\tau)
\,d\tau.
\label{eq:app-tailrl-weighted-objective}
\end{equation}

\begin{proposition}[Threshold weighting is reward-axis rescaling]
\label{prop:tailrl-threshold-reparameterization}
Define
\begin{equation}
F_w(t)
:=
\int_0^t w(s)\,ds,
\qquad
\widetilde r(x,z):=F_w\!\left(r(x,z)\right).
\label{eq:app-tailrl-reward-transform}
\end{equation}
Then $F_w$ is a strictly increasing map from $[0,1]$ to $[0,1]$, and
\begin{equation}
J_w(\theta;x)
=
\int_0^1
\log
\Pr_{z\sim\pi_\theta(\cdot\mid x)}
\!\left(\widetilde r(x,z)>u\right)
\,du.
\label{eq:app-tailrl-weighting-equivalence}
\end{equation}
Thus, weighted TailRL on $r$ is exactly uniform TailRL on the monotone rescaling $\widetilde r=F_w(r)$.
\end{proposition}

\begin{proof}
Since $w$ is positive and integrates to one, $F_w$ is strictly increasing with $F_w(0)=0$ and $F_w(1)=1$.
For $u\in[0,1]$,
\begin{align}
\Pr\!\left(\widetilde r(x,z)>u\right)
&=
\Pr\!\left(F_w(r(x,z))>u\right)
\nonumber\\
&=
\Pr\!\left(r(x,z)>F_w^{-1}(u)\right)
=
p_\theta\!\left(x,F_w^{-1}(u)\right).
\end{align}
Therefore,
\begin{align}
\int_0^1
\log\Pr\!\left(\widetilde r(x,z)>u\right)du
&=
\int_0^1
\log p_\theta\!\left(x,F_w^{-1}(u)\right)du
\nonumber\\
&=
\int_0^1
w(\tau)\log p_\theta(x,\tau)d\tau,
\end{align}
where the last equality uses the substitution $u=F_w(\tau)$.
\end{proof}

The proposition shows that choosing a threshold density is equivalent to choosing a monotone coordinate system for reward quality.
Uniform weighting corresponds to the original normalized reward axis and introduces no additional function or hyperparameter.
A nonuniform weighting may still be useful when a task provides a preferred reward scale, but selecting or learning that scale is outside the scope of this work.

\subsubsection{Recovery of MaxRL on Binary Rewards}
\label{app:tailrl-objective-binary-recovery}

\begin{proof}[Proof of \cref{eq:tailrl-maxrl-population}]
If $r(x,z)\in\{0,1\}$, then for every $\tau\in[0,1)$,
\begin{equation}
\{r(x,z)>\tau\}
=
\{r(x,z)=1\}.
\end{equation}
Therefore $p_\theta(x,\tau)=q_\theta(x)$ for every nontrivial threshold, and
\begin{equation}
J_{\mathrm{\tailrl{}}}(\theta;x)
=
\int_0^1\log q_\theta(x)\,d\tau
=
\log q_\theta(x).
\end{equation}
The same substitution in the finite objective gives
\begin{equation}
J_{\mathrm{\tailrl{}}}^{(T)}(\theta;x)
=
-\sum_{k=1}^{T}
\frac{\left(1-q_\theta(x)\right)^k}{k},
\end{equation}
which is the order-$T$ Maclaurin truncation of MaxRL.
Finally, the maximum of $k$ binary rewards equals one exactly when at least one rollout succeeds, so
$\text{Best-of-}k(\theta;x)=\mathrm{Pass@}k(\theta;x)$.
\end{proof}

\subsection{Harmonic Best-of-\texorpdfstring{$k$}{k} Expansion of the Tail-Likelihood}
\label{app:tailrl-harmonic}

This section proves that the truncated objective is a partial sum of Best-of-$k$ objectives with harmonic coefficients, and that the tail-likelihood is the full series. Since the maximum of $k$ rewards clears a threshold unless all $k$ rollouts fail it, by layer cake identity,
\begin{equation}
\text{Best-of-}k(\theta;x)
=
\mathbb{E}\Big[\max_{1\le i\le k} r(x,z_i)\Big]
=
\int_0^1 \Big[1-\big(1-p_\theta(x,\tau)\big)^{k}\Big]\,d\tau .
\label{eq:app-bok}
\end{equation}

\begin{proposition}[Harmonic Best-of-$k$ expansion of the objective]
\label{prop:harmonic}
Fix an input $x$ and a parameter $\theta$.
For every truncation order $T \ge 1$,
\begin{equation}
J^{(T)}_{\mathrm{\tailrl{}}}(\theta;x)
=
\sum_{k=1}^{T} \frac{1}{k}\Big(\text{\emph{Best-of-}}k(\theta;x) - 1\Big),
\label{eq:app-harmonic-finite}
\end{equation}
and if $J_{\mathrm{\tailrl{}}}(\theta;x) > -\infty$, the same identity holds for the full series,
\begin{equation}
J_{\mathrm{\tailrl{}}}(\theta;x)
=
\sum_{k=1}^{\infty} \frac{1}{k}\Big(\text{\emph{Best-of-}}k(\theta;x) - 1\Big).
\label{eq:app-harmonic}
\end{equation}
\end{proposition}

\begin{proof}
We prove the finite identity first and obtain the series as its limit.
Start from the definition of the truncated objective in \cref{eq:tailrl-harmonic-family} and exchange the finite sum with the threshold integral:
\begin{equation}
J^{(T)}_{\mathrm{\tailrl{}}}(\theta;x)
=
-\int_0^1 \sum_{k=1}^{T} \frac{\big(1-p_\theta(x,\tau)\big)^{k}}{k}\,d\tau
=
-\sum_{k=1}^{T} \frac{1}{k} \int_0^1 \big(1-p_\theta(x,\tau)\big)^{k}\,d\tau .
\label{eq:app-harmonic-swap}
\end{equation}
Each integral on the right is one minus a Best-of-$k$ value by \cref{eq:app-bok}, and substituting it in gives \cref{eq:app-harmonic-finite}.

For the series, we let $T \to \infty$ on both sides of \cref{eq:app-harmonic-finite} and identify the two limits.

We first show the left side converges to the tail-likelihood.
Since $J_{\mathrm{\tailrl{}}}(\theta;x) > -\infty$, the tail-probability $p_\theta(x,\tau)$ is positive for almost every threshold.
Fix such a $\tau$, so that $0 \le 1-p_\theta(x,\tau) < 1$.
The Maclaurin series of the logarithm, evaluated at $1-p_\theta(x,\tau)$, gives the pointwise limit of the partial sums:
\begin{equation}
\lim_{T\to\infty}
\sum_{k=1}^{T} \frac{\big(1-p_\theta(x,\tau)\big)^{k}}{k}
=
\sum_{k=1}^{\infty} \frac{\big(1-p_\theta(x,\tau)\big)^{k}}{k}
=
-\log p_\theta(x,\tau).
\label{eq:app-maclaurin}
\end{equation}
The convergence is monotone: every added term $\big(1-p_\theta(x,\tau)\big)^{T+1}/(T+1)$ is nonnegative, so the partial sums only grow with $T$.
The integrands in the definition \cref{eq:tailrl-harmonic-family} of $J^{(T)}_{\mathrm{\tailrl{}}}$ are exactly these partial sums.
Because they are nonnegative and increasing in $T$, the limit of their integrals is the integral of their limit, which is the monotone convergence theorem.
Therefore
\begin{equation}
\lim_{T\to\infty} J^{(T)}_{\mathrm{\tailrl{}}}(\theta;x)
=
-\int_0^1 \lim_{T\to\infty} \sum_{k=1}^{T} \frac{\big(1-p_\theta(x,\tau)\big)^{k}}{k}\,d\tau
=
\int_0^1 \log p_\theta(x,\tau)\,d\tau
=
J_{\mathrm{\tailrl{}}}(\theta;x).
\label{eq:app-harmonic-limit}
\end{equation}

The right side needs no computation.
For every $T$, the right side of \cref{eq:app-harmonic-finite} is the $T$-th partial sum of the series in \cref{eq:app-harmonic}.
Its limit exists because it equals the left side at every $T$, and by \cref{eq:app-harmonic-limit} that limit is $J_{\mathrm{\tailrl{}}}(\theta;x)$.
This is \cref{eq:app-harmonic}.
\end{proof}

Every term of the series is nonpositive, since no Best-of-$k$ value exceeds $1$, so the partial sums decrease monotonically to the tail-likelihood: each truncation is an upper bound on $J_{\mathrm{\tailrl{}}}$, tightening as $T$ grows.
At the other end, $J^{(1)}_{\mathrm{\tailrl{}}}(\theta;x) = \text{Best-of-}1(\theta;x) - 1 = J_{\mathrm{RL}}(\theta;x) - 1$, so the first member of the family is expected reward reinforcement learning up to an additive constant and shares its gradient.

\begin{proposition}[Harmonic Best-of-$k$ expansion of the gradient]
\label{prop:harmonic-grad}
Under \cref{as:tailrl-smooth,as:tailrl-finite,as:tailrl-bounded}, for every $\theta\in \Theta$ and every $T \ge 1$,
\begin{equation}
\nabla_\theta J^{(T)}_{\mathrm{\tailrl{}}}(\theta;x)
=
\sum_{k=1}^{T} \frac{1}{k}\,\nabla_\theta\, \text{\emph{Best-of-}}k(\theta;x)
=
\int_0^1 \frac{1-\big(1-p_\theta(x,\tau)\big)^{T}}{p_\theta(x,\tau)}\,\nabla_\theta p_\theta(x,\tau)\,d\tau .
\label{eq:app-harmonic-grad}
\end{equation}
\end{proposition}

\begin{proof}
\Cref{lem:tailrl-regularity} moves the gradient inside the threshold integral of \cref{eq:app-bok}, and the chain rule gives
\begin{equation}
\nabla_\theta\, \text{Best-of-}k(\theta;x)
=
k \int_0^1 \big(1-p_\theta(x,\tau)\big)^{k-1}\,\nabla_\theta p_\theta(x,\tau)\,d\tau .
\label{eq:app-bok-grad}
\end{equation}
Multiply \cref{eq:app-bok-grad} by $1/k$, sum over $k = 1, \ldots, T$, and exchange the finite sum with the integral:
\begin{equation}
\sum_{k=1}^{T} \frac{1}{k}\,\nabla_\theta\, \text{Best-of-}k(\theta;x)
=
\int_0^1 \Bigg[\sum_{k=1}^{T} \big(1-p_\theta(x,\tau)\big)^{k-1}\Bigg]\,\nabla_\theta p_\theta(x,\tau)\,d\tau .
\label{eq:app-harmonic-grad-sum}
\end{equation}
The bracket is a finite geometric sum, $\sum_{k=1}^{T}(1-p)^{k-1} = \big(1-(1-p)^{T}\big)/p$, which gives the right side of \cref{eq:app-harmonic-grad}.
The left equality is \cref{lem:tailrl-regularity} again, differentiating \cref{eq:app-harmonic-finite} term by term.
Every coefficient $1/k$ is positive, so the truncated gradient is a fixed, positively weighted combination of the Best-of-$k$ gradients for $k = 1, \ldots, T$.
\end{proof}

\subsection{Unbiasedness of the finite rollout Estimator}
\label{app:tailrl-finite-unbiasedness}

This section proves \cref{thm:tailrl-finite-unbiasedness}, restated here in full.

\noindent
\textbf{\Cref{thm:tailrl-finite-unbiasedness}} \textbf{(Unbiasedness of the \tailrl{} estimator, restated).}
The unbiased estimator of  $\nabla_\theta J_{\mathrm{\tailrl{}}}^{(N)}(\theta;x)$ can be expressed as,
% Under \cref{lem:tailrl-regularity},
\begin{equation}
\gtailrl(x)
:=
\sum_{i=1}^N
\omega(r_i)\,S_i,
\qquad
\omega(r_i)
:=
\int_0^{r_i}\frac{d\tau}{\sum_{j=1}^{N}\ind_{\{r_j>\tau\}}}
\label{eq:app-tailrl-weight-restated}
\end{equation}
\emph{Under \cref{as:tailrl-smooth,as:tailrl-finite,as:tailrl-bounded}, the estimator $\gtailrl(x):=\sum_{i=1}^N\omega\bigl(r(x,z_i)\bigr)\,S(x,z_i)$ is unbiased for the order-$N$ truncated gradient:}
\begin{equation}
\mathbb{E}_{z_{1:N}}\!\left[\gtailrl(x)\right]
=
\nabla_\theta J_{\mathrm{\tailrl{}}}^{(N)}(\theta;x).
\label{eq:app-tailrl-unbiasedness-restated}
\end{equation}

\begin{proof}[Proof of \cref{thm:tailrl-finite-unbiasedness}]

\noindent
The weight \cref{eq:app-tailrl-weight-restated} for rollout $i$ over its integration range clears the threshold, so its denominator is at least one.
The first step extends the integral from $[0,r(x,z_i))$ to $[0,1)$ with an integrand that vanishes beyond $r(x,z_i)$.
We cap the denominator below at one; the cap is active only where the numerator already vanishes:
\begin{equation}
\omega\bigl(r(x,z_i)\bigr)
=
\int_0^{r_i}\frac{d\tau}{\sum_{j=1}^{N}\ind_{\{r_j>\tau\}}}
=
\int_0^1
\frac{\ind_{\{r(x,z_i)>\tau\}}}{\max\Bigl(1,\;\sum_{j=1}^{N}\ind_{\{r(x,z_j)>\tau\}}\Bigr)}
\,d\tau.
\label{eq:app-tailrl-weight-indicator}
\end{equation}

Every ratio from here on carries this capped denominator, so no expression in the proof is ever indeterminate.
Multiplying by $S(x,z_i)$ with their weights $\omega(r(x, z_i))$ and summing over the group, we make use of the score-functions being independent of the thresholds, so we swap the order of integration and summation:
\begin{equation}
\gtailrl(x)
=
\int_0^1
\frac{\sum_{i=1}^{N}\ind_{\{r(x,z_i)>\tau\}}\,S(x,z_i)}{\max\Bigl(1,\;\sum_{j=1}^{N}\ind_{\{r(x,z_j)>\tau\}}\Bigr)}
\,d\tau.
\label{eq:app-tailrl-threshold-form}
\end{equation}

\paragraphi{The chain of expectations}
Take the expectation over the training group.
\Cref{lem:tailrl-regularity} moves it inside the integral over thresholds $\tau$, and linearity then moves it inside the finite sum over rollouts:
\begin{align}
\mathbb{E}_{z_{1:N}}\!\left[\gtailrl(x)\right]
&=
\int_0^1
\mathbb{E}_{z_{1:N}}\!\left[ \left ( \sum_{i=1}^{N}
\frac{\ind_{\{r(x,z_i)>\tau\}}\,S(x,z_i)}{\max\Bigl(1,\;\sum_{j=1}^{N}\ind_{\{r(x,z_j)>\tau\}}\Bigr)}
\right )\right]
d\tau
\nonumber\\
&=
\int_0^1
\sum_{i=1}^{N} \left (
\mathbb{E}_{z_{1:N}}\!\left[
\frac{\ind_{\{r(x,z_i)>\tau\}}\,S(x,z_i)}{\max\Bigl(1,\;\sum_{j=1}^{N}\ind_{\{r(x,z_j)>\tau\}}\Bigr)}
\right]  \right )
d\tau.
\label{eq:app-tailrl-exchange}
\end{align}
By regularity condition \cref{lem:tailrl-regularity}, this integral is bounded. 
The next tool  we use is the tower property of conditional expectation: for random variables $X$ and $Y$ on the same support,
\begin{equation}
\mathbb{E}_{X}\!\left[X\right]
=
\mathbb{E}_{Y}\!\left[\,\mathbb{E}_{X}\!\left[X\mid Y\right]\right],
\label{eq:app-tailrl-tower-property}
\end{equation}
where the inner expectation averages $X$ with $Y$ held at its realized value and the outer expectation averages the result over $Y$.
We apply it with the substitution
\begin{equation}
X
:=
\frac{\ind_{\{r(x,z_i)>\tau\}}\,S(x,z_i)}{\max\Bigl(1,\;\sum_{j=1}^{N}\ind_{\{r(x,z_j)>\tau\}}\Bigr)},
\qquad
Y
:=
\{ \ind_{\{r(x,z_i)>\tau\}}\}_{i=1}^{N},
\label{eq:app-tailrl-xy-def}
\end{equation}
the clearance pattern of the group at threshold $\tau$; both are functions of the training group $z_{1:N}$, so every expectation below averages over $z_{1:N}$.
The tower property \cref{eq:app-tailrl-tower-property} gives
\begin{equation}
\mathbb{E}_{z_{1:N}}\!\left[
\frac{\ind_{\{r(x,z_i)>\tau\}}\,S(x,z_i)}{\max\Bigl(1,\;\sum_{j=1}^{N}\ind_{\{r(x,z_j)>\tau\}}\Bigr)}
\right]
=
\mathbb{E}_{z_{1:N}}\!\left[\,
\mathbb{E}_{z_{1:N}}\!\left[\left.
\frac{\ind_{\{r(x,z_i)>\tau\}}\,S(x,z_i)}{\max\Bigl(1,\;\sum_{j=1}^{N}\ind_{\{r(x,z_j)>\tau\}}\Bigr)}
\,\right|\,\{ \ind_{\{r(x,z_i)>\tau\}}\}_{i=1}^{N}\right]
\right].
\label{eq:app-tailrl-step-tower}
\end{equation}
Before taking the inner expectation, note one pointwise identity.
An indicator equals its own square, so multiplying the fraction by a second copy
of rollout $i$'s indicator changes nothing:
\begin{equation}
\frac{\ind_{\{r(x,z_i)>\tau\}}}{\max\Bigl(1,\;\sum_{j=1}^{N}\ind_{\{r(x,z_j)>\tau\}}\Bigr)}
\;\ind_{\{r(x,z_i)>\tau\}}
=
\frac{\ind_{\{r(x,z_i)>\tau\}}}{\max\Bigl(1,\;\sum_{j=1}^{N}\ind_{\{r(x,z_j)>\tau\}}\Bigr)}.
\label{eq:app-tailrl-absorb}
\end{equation}
So $X$ factors into the fraction times the conditional score:
\begin{equation}
X
=
\frac{\ind_{\{r(x,z_i)>\tau\}}}{\max\Bigl(1,\;\sum_{j=1}^{N}\ind_{\{r(x,z_j)>\tau\}}\Bigr)}
\;\ind_{\{r(x,z_i)>\tau\}}\,S(x,z_i),
\qquad
Y
=
\bigl\{\ind_{\{r(x,z_j)>\tau\}}\bigr\}_{j=1}^{N}.
\label{eq:app-tailrl-x-split}
\end{equation}
Conditioning on $\bigl\{\ind_{\{r(x,z_j)>\tau\}}\bigr\}_{j=1}^{N}$ determines all $N$ indicator functions, so the expression simplifies as:
\begin{equation}
\mathbb{E}_{z_{1:N}}\!\left[\left.
\frac{\ind_{\{r(x,z_i)>\tau\}}\,S(x,z_i)}{\max\Bigl(1,\;\sum_{j=1}^{N}\ind_{\{r(x,z_j)>\tau\}}\Bigr)}
\,\right|\,Y\right]
=
\frac{\ind_{\{r(x,z_i)>\tau\}}}{\max\Bigl(1,\;\sum_{j=1}^{N}\ind_{\{r(x,z_j)>\tau\}}\Bigr)}
\;
\mathbb{E}_{z_{1:N}}\!\left[\,\ind_{\{r(x,z_i)>\tau\}}\,S(x,z_i)\,\middle|\,Y\right].
\label{eq:app-tailrl-step-pullout}
\end{equation}
It remains to compute the conditional mean of the gated score.
The rollouts are independent, so the only part of $Y$ that constrains $z_i$ is its own
indicator, and we split on the two values it can take.
If rollout $i$ clears, $z_i$ follows the policy restricted to the clearing set,
$\pi_\theta(z\mid x)\,\ind_{\{r(x,z)>\tau\}}/p_\theta(x,\tau)$; the gate is one everywhere
on this support, so the mean is the clearing-rollout mean score, whose norm is at most $C$
by \cref{as:tailrl-bounded}.
If rollout $i$ fails, $z_i$ is supported on $r(x,z)\le\tau$, where the gate is zero, so the
mean is exactly the zero vector.
Both branches are finite, and together they give:
\begin{equation}
\mathbb{E}_{z_{1:N}}\!\left[\,\ind_{\{r(x,z_i)>\tau\}}\,S(x,z_i)\,\middle|\,Y\right]
=
\ind_{\{r(x,z_i)>\tau\}}\;
\mathbb{E}_{z\sim\pi_\theta(\cdot\mid x)}\!\left[S(x,z)\mid r(x,z)>\tau\right].
\label{eq:app-tailrl-gated-mean}
\end{equation}
Substituting \cref{eq:app-tailrl-gated-mean} into \cref{eq:app-tailrl-step-pullout} and applying \cref{eq:app-tailrl-absorb} once more removes the extra indicator:
\begin{equation}
\mathbb{E}_{z_{1:N}}\!\left[\left.
\frac{\ind_{\{r(x,z_i)>\tau\}}\,S(x,z_i)}{\max\Bigl(1,\;\sum_{j=1}^{N}\ind_{\{r(x,z_j)>\tau\}}\Bigr)}
\,\right|\,Y\right]
=
\frac{\ind_{\{r(x,z_i)>\tau\}}}{\max\Bigl(1,\;\sum_{j=1}^{N}\ind_{\{r(x,z_j)>\tau\}}\Bigr)}
\;
\mathbb{E}_{z\sim\pi_\theta(\cdot\mid x)}\!\left[S(x,z)\mid r(x,z)>\tau\right].
\label{eq:app-tailrl-step-branch}
\end{equation}
Finally the outer expectation of \cref{eq:app-tailrl-step-tower} wraps over \cref{eq:app-tailrl-step-branch}.
The clearing-rollout mean score-function is a fixed vector, so by linearity the outer expectation acts only on the fraction of indicators:
\begin{equation}
\begin{split}
\mathbb{E}_{z_{1:N}}\!\left[
\frac{\ind_{\{r(x,z_i)>\tau\}}\,S(x,z_i)}{\max\Bigl(1,\;\sum_{j=1}^{N}\ind_{\{r(x,z_j)>\tau\}}\Bigr)}
\right]
&=
\mathbb{E}_{z\sim\pi_\theta(\cdot\mid x)}\!\left[S(x,z)\mid r(x,z)>\tau\right]
\\[2pt]
&\quad\times\;
\mathbb{E}_{z_{1:N}}\!\left[
\frac{\ind_{\{r(x,z_i)>\tau\}}}{\max\Bigl(1,\;\sum_{j=1}^{N}\ind_{\{r(x,z_j)>\tau\}}\Bigr)}
\right].
\end{split}
\label{eq:app-tailrl-per-term}
\end{equation}
Substituting \cref{eq:app-tailrl-per-term} into \cref{eq:app-tailrl-exchange}, summing over $i$, and recombining the $N$ terms under one expectation by linearity leaves the sum of the fractions, which share one denominator and collapse to a single ratio of the same count:
\begin{equation}
\sum_{i=1}^{N}
\frac{\ind_{\{r(x,z_i)>\tau\}}}{\max\Bigl(1,\;\sum_{j=1}^{N}\ind_{\{r(x,z_j)>\tau\}}\Bigr)}
=
\frac{\sum_{j=1}^{N}\ind_{\{r(x,z_j)>\tau\}}}{\max\Bigl(1,\;\sum_{j=1}^{N}\ind_{\{r(x,z_j)>\tau\}}\Bigr)}
=
\ind_{\{\sum_{j=1}^{N}\ind_{\{r(x,z_j)>\tau\}}\geq 1\}}.
\label{eq:app-tailrl-share-identity}
\end{equation}
The last equality is checked outcome by outcome: if the number of rollouts attaining a reward above the threshold is $\geq1$, the cap is inactive and the ratio is $=1$; if no rollout clears, the ratio is $0/1=0$.
The expectation of this event indicator is its probability, and the $N$ rollouts fail the threshold independently, each with probability $1-p_\theta(x,\tau)$:
\begin{equation}
\mathbb{E}_{z_{1:N}}\!\left[\ind_{\{\sum_{j=1}^{N}\ind_{\{r(x,z_j)>\tau\}}\geq 1\}}\right]
=
1-\left(1-p_\theta(x,\tau)\right)^{N}.
\label{eq:app-tailrl-survive-prob}
\end{equation}
The remaining constant is evaluated by the definition of conditional expectation given an event, $\mathbb{E}[S\mid A]=\mathbb{E}[S\,\ind_A]/\Pr(A)$, with $A=\{r(x,z)>\tau\}$ and $\Pr_{z\sim\pi_\theta(\cdot\mid x)}(A)=p_\theta(x,\tau)$; its numerator is the score-function identity \cref{eq:app-tailrl-score-identity}:
\begin{equation}
\mathbb{E}_{z\sim\pi_\theta(\cdot\mid x)}\!\left[S(x,z)\mid r(x,z)>\tau\right]
=
\frac{\mathbb{E}_{z\sim\pi_\theta(\cdot\mid x)}\!\left[\ind_{\{r(x,z)>\tau\}}\,S(x,z)\right]}{p_\theta(x,\tau)}
=
\frac{\nabla_\theta p_\theta(x,\tau)}{p_\theta(x,\tau)}.
\label{eq:app-tailrl-conditional-mean}
\end{equation}

\paragraphi{Assembling the chain}
Substituting \cref{eq:app-tailrl-per-term,eq:app-tailrl-share-identity,eq:app-tailrl-survive-prob,eq:app-tailrl-conditional-mean} into \cref{eq:app-tailrl-exchange} yields the truncated gradient of \cref{eq:tailrl-order-n-gradient}:
\begin{equation}
\mathbb{E}_{z_{1:N}}\!\left[\gtailrl(x)\right]
=
\int_0^1
\Bigl(1-\left(1-p_\theta(x,\tau)\right)^{N}\Bigr)
\frac{\nabla_\theta p_\theta(x,\tau)}{p_\theta(x,\tau)}
\,d\tau
=
\nabla_\theta J_{\mathrm{\tailrl{}}}^{(N)}(\theta;x).
\label{eq:app-tailrl-finite-unbiasedness}
\end{equation}
\end{proof}

The proof identifies the estimator as MaxRL run at every threshold on one shared group of rollouts: wherever any rollout clears a threshold the cap is inactive and the integrand of \cref{eq:app-tailrl-threshold-form} is the average score-function of the clearing rollouts, MaxRL's success-averaging rule for the threshold event, while at thresholds no rollout clears it is zero, MaxRL's rule for a group with no successes; its expectation carries MaxRL's order-$N$ truncated weight from \cref{eq:truncated-weight-function}, and the binary case, where a single threshold carries all the mass, recovers MaxRL exactly (\cref{cor:tailrl-binary-recovery}).

\paragraph{Loss reduction}
The estimator is a sum over the group.
A mean-reduced policy-gradient loss, which divides the group sum by $N$, must therefore use the coefficients $N\,\omega\bigl(r(x,z_i)\bigr)$, and after mean-centering $N\bigl(\omega\bigl(r(x,z_i)\bigr)-\bar\omega\bigr)$.
A sum-reduced loss uses $\omega\bigl(r(x,z_i)\bigr)$, and after mean-centering $\omega\bigl(r(x,z_i)\bigr)-\bar\omega$.

\subsection{Closed-Form Weights and Algorithm Correctness}
\label{app:tailrl-closed-form}

The next proposition proves the recurrence in \cref{eq:tailrl-weight-recurrence}.
Throughout, $z_1,\ldots,z_N$ is the training group of \cref{thm:tailrl-finite-unbiasedness}, and $r_{(1)}\le\cdots\le r_{(N)}$ denote the group rewards $r(x,z_1),\ldots,r(x,z_N)$ sorted increasingly, with $r_{(0)}:=0$.
The weight \cref{eq:app-tailrl-weight-restated} depends on a rollout only through its reward, so rollouts with equal rewards receive equal weights and $\omega\bigl(r_{(i)}\bigr)$ is well defined.

\begin{proposition}[Closed-form empirical weights]
\label{prop:tailrl-closed-form-weights}
For every $i=1,\ldots,N$,
\begin{equation}
\omega\bigl(r_{(i)}\bigr)
=
\sum_{k=1}^{i}
\frac{r_{(k)}-r_{(k-1)}}{N-k+1}.
\label{eq:app-tailrl-ascending-weights}
\end{equation}
Equivalently, all weights follow the recurrence
\begin{equation}
\omega\bigl(r_{(i)}\bigr)
=
\omega\bigl(r_{(i-1)}\bigr)
+
\frac{r_{(i)}-r_{(i-1)}}{N-i+1},
\qquad
\omega\bigl(r_{(0)}\bigr)=0.
\label{eq:app-tailrl-weight-recurrence}
\end{equation}
\end{proposition}

\begin{proof}
The empirical tail-probability distribution is a piecewise constant function.
Fix $k\in\{1,\ldots,N\}$ and a threshold $\tau\in[r_{(k-1)},r_{(k)})$.
The rewards $r_{(k)},\ldots,r_{(N)}$ are at least $r_{(k)}$ and therefore exceed $\tau$, while the rewards $r_{(1)},\ldots,r_{(k-1)}$ are at most $r_{(k-1)}\le\tau$ and fail, since clearance is strict:
\begin{equation}
\sum_{j=1}^{N}\ind_{\{r(x,z_j)>\tau\}}
=
N-k+1
\qquad
\text{for every }\tau\in[r_{(k-1)},r_{(k)}).
\label{eq:app-tailrl-piecewise-count}
\end{equation}
The intervals $[r_{(k-1)},r_{(k)})$ for $k=1,\ldots,i$ partition the integration range $[0,r_{(i)})$ of the weight, with tied rewards contributing empty intervals, so substituting \cref{eq:app-tailrl-piecewise-count} evaluates the integral interval by interval:
\begin{equation}
\omega\bigl(r_{(i)}\bigr)
=
\int_0^{r_{(i)}}
\frac{d\tau}{\sum_{j=1}^{N}\ind_{\{r(x,z_j)>\tau\}}}
=
\sum_{k=1}^{i}
\int_{r_{(k-1)}}^{r_{(k)}}
\frac{d\tau}{N-k+1}
=
\sum_{k=1}^{i}
\frac{r_{(k)}-r_{(k-1)}}{N-k+1},
\label{eq:app-tailrl-weight-partition}
\end{equation}
which is \cref{eq:app-tailrl-ascending-weights}; every denominator on the range is at least $N-i+1\geq 1$, so no ratio is ever indeterminate.
Subtracting \cref{eq:app-tailrl-ascending-weights} at ranks $i$ and $i-1$ leaves the single term $k=i$, which is the recurrence \cref{eq:app-tailrl-weight-recurrence}; the base case is the integral over the empty range $[0,r_{(0)})$.
\end{proof}

\paragraph{Scale of the weights}
Every weight is nonnegative, since every gap $r_{(k)}-r_{(k-1)}$ is nonnegative.
The total is the largest reward in the group: exchanging the order of the finite double sum counts each gap once per rank at or above it, and that count cancels its denominator, leaving a telescoping sum:
\begin{equation}
\sum_{i=1}^{N}\omega\bigl(r_{(i)}\bigr)
=
\sum_{k=1}^{N}\;\sum_{i=k}^{N}
\frac{r_{(k)}-r_{(k-1)}}{N-k+1}
=
\sum_{k=1}^{N}\bigl(r_{(k)}-r_{(k-1)}\bigr)
=
r_{(N)}
=
\max_{1\le j\le N} r(x,z_j).
\label{eq:app-tailrl-weight-total}
\end{equation}
For rewards in $[0,1]$, the pre-centering advantages $N\,\omega\bigl(r(x,z_i)\bigr)$ therefore sum to $N\,r_{(N)}\le N$ and each lies in $[0,\,N\,r_{(N)}]$, so the update scale is controlled by the best reward observed in the group.

\subsection{Effect of the Mean-Centering Baseline}
\label{app:tailrl-centering}

\begin{proposition}[The mean baseline lowers the truncation order by one]
\label{prop:tailrl-centered-target}
For $N\geq2$, let
\begin{equation}
\bar\omega
:=
\frac{1}{N}\sum_{i=1}^N\omega\bigl(r(x,z_i)\bigr),
\qquad
\gtailrlhat(x)
:=
\sum_{i=1}^N
\Bigl(\omega\bigl(r(x,z_i)\bigr)-\bar\omega\Bigr)\,S(x,z_i).
\label{eq:app-tailrl-centered-estimator}
\end{equation}
Then
\begin{equation}
\mathbb{E}_{z_{1:N}}\!\left[\gtailrlhat(x)\right]
=
\nabla_\theta J_{\mathrm{\tailrl{}}}^{(N-1)}(\theta;x).
\label{eq:app-tailrl-centered-target}
\end{equation}
\end{proposition}

\begin{proof}
By \cref{eq:app-tailrl-weight-total}, the weights of a group sum to its largest reward, so the centered estimator is the uncentered one minus a max-reward-weighted sum of scores:
\begin{equation}
\gtailrlhat(x)
=
\gtailrl(x)
-
\frac{1}{N}
\Bigl(\max_{1\le i\le N}r(x,z_i)\Bigr)
\sum_{i=1}^N S(x,z_i).
\label{eq:app-tailrl-centered-decomposition}
\end{equation}
The subtracted term is the score-function estimator of the Best-of-$N$ gradient: the group is one draw from the product policy, whose score-function is $\sum_{i=1}^N S(x,z_i)$, so
\begin{equation}
\nabla_\theta\,
\mathbb{E}_{z_{1:N}}\!\left[\max_{1\le i\le N}r(x,z_i)\right]
=
\mathbb{E}_{z_{1:N}}
\!\left[
\Bigl(\max_{1\le i\le N}r(x,z_i)\Bigr)
\sum_{i=1}^N S(x,z_i)
\right].
\label{eq:app-tailrl-best-n-gradient}
\end{equation}
Taking expectations in \cref{eq:app-tailrl-centered-decomposition} and invoking \cref{thm:tailrl-finite-unbiasedness} for the first term and \cref{eq:app-tailrl-best-n-gradient} for the second yields
\begin{equation}
\mathbb{E}_{z_{1:N}}\!\left[\gtailrlhat(x)\right]
=
\nabla_\theta J_{\mathrm{\tailrl{}}}^{(N)}(\theta;x)
-
\frac{1}{N}
\nabla_\theta\,
\mathbb{E}_{z_{1:N}}\!\left[\max_{1\le i\le N}r(x,z_i)\right].
\label{eq:app-tailrl-centered-expectation}
\end{equation}
It remains to identify the right side as the order-$(N-1)$ gradient.
The expected maximum is the Best-of-$N$ value, whose layer-cake form is \cref{eq:app-bok} at $k=N$:
\begin{equation}
\mathbb{E}_{z_{1:N}}\!\left[\max_{1\le i\le N}r(x,z_i)\right]
=
\int_0^1
\Bigl[1-\left(1-p_\theta(x,\tau)\right)^N\Bigr]d\tau.
\label{eq:app-tailrl-best-n-layer-cake}
\end{equation}
Meanwhile, the definitions of two consecutive truncations differ in one term of the inner sum:
\begin{equation}
J_{\mathrm{\tailrl{}}}^{(N)}(\theta;x)
-
J_{\mathrm{\tailrl{}}}^{(N-1)}(\theta;x)
=
-\frac{1}{N}
\int_0^1
\left(1-p_\theta(x,\tau)\right)^N d\tau.
\label{eq:app-tailrl-consecutive-truncations}
\end{equation}
Combining \cref{eq:app-tailrl-best-n-layer-cake,eq:app-tailrl-consecutive-truncations} shows the two sides differ by a constant:
\begin{equation}
J_{\mathrm{\tailrl{}}}^{(N)}(\theta;x)
-
J_{\mathrm{\tailrl{}}}^{(N-1)}(\theta;x)
=
\frac{1}{N}\,
\mathbb{E}_{z_{1:N}}\!\left[\max_{1\le i\le N}r(x,z_i)\right]
-
\frac{1}{N}.
\label{eq:app-tailrl-truncation-best-n}
\end{equation}
The final term is constant in $\theta$, so differentiating gives
\begin{equation}
\nabla_\theta J_{\mathrm{\tailrl{}}}^{(N)}(\theta;x)
-
\frac{1}{N}
\nabla_\theta\,
\mathbb{E}_{z_{1:N}}\!\left[\max_{1\le i\le N}r(x,z_i)\right]
=
\nabla_\theta J_{\mathrm{\tailrl{}}}^{(N-1)}(\theta;x).
\label{eq:app-tailrl-centered-gradient-identity}
\end{equation}
Substituting \cref{eq:app-tailrl-centered-gradient-identity} into \cref{eq:app-tailrl-centered-expectation} proves the result.
\end{proof}

%%%%%%%%%%%%%%%%%%%%%%%%%%%%

\subsection{Estimator-Level Recovery of MaxRL on Binary Rewards}
\label{app:tailrl-binary-recovery}
\begin{corollary}[Binary-reward recovery]
\label{cor:tailrl-binary-recovery}
Suppose $r_i\in\{0,1\}$ and let
$M:=\sum_{i=1}^N\ind\{r_i=1\}$. If $M>0$, then every successful
rollout receives weight $1/M$ and every unsuccessful rollout receives weight
zero. Hence the uncentered \tailrl{} estimator reduces to
\begin{equation}
\gtailrl(x)
=
\frac{1}{M}
\sum_{i:r_i=1}S_i,
\label{eq:app-tailrl-binary-uncentered}
\end{equation}
which is the MaxRL estimator. After mean-centering, successful rollouts have
weight $1/M-1/N$ and unsuccessful rollouts have weight $-1/N$. If $M=0$,
all uncentered and centered weights are zero.
\end{corollary}
\begin{proof}
For binary rewards with $\sum_{i=1}^N\ind\{r_i=1\}>0$,
\[
\omega(1)=\int_0^1\frac{d\tau}{\sum_{i=1}^N\ind\{r_i=1\}}=\frac{1}{\sum_{i=1}^N\ind\{r_i=1\}}
\]
and $\omega(0)=0$. The uncentered result follows immediately. The mean weight is
\[
\bar\omega=\left(\sum_{i=1}^N\ind\{r_i=1\}\right)\frac{1}{\sum_{i=1}^N\ind\{r_i=1\}}\cdot\frac{1}{N}=\frac{1}{N},
\]
which gives the centered weights. When $\sum_{i=1}^N\ind\{r_i=1\}=0$,
every reward and hence every weight is zero.
\end{proof}

\section{TailRL on a General Reward Range}
\label{app:reward-range}

The main paper assumes rewards in $[0,1]$ only to simplify notation.
The same objective applies to any bounded reward range.

Suppose the reward function returns rewards in an arbitrary bounded range, $r(x,z)\in[r_{\mathrm{min}},r_{\mathrm{max}}]$ with $r_{\mathrm{min}}<r_{\mathrm{max}}$.
For thresholds $\tau\in[r_{\mathrm{min}},r_{\mathrm{max}})$, define the tail-probability as before:
$p_\theta(x,\tau):=\Pr_{z\sim\pi_\theta(\cdot\mid x)}[r(x,z)>\tau]$.
The TailRL objective averages the log-tail-probability across this range:
\begin{equation}
J_{\mathrm{TailRL}}^{[r_{\mathrm{min}},r_{\mathrm{max}}]}(\theta;x)
:=
\frac{1}{r_{\mathrm{max}}-r_{\mathrm{min}}}
\int_{r_{\mathrm{min}}}^{r_{\mathrm{max}}} \log p_\theta(x,\tau)\,d\tau .
\label{eq:tailrl-general-range}
\end{equation}
The factor $1/(r_{\mathrm{max}}-r_{\mathrm{min}})$ makes this the expected log-tail-probability under a threshold drawn uniformly from $[r_{\mathrm{min}},r_{\mathrm{max}}]$.
It also prevents the scale of the objective from depending on the units of the reward.

To recover the unit-interval objective, substitute $\tau = r_{\mathrm{min}}+(r_{\mathrm{max}}-r_{\mathrm{min}})\,u$ with $u\in[0,1)$ and define the rescaled reward $\tilde r := (r-r_{\mathrm{min}})/(r_{\mathrm{max}}-r_{\mathrm{min}})\in[0,1]$.
Then
\begin{equation}
J_{\mathrm{TailRL}}^{[r_{\mathrm{min}},r_{\mathrm{max}}]}(\theta;x)
=
\int_0^1
\log
\Pr_{z\sim\pi_\theta(\cdot\mid x)}\!\bigl[\tilde r(x,z)>u\bigr]
\,du
=
J_{\mathrm{TailRL}}(\theta;\,x)
\quad
\text{for the reward }\tilde r .
\label{eq:tailrl-affine-invariance}
\end{equation}
Thus, applying TailRL to rewards in $[r_{\mathrm{min}},r_{\mathrm{max}}]$ is equivalent to applying the unit-interval objective to $\tilde r$.
Shifting or rescaling the reward only relabels its thresholds, so every result in the paper carries over directly.

The finite rollout estimator behaves in the same way.
For sorted rewards, consecutive weights satisfy
\begin{equation}
\omega_{(i)}-\omega_{(i-1)}
=
\frac{r_{(i)}-r_{(i-1)}}{N-i+1},
\label{eq:general-range-gaps}
\end{equation}
so an affine rescaling of the rewards multiplies every weight by $r_{\mathrm{max}}-r_{\mathrm{min}}$.
The normalization in \cref{eq:tailrl-general-range} cancels this factor.
Without the normalization, the update direction remains unchanged and only its magnitude is rescaled.
The estimator can therefore operate on raw rewards from any bounded range without changing the algorithm.

\section{Connection to Ordinal Cross-Entropy}
\label{app:ordinal-ce}

Threshold decompositions turn an ordinal or continuous target into a family of binary events, one event per threshold, and fit each event with a binary classifier \citep{mccullagh1980regression,frank2001simple,li2006ordinal,niu2016ordinal,cao2020rank}. This appendix makes the relation to \tailrl{} exact. The population \tailrl{} objective is the ordinal cross-entropy objective evaluated at the maximal target $r=1$, and the finite rollout \tailrl{} procedure is the finite sample estimation of that objective.

Fix an input $x$ and recall the tail-probability $p_\theta(x,\tau) = \Pr_{z\sim\pi_\theta(\cdot\mid x)}\left(r(x,z) > \tau\right)$. For a target level $t \in [0,1]$, define the threshold-decomposed cross-entropy between the point target $t$ and the model's family of threshold events,
\begin{equation}
\mathrm{CE}(t;\theta,x) := -\int_0^1 \Big[\, \ind_{\{t>\tau\}}\,\log p_\theta(x,\tau) \;+\; \ind_{\{t\le\tau\}}\,\log\big(1-p_\theta(x,\tau)\big) \,\Big]\, d\tau .
\label{eq:ordinal-ce-def}
\end{equation}
For each fixed $\tau$ the bracket is the binary cross-entropy of the event $\{r>\tau\}$ against the label $\ind_{\{t>\tau\}}$, and the integral weights all thresholds by the same uniform measure that defines the tail-likelihood.

Setting the target to the maximal reward $t=1$ makes $\ind_{\{t>\tau\}}=1$ for every $\tau\in[0,1)$, so the second term in \cref{eq:ordinal-ce-def} vanishes on a set of full measure and
\begin{equation}
\mathrm{CE}(1;\theta,x) \;=\; -\int_0^1 \log p_\theta(x,\tau)\, d\tau \;=\; -\,J_{\mathrm{\tailrl{}}}(\theta;x).
\label{eq:ordinal-ce-collapse}
\end{equation}
Maximizing the tail-likelihood is therefore exactly minimizing the ordinal cross-entropy against the ideal target $r=1$. When the reward is binary the tail-probability is constant in $\tau$, the integral collapses to a single binary cross-entropy against the label $1$, and \cref{eq:ordinal-ce-collapse} reduces to the MaxRL objective $\log q_\theta(x)$, consistent with \cref{sec:tailrl-binary-recovery}.

The correspondence extends to the finite rollout procedure. The order-$N$ objective $J^{(N)}_{\mathrm{\tailrl{}}}$ of \cref{eq:tailrl-harmonic-family} truncates the same integral at the resolution a group of $N$ rollouts can support, and the estimator of \cref{thm:tailrl-finite-unbiasedness} is unbiased for its gradient using only the $N$ sampled rewards. Training with \tailrl{} on $N$ rollouts is in this sense the finite sample estimation of the ordinal cross-entropy objective at target $r=1$: the rollouts play the role of the samples from which the threshold events are estimated, and the truncation order grows with the sample size, recovering \cref{eq:ordinal-ce-collapse} as $N\to\infty$.

\section{Advantage Functions of GRPO, RLOO, and \tailrl{}}
\label{app:advantage-comparison}

All three methods share the same critic-free template: draw $N$ rollouts, map their rewards to advantages, and form the update $\sum_{i=1}^{N} A_i\,S(x,z_i)$.
They differ only in that map, which we record here.
For convenience of notation, we write $r_i:=r(x,z_i)$ for the reward of rollout $z_i$.

\paragraph{RLOO}
RLOO subtracts the leave-one-out mean of the other rollouts,
\begin{equation}
A_i^{(\mathrm{RLOO})}
=
r_i - \frac{1}{N-1}\sum_{j \ne i} r_j .
\end{equation}
The baseline is independent of rollout $i$, so the update is an unbiased estimator of the expected reward gradient.
The advantage is affine in $r_i$ with unit slope: a rollout is promoted by its raw margin over the rest of the group.

\paragraph{GRPO}
GRPO standardizes within the group,
\begin{equation}
A_i^{(\mathrm{GRPO})}
=
\frac{r_i - \bar r}{\sigma(r) + \epsilon},
\qquad
\bar r = \frac{1}{N}\sum_{j=1}^{N} r_j ,
\qquad
\sigma(r)^2 = \frac{1}{N}\sum_{j=1}^{N} (r_j - \bar r)^2 .
\end{equation}
Dividing by the group standard deviation makes the update invariant to affine rescaling of the reward.
The same divisor applies to every rollout in the group, so it changes the size of an input's update rather than the relative weight of rewards within it.

\paragraph{\tailrl{}}
\tailrl{} assigns weights proportional to inverse tail-probability's integral and centers it,
\begin{equation}
A_i^{(\mathrm{\tailrl{}})}
=
\omega(r_i) - \bar\omega,
\qquad
\omega(r_i) = \int_0^{r_i} \frac{d\tau}{\sum_{j=1}^{N} \ind_{\{r_j > \tau\}}},
\qquad
\bar\omega = \frac{1}{N}\sum_{j=1}^{N}\omega(r_j) .
\end{equation}
The integrand is the reciprocal of the number of rollouts attaining a reward above a given reward threshold, so a level cleared by one rollout out of $N$ contributes $N$ times the weight per unit of reward of one cleared by all.
The map is computed in closed form by the recurrence in \cref{eq:tailrl-weight-recurrence} and reduces to the MaxRL advantages when the reward is binary (\cref{app:tailrl-binary-recovery}).

\section{Pass@\texorpdfstring{$k$}{k} and Best-of-\texorpdfstring{$k$}{k} Empirical Calculation}
\label{app:passk-bestk}

Pass@$k$ is used for binary rewards and measures the probability that at least one rollout succeeds.
Best-of-$k$ is used for continuous rewards and measures the  expected reward of the highest-scoring rollout. 
Both metrics evaluate a policy when we sample $k$ rollouts and keep the best one.
The two metrics are identical for binary rewards, and both improve or remain unchanged as $k$ increases.
When we have $K$ evaluation rollouts to estimate Pass@$k$ or Best-of-$k$ with $K > k$, we use all $K$ rollouts to compute their unbiased estimators below.

\subsection{Pass@\texorpdfstring{$k$}{k}}
\label{app:passk-calculation}
If $M$ of the $K$ sampled rollouts succeed, the unbiased estimator of \citet{chen2021codex} is
\begin{equation}
\widehat{\mathrm{Pass@}k}
=
1 - \frac{\binom{K-M}{k}}{\binom{K}{k}} .
\label{eq:app-passk-estimator}
\end{equation}
The ratio is the probability that a uniformly random size-$k$ subset of the $K$ rollouts avoids all $M$ successes.

\subsection{Best-of-\texorpdfstring{$k$}{k}}
\label{app:bestk-calculation}
Best-of-$k$ is defined as in \cref{eq:app-bok}
\begin{equation}
\text{Best-of-}k(\theta; x)
=
\mathbb{E}\!\left[\max_{1 \le i \le k} r(x,z_i)\right]
=
\int_0^1 \left[1 - \left(1 - p_\theta(x,\tau)\right)^{k}\right] d\tau ,
\label{eq:app-bok-restated}
\end{equation}
the second form by the layer-cake identity applied to the maximum, which reduces to Pass@$k$ when the reward is binary.
Sorting the $K$ sampled rewards increasingly as $r_{(1)} \le \cdots \le r_{(K)}$, the unbiased estimator weights each order statistic by the probability that it is the maximum of a uniformly random size-$k$ subset:
\begin{equation}
\widehat{\text{Best-of-}k}
=
\sum_{i=k}^{K} \frac{\binom{i-1}{k-1}}{\binom{K}{k}}\, r_{(i)} .
\label{eq:app-bok-estimator}
\end{equation}

\section{ImageNet Object Localization}
\label{app:imagenet}

We study single-object localization on ImageNet-scale data \citep{russakovsky2015imagenetlargescalevisual}.
Given an image, the policy predicts a bounding box and receives its intersection-over-union (IoU) with the ground-truth box as a reward in $[0,1]$.

The policy uses a pretrained ResNet-50 backbone \citep{he2015deepresiduallearningimage} with four categorical heads, one for each box coordinate.
Each head contains $50$ uniformly spaced bins, and a rollout samples one bin from each head to form a box. All methods train for $30$ epochs with Adam, a learning rate of $5\times10^{-4}$ with warmup, and a batch size of $128$.
They differ only in how they compute the advantages.
We train \tailrl{} with $N\in\{16,64,256,1024\}$ rollouts and separately optimize the exact population-level objective described in \cref{app:imagenet-population}.

For comparison, we also train supervised models directly on the ground-truth coordinates using MSE, L1, GIoU \citep{rezatofighi2019giou}, and L1+GIoU losses (\cref{app:imagenet-losses}).

We evaluate on held-out images using CorLoc@$\delta$, mean IoU, and Best-of-$1024$ IoU.
CorLoc@$\delta$ is the fraction of images whose greedy prediction exceeds IoU $\delta$.
Best-of-$1024$ IoU is the highest IoU among $1024$ sampled boxes for each image.
We report results over $3$ seeds.
The large-$N$ gradient measurements in \cref{fig:imagenet-gradients} use one seed, as noted in the figure.
\Cref{tab:hp-imagenet} summarizes the full configuration.

\begin{table}[!ht]
  \caption{Training hyperparameters for ImageNet object localization.}
  \label{tab:hp-imagenet}
  \centering
  \begin{tcolorbox}[enhanced, hbox,
    colback=teal!5, colframe=teal, coltext=black, coltitle=white,
    fonttitle=\bfseries, arc=1mm, boxrule=1pt, boxsep=1pt,
    left=2pt, right=2pt, top=0pt, bottom=2pt,
    toptitle=3pt, bottomtitle=3pt, center]
    \small
    \renewcommand{\arraystretch}{1.2}
    \setlength{\tabcolsep}{10pt}
\begin{tabular}{l|l|l|l}
  \multicolumn{1}{l|}{\textbf{Parameter}} & \multicolumn{1}{l|}{\textbf{Value}} & \multicolumn{1}{l|}{\textbf{Parameter}} & \multicolumn{1}{l}{\textbf{Value}} \\
  \hline
  Backbone & ResNet-50 (pretrained) & Policy head & $4$ heads $\times$ $50$ bins \\
  Reward & IoU, $[0, 1]$ & Optimizer & Adam \\
  Learning rate & $5 \times 10^{-4}$, warmup & Batch size & $128$ \\
  Epochs & $30$ & Rollouts $N$ & $\{16, 64, 256, 1024\}$ \\
  Population budget & $50^{4}$ & Baselines & GRPO, RLOO \\
  Anchors & MSE, L1, GIoU, L1+GIoU & Eval samples & $1{,}024$ per image \\
  CorLoc decoding & greedy & Seeds & $3$  \\
\end{tabular}
  \end{tcolorbox}
\end{table}

%%%%%%%%%%%%%%%%%%%%%%%%%%%%%%%%%

\subsection{Supervised Anchor Losses}
\label{app:imagenet-losses}
Let $b = (x_1,y_1,x_2,y_2)$ denote the predicted box and $b^\star$ the ground-truth box.
Both use normalized corner coordinates.
The L1 loss measures the absolute error across the four coordinates:
\begin{equation}
\mathcal{L}_{\mathrm{L1}}(b, b^\star)
=
\left\| b - b^\star \right\|_1
=
\sum_{j=1}^{4} \left\| b_j - b^\star_j \right\|_1 .
\end{equation}
The GIoU loss \citep{rezatofighi2019giou} extends IoU with a penalty based on the smallest box enclosing both boxes. We use $|b|$ to denote the area enclosed by the bounding box $b$.
Let $\mathrm{I} = |b \cap b^\star|$ be the intersection area, $\mathrm{U} = |b\cup b^\star|$ the union area, and $c$ the smallest axis-aligned box enclosing both $b$ and $b^{\star}$ and has an area $|c|$:
\begin{equation}
\mathcal{L}_{\mathrm{GIoU}}(b, b^\star)
=
1 - \underbrace{\frac{\mathrm{I}}{\mathrm{U}}}_{\mathrm{IoU}}
+ \frac{|c| - \mathrm{U}}{|c|} .
\end{equation}
The final term measures the empty space inside the enclosing box.
Unlike IoU, it remains informative when the predicted and ground-truth boxes do not overlap.
The combined loss follows DETR \citep{carion2020detr}:
\begin{equation}
\mathcal{L}_{\mathrm{L1+GIoU}}(b, b^\star)
=
\lambda_{\mathrm{L1}} \mathcal{L}_{\mathrm{L1}}(b, b^\star)
+
\lambda_{\mathrm{GIoU}} \mathcal{L}_{\mathrm{GIoU}}(b, b^\star) .
\end{equation}
We use the DETR weights $\lambda_{\mathrm{L1}} = 5$ and $\lambda_{\mathrm{GIoU}} = 2$.
All three supervised anchors make one deterministic prediction and use no rollouts.
Their greedy, mean, and Best-of-$k$ outputs are therefore identical.

%%%%%%%%%%%%%%%%%%%%%%%%%%%%%%%%%

\subsection{Computing the Population-Level Objective}
\label{app:imagenet-population}

Computing the population objective $J_{\mathrm{\tailrl{}}}(\theta;x) = \int_0^1 \log p_\theta(x,\tau) d\tau$ usually requires the unknown tail-probability $p_\theta(x,\tau)$.
In this setting, we can compute it exactly because the policy has a finite output space. The policy predicts each of the four box coordinates with an independent categorical head over $K = 50$ uniformly spaced bins.
A rollout samples one bin from each head, so its probability factorizes as
\begin{equation}
\pi_\theta(b \mid x) = \prod_{j=1}^{4} \pi_\theta^{(j)}(b_j \mid x), \qquad b \in \mathcal{B}, \quad |\mathcal{B}| = K^4 = 6{,}250{,}000 .
\end{equation}
Each box $b$ has a deterministic reward $r(x,b)$ given by its IoU with the ground-truth box.
We can therefore compute $\left(\pi_\theta(b \mid x), r(x,b)\right)$ for every box in $\mathcal{B}$.
The tail-probability at a threshold is the total probability of all boxes whose rewards exceed that threshold:
\begin{equation}
p_\theta(x,\tau) = \sum_{b \in \mathcal{B}} \pi_\theta(b \mid x) \ind_{\{r(x,b) > \tau\}} 
\end{equation}
The finite set $\mathcal{B}$ produces a finite set of reward values.
The function $\tau \mapsto p_\theta(x,\tau)$ is constant between consecutive reward values, so the threshold integral becomes a finite sum.
This calculation introduces no discretization error beyond the original coordinate bins. In practice, we use the independence of the four heads to avoid explicitly enumerating all $K^4$ boxes.
\Cref{alg:population-objective} forms the joint probabilities from the four marginal distributions and accumulates them over the sorted reward values.

\begin{algorithm}[!ht]
\caption{Population-level \tailrl{} objective and gradient for one image}
\label{alg:population-objective}
\begin{algorithmic}[1]
\Require Head distributions $\pi_\theta^{(j)}(\cdot \mid x)$ over $K$ bins, $j = 1,\ldots,4$; ground-truth box $b^\star$
\State $\mathcal{B} \gets \{1,\ldots,K\}^{4}$ \algcomment{every box the policy can emit}
\For{$b = (b_1,\ldots,b_4) \in \mathcal{B}$}
    \State $\pi_\theta(b \mid x) \gets \prod_{j=1}^{4} \pi_\theta^{(j)}(b_j \mid x)$ \algcomment{the heads are independent}
    \State $r(x,b) \gets \mathrm{IoU}(b, b^\star)$
\EndFor
\State $u_1 < \cdots < u_m \gets$ the distinct values of $\{\, r(x,b) : b \in \mathcal{B} \,\}$, sorted increasingly
\For{$i = 1,\ldots,m$}
    \State $q_i \gets \sum_{b \,:\, r(x,b) = u_i} \pi_\theta(b \mid x)$ \algcomment{reward distribution}
\EndFor
\State $p_\theta(x, u_m) \gets 0$ \algcomment{no box exceeds the largest attainable reward}
\For{$i = m-1,\ldots,1$}
    \State $p_\theta(x, u_i) \gets p_\theta(x, u_{i+1}) + q_{i+1}$ \algcomment{$p_\theta(x,\tau) = p_\theta(x,u_i)$ for $\tau \in [u_i, u_{i+1})$}
\EndFor
\State $J_{\mathrm{\tailrl{}}}(\theta;x) \gets \sum_{i=1}^{m-1} (u_{i+1} - u_i) \log p_\theta(x, u_i)$
\State $\nabla_\theta J_{\mathrm{\tailrl{}}}(\theta;x) \gets \sum_{i=1}^{m-1} (u_{i+1} - u_i) \dfrac{\nabla_\theta\, p_\theta(x, u_i)}{p_\theta(x, u_i)}$
\Ensure $J_{\mathrm{\tailrl{}}}(\theta;x)$ and $\nabla_\theta J_{\mathrm{\tailrl{}}}(\theta;x)$, exact, with no sampling error
\end{algorithmic}
\end{algorithm}

Substituting the exact $p_\theta(x,\tau)$ into \cref{eq:tailrl-gradient-over-thresholds} gives the gradient used by the population variant.
This is an exact population gradient, not a large-$N$ approximation.
We obtain it by automatically differentiating through the probability accumulation because each tail-probability value is a differentiable function of the prediction head probabilities.

\subsection{Additional Results}
\label{app:imagenet-additional}
\paragraph{Training dynamics}
\Cref{fig:app-imagenet-dynamics} shows training from three perspectives.
The left panel reports final held-out CorLoc@$0.5$ across training rollout budgets $N$.
\tailrl{} improves steadily as $N$ increases and approaches the exact population objective.
GRPO improves only slightly, while RLOO does not improve with additional rollouts.
At every budget, both baselines perform worse than \tailrl{} at its smallest reported budget. 

The center panel shows the mean gradient norm by epoch at $N=1024$.
GRPO produces the largest gradients, while RLOO produces the smallest.
However, GRPO also achieves the lowest final accuracy, showing that larger gradients do not explain the performance differences in \cref{fig:imagenet-rl-curves}.The right panel shows each method's training loss over RL steps using a logarithmic step axis.
Because the methods optimize different objectives, their loss values are not directly comparable.
Nevertheless, all three losses flatten near zero late in training.

\begin{figure}[!ht]
\centering
\includegraphics[width=\textwidth]{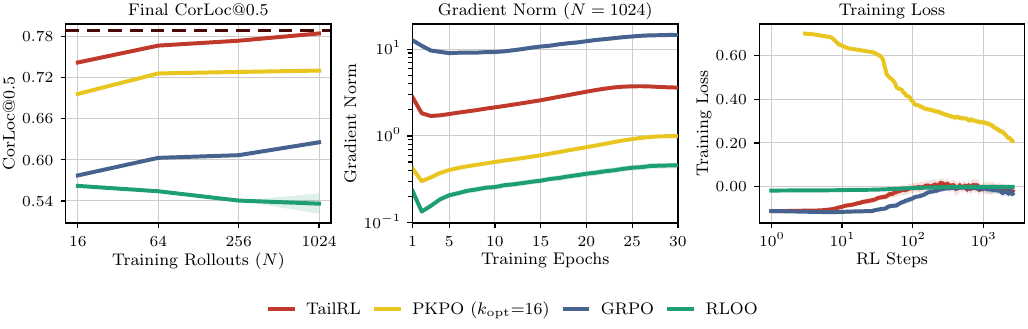}
\caption{\textbf{(ImageNet object localization)}
Left: final validation CorLoc@$0.5$ against the training rollout budget $N$ for all three methods, with the population-level objective as a dashed reference.
Center: mean gradient norm by epoch at $N=1024$, log scale.
Right: each method's training loss against RL steps, logarithmic step axis; each method optimizes its own surrogate objective.
Means over the seeds of \cref{fig:imagenet-rl-curves}.}
\label{fig:app-imagenet-dynamics}
\end{figure}

\begin{figure}[!ht]
\centering
\includegraphics[width=\textwidth]{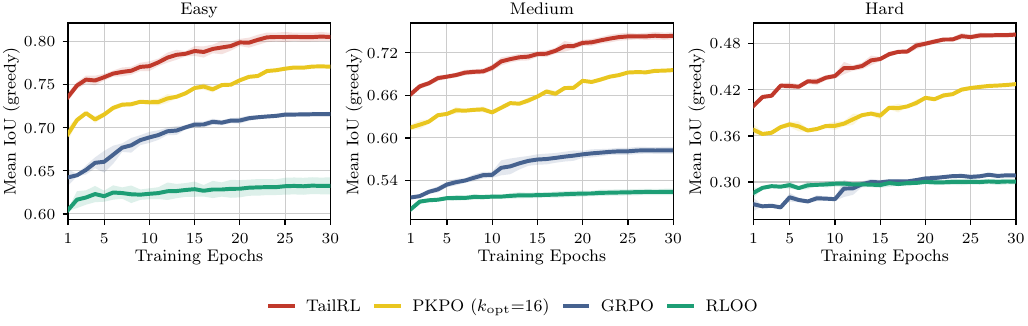}
\caption{\textbf{(ImageNet object localization)} Mean IoU by difficulty band for the three methods at $N=1024$.
Bands partition the validation set by the area of the ground-truth box as a fraction of the image: easy is $0.30$ to $0.70$ ($19{,}628$ images), medium is $0.10$ to $0.30$ or $0.70$ to $0.95$ ($19{,}977$), and hard is below $0.10$ or above $0.95$ ($10{,}395$).
Curves are means with standard-error bands over $3$ seeds.}
\label{fig:imagenet-difficulty}
\end{figure}

\Cref{fig:imagenet-difficulty} stratifies validation performance by difficulty band.
The bands are defined by the area of the ground-truth box as a fraction of the image, since a box that fills a moderate part of the frame is far easier to localize than a very small or a nearly full-frame one.
An image is easy when that area lies between $0.30$ and $0.70$, medium when it lies between $0.10$ and $0.30$ or between $0.70$ and $0.95$, and hard when it falls below $0.10$ or above $0.95$; the three bands partition the $50{,}000$ validation images into $19{,}628$, $19{,}977$, and $10{,}395$.
The gap between \tailrl{} and the stronger expected reward baseline widens as the band gets harder: roughly $0.09$ IoU on the easy band, $0.16$ on medium, and $0.18$ on hard, where both baselines sit near $0.30$ and \tailrl{} reaches $0.49$.
Hard inputs are those on which high-reward rollouts are rare, and rarity is where the inverse-probability weighting of the tail-likelihood concentrates its effort, so the ordering of the gaps matches the mechanism the objective is built around.

\subsection{The Effect of Binarizing the Reward}
\label{app:imagenet-binarization}

A continuous reward contains more information than binary reward.
Binarizing it with a threshold removes this information by treating all outputs above or below the threshold as equally goor or bad.
We test the cost of this lost signal by binarizing IoU at $\{0.5,0.75\}$ and training MaxRL on each binary reward.
We compare both variants with \tailrl{} trained on the original continuous reward across four rollout budgets, with all other settings held fixed.

\begin{figure}[!ht]
\centering
\includegraphics[width=\textwidth]{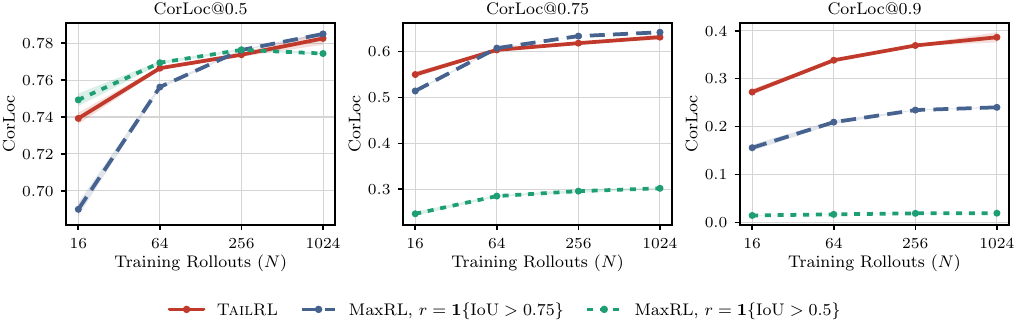}
\caption{\textbf{(ImageNet object localization)} Effect of binarizing the IoU reward on that task of ImageNet object localization.
We compare \tailrl{} trained on continuous IoU with MaxRL trained on rewards binarized at $\mathrm{IoU}>0.5$ and $\mathrm{IoU}>0.75$.
All methods perform similarly at the easiest evaluation threshold, but the binarized methods degrade at CorLoc above their binarization threshold.}
\label{fig:imagenet-binary-maxrl}
\end{figure}

Binarization produces policies that perform well near their chosen threshold but poorly at higher quality levels (\cref{fig:imagenet-binary-maxrl}).
At CorLoc@$0.5$, all methods reach approximately $0.78$.
At CorLoc@$0.75$, MaxRL trained with the $0.5$ threshold reaches only $0.30$, compared with approximately $0.63$ for \tailrl{} and MaxRL trained with the $0.75$ threshold.
At CorLoc@$0.9$, the gap widens further: \tailrl{} reaches approximately $0.39$, compared with $0.24$ for the $0.75$ threshold and $0.02$ for the $0.5$ threshold.

Once a rollout has an IoU above its binarization threshold, the binary reward no longer distinguishes a barely acceptable box from a nearly perfect one.
This lost ordering produces both qualitative specialization to the chosen threshold and large quantitative losses above it.
Increasing the rollout budget cannot recover information removed from the reward, so the binarized methods flatten while \tailrl{} continues to improve at stricter thresholds.
\tailrl{} avoids choosing a specific threshold and learns from the full continuous reward, allowing one policy to perform well across all evaluated quality levels.

\section{Text-Maze Navigation}
\label{app:maze}

We study navigation in $17 \times 17$ gridworld mazes represented as text.
Given a maze, the policy generates a sequence of movement tokens.
The continuous reward measures whether the sequence is well formed, how close it ends to the goal, and how its length compares with the shortest path.
A rollout that reaches the goal along a shortest path receives a reward of $1$.

The policy is a $3$M-parameter decoder-only transformer trained from scratch.
We first pretrain it through supervised learning on $1.3$M mazes with up to $16$ annotated paths per maze.
To test how each RL method behaves under different initial policy qualities, we retain seven pretraining checkpoints.
Their held-out shortest-path rates range from $0.012\%$, or roughly one success in ten thousand attempts, to $0.83\%$.

Starting from each checkpoint, we train every method for $5{,}000$ steps using $N = 16$ rollouts for each of $256$ mazes per step.
Training is fully on-policy, uses no KL regularization, and is repeated over three seeds.
We report the shortest-path rate, defined as the probability that a single sampled rollout reaches the goal using an optimal-length path.
We estimate this rate from $64$ rollouts on each of $256$ held-out mazes.

\subsection{Dataset and Pretraining}
\label{app:maze-data}

Each maze is a $17 \times 17$ grid with the start and goal at opposite corners.
We generate a perfect maze using Prim's algorithm and then remove a uniformly sampled $5$ to $30\%$ of its interior walls.
Removing these walls creates alternative routes and varies the number of valid paths across mazes.

For pretraining, we pair each maze with up to $16$ goal-reaching paths selected by \cref{alg:maze-corpus}.
We first find the shortest-path length $L^{\star}$ using breadth-first search.
A budgeted depth-first search then finds simple paths shorter than $\mathrm{ub}=60$.
To prevent common path lengths from dominating the corpus, reservoir sampling retains at most four paths of each length.

We select paths that span a range of solution qualities.
Each candidate path $P$ receives a reward, where a shortest path receives $1$ and longer paths receive smaller values.
We divide reward $ r \in (0,1]$ into $16$ equal intervals and select at most one path from each nonempty interval.
When an interval contains multiple candidates, we select the path that overlaps least with those already chosen.
This produces paths that vary in both length and spatial route.
We train the policy on the resulting maze-path pairs using next-token prediction.

\begin{algorithm}[!ht]
\caption{Selecting pretraining paths for one maze}
\label{alg:maze-corpus}
\begin{algorithmic}[1]
\Require Maze $m$; length limit $\mathrm{ub}=60$; target number of paths $n_{\mathrm{paths}}=16$

\State $\mathcal{C}\gets$ all simple start-to-goal paths found by breadth-first search with path length$<\mathrm{ub}$
\State Divide $\mathcal{C}$ into $n_{\mathrm{paths}}$ equal buckets according to path length
\State $\mathcal{S}\gets\emptyset$
\For{each nonempty bucket $\mathcal{C}_t$}
\State Select one path $P_t\in\mathcal{C}_t$
\State $\mathcal{S}\gets\mathcal{S}\cup{\{P_t\}}$
\EndFor
\Ensure A set $\mathcal{S}$ of paths with diverse lengths
\end{algorithmic}
\end{algorithm}

\subsection{Post-Training Configuration}
\label{app:maze-rlconfig}
The reinforcement-learning stage starts from the checkpoints above and uses the configuration of \cref{tab:hp-maze}, which also lists the evaluation sampling.

\begin{table}[!ht]
  \caption{Training hyperparameters for Text-Maze navigation.}
  \label{tab:hp-maze}
  \centering
  \begin{tcolorbox}[enhanced, hbox,
    colback=teal!5, colframe=teal, coltext=black, coltitle=white,
    fonttitle=\bfseries, arc=1mm, boxrule=1pt, boxsep=1pt,
    left=2pt, right=2pt, top=0pt, bottom=2pt,
    toptitle=3pt, bottomtitle=3pt, center]
    \small
    \renewcommand{\arraystretch}{1.2}
    \setlength{\tabcolsep}{10pt}
\begin{tabular}{l|l|l|l}
  \multicolumn{1}{l|}{\textbf{Parameter}} & \multicolumn{1}{l|}{\textbf{Value}} & \multicolumn{1}{l|}{\textbf{Parameter}} & \multicolumn{1}{l                                  }{\textbf{Value}} \\
  \hline
  Maze & $17 \times 17$, text & Policy & $3$M decoder ($4$L, $4$H, $256$) \\
  Pretraining corpus & $1.3$M mazes, $\le 16$ paths & Path length & $\le 60$ \\
  RL framework & verl, on-policy & PPO epochs & $1$, no KL \\
  Prompts per step & $256$ & Rollouts $N$ & $16$ (sweep $\{4, 16, 64, 256\}$) \\
  Learning rate & $10^{-4}$ & Steps & $5{,}000$ \\
  Evaluation & $64$ samples $\times$ $1024$ mazes & Seeds & $3$ \\
  Eval.\ rollouts per maze & $64$ & & \\
\end{tabular}
  \end{tcolorbox}
\end{table}

\subsection{Reward Function}
\label{app:maze-reward}
A rollout is parsed into a sequence of moves and replayed in the maze.
Write $L^\star$ for the length of a shortest path from start to goal, $L$ for the length of the rollout's path when it reaches the goal, and $d$ for the breadth-first distance from the rollout's final cell to the goal.
The reward is the sum of a progress term and a solution term,
\begin{equation}
r(x,z) = \underbrace{\tfrac{1}{2}\min\!\left(1, \tfrac{L^\star - d}{L^\star}\right)}_{\text{progress}} \; + \; \underbrace{\tfrac{1}{2}\min\!\left(1, \tfrac{L^\star}{L}\right)\ind_{\{\text{goal reached}\}}}_{\text{solution}} ,
\end{equation}
with both terms clipped below at zero, and with $r(x,z) = 0$ whenever the rollout is malformed or walks into a wall.
The progress term pays partial credit for ending closer to the goal than the start, so a rollout that never arrives is still ranked by how far it got.
The solution term pays full credit only for a shortest path and decays as $L^\star / L$ when the policy reaches the goal by a longer route.
The reward equals $1$ exactly when the rollout reaches the goal along a shortest path, the event called shortest-path success in \cref{sec:exp-maze}.

\subsection{Task Representation and Prompt Template}
\label{app:maze-prompt}
A maze is serialized as a flat token sequence rather than as natural language, so the policy never sees English and cannot rely on pretrained language priors.
The grid is written cell by cell with one token per cell, rows separated by a newline token, and the prompt ends at the marker that opens the path; the model then generates the path itself as a sequence of coordinate tokens.
Each row of the $17 \times 17$ grid contributes seventeen cell tokens drawn from \texttt{WALL}, \texttt{PATH}, \texttt{START}, and \texttt{GOAL}, followed by \texttt{NEWLINE}; a rollout is the continuation after \texttt{PATH\_START}, a sequence of movement tokens closed by \texttt{DONE}.
It is scored for well-formedness, for how close it ends to the goal, and for its length against the shortest path.

\begin{tcolorbox}[
  colback=gray!3,
  colframe=gray!40,
  title={$17 \times 17$ Maze Example Model Input and Output Format},
  fonttitle=\small\bfseries,
  boxrule=0.4pt,
  arc=2mm,
  left=4pt,
  right=4pt,
  top=4pt,
  bottom=4pt
]
\textbf{Input:}

{\ttfamily\small
\textless bos\textgreater{} GRID\_START WALL WALL WALL $\ldots$ WALL START PATH PATH $\ldots$ NEWLINE $\ldots$ GOAL WALL NEWLINE $\ldots$ GRID\_END PATH\_START
}

\textbf{Output:}

{\ttfamily\small
RIGHT RIGHT DOWN DOWN $\ldots$ RIGHT DONE \textless eos\textgreater
}
\end{tcolorbox}

\paragraph{Reward examples}
% \label{app:maze-reward-examples}
\Cref{fig:app-maze-examples} shows four rollouts on one maze with the reward each receives.
The two left paths never reach the goal and are scored by how far they get, which is what makes the reward continuous rather than binary.
The third path reaches the goal but wanders, so it scores below a shortest path; only the rightmost path, which reaches the goal along a shortest path, receives reward $1$.

\subsection{Additional Results}
\label{app:maze-additional}
\begin{figure}[!ht]
\centering
\includegraphics[width=\textwidth]{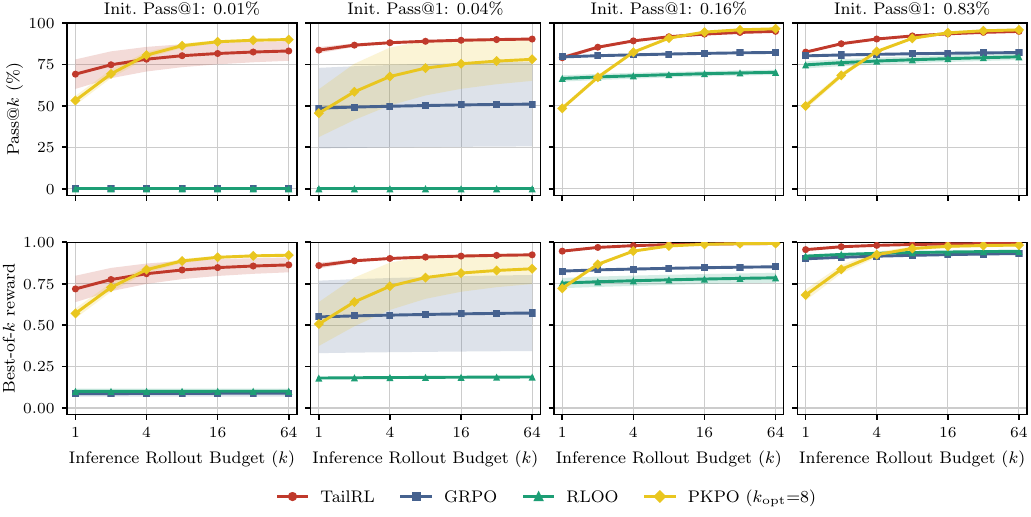}
\caption{\textbf{(Text-Maze)} Inference-time scaling on Text-Maze for four initial policies, labeled by their shortest-path success before reinforcement learning.
The rows report Pass@$k$, the probability that any of $k$ rollouts reaches the goal along a shortest path, and Best-of-$k$ reward is the expected maximum reward attained by the policy among $k$ rollouts.}
\label{fig:maze-ladders}
\end{figure}

\paragraph{Effect of the inference rollout budget}
We evaluate how the learned rollout distributions respond to additional inference rollouts.
\Cref{fig:maze-ladders} reports Pass@$k$ and Best-of-$k$ reward as $k$ increases.
At small $k$, policies trained with the three methods can appear similar.
As $k$ increases, both metrics improve more rapidly for \tailrl{} than for GRPO or RLOO. Despite GRPO and RLOO improving their Pass@$1$ after RL post-training on policy initializations with good coverage, \tailrl{} still outperforms them at test-time scaling.
Evaluation at $k=1$ therefore understates the differences among the learned rollout distributions.
The increasing separation is consistent with the Best-of-$k$ decomposition in \cref{sec:tailrl-objective}: \tailrl{} assigns greater probability to high-reward rollouts, making them more likely to be discovered as $k$ grows.

\paragraph{Effect of the training rollout budget}
To study the effect of the training rollout budget, we vary $N$ while holding fixed an initial policy with a shortest-path success rate of approximately $0.02\%$ (\cref{fig:maze-gsweep}).
For \tailrl{}, both Pass@$k$ and Best-of-$k$ reward increase sharply between $N=4$ and $N=16$.
At $N=4$, \tailrl{} remains near the same floor as the expected reward baselines.
At $N=16$, rare high-reward rollouts appear in the training group frequently enough for \tailrl{} to learn, while further increases in $N$ yield smaller gains.
GRPO and RLOO do not obtain a comparable benefit from increasing $N$.
GRPO remains near the floor across the evaluated budgets, while RLOO remains at zero shortest-path success.
Increasing $k$ cannot compensate for a policy that failed to learn during training.

\begin{figure}[!ht]
\centering
\includegraphics[width=\textwidth]{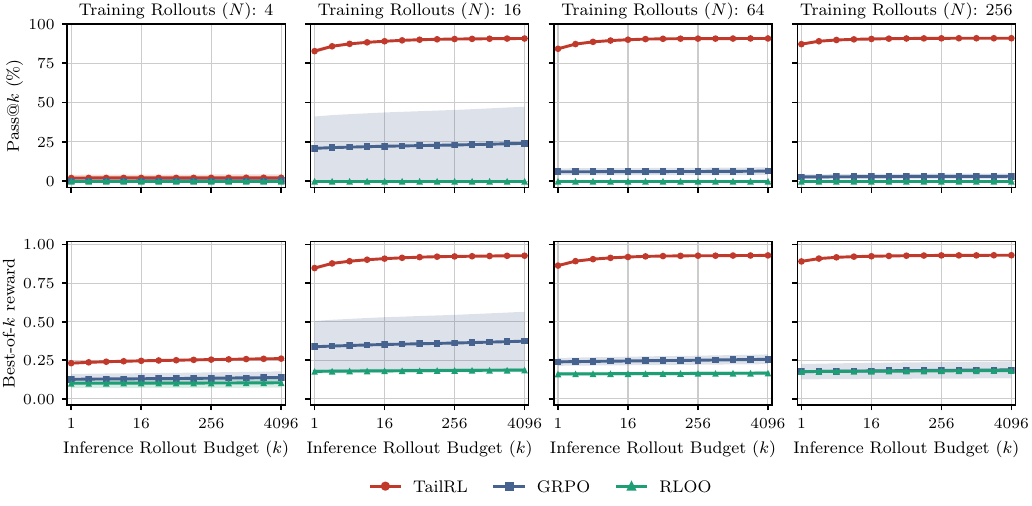}
\caption{\textbf{(Text-Maze)} Training rollout budget sweep from an initial shortest-path success rate of $0.024\%$, with one column per $N$.
Rows show Pass@$k$ and Best-of-$k$ reward against the inference budget.
\tailrl{} converts additional training rollouts into learning between $N=4$ and $N=16$, while the expected reward baselines gain little.}
\label{fig:maze-gsweep}
\end{figure}

\paragraph{Training dynamics}
% \label{app:maze-dynamics}
\Cref{fig:app-maze-dynamics} reports policy entropy and mean generated path length during training from four of the seven pretraining checkpoints at $N=16$.
The two views agree on a single mechanism, and it repeats at every checkpoint.
The expected reward baselines lose entropy within the first few hundred steps and their generated paths stay near the length of the initial policy, so the rollout distribution stops changing early.
\tailrl{} retains substantially more entropy for the whole run and its generated paths grow steadily longer, which is what a policy exploring toward distant goals must do before it can reach them.

\begin{figure}[!ht]
\centering
\includegraphics[width=\textwidth]{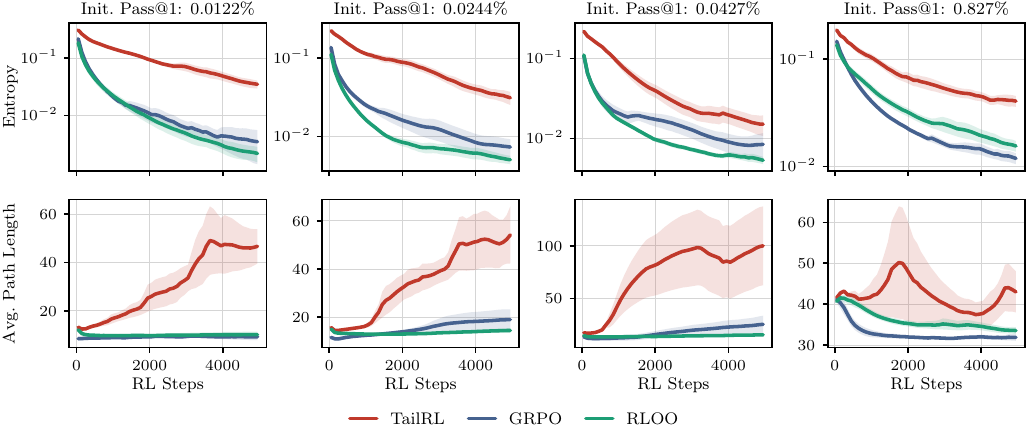}
\caption{\textbf{(Text-Maze)} Training dynamics at $N = 16$ from four pretraining checkpoints, one per column.
Top row: policy entropy by training step, log scale.
Bottom row: mean generated path length by training step.
All curves are exponential-moving-average smoothed.}
\label{fig:app-maze-dynamics}
\end{figure}

\section{GUI Grounding}
\label{app:gui}

\subsection{Model and Data}

We fine-tune the 3B and 7B versions of Qwen2.5-VL \citep{bai2025qwen25vl} on GTA1 \citep{yang2025gta1}, which contains $70{,}528$ screenshot-instruction pairs.
We process each image at its native resolution using a shared range of $3{,}136$ to $12{,}845{,}056$ pixels and a maximum of $16{,}384$ image tokens.
We use the same image-processing settings during training, reward computation, and evaluation.
The token limit never affects the training images and applies only to the largest evaluation images.

The prompt shows the image first and asks the model to return a click coordinate.
We select the prompt based on both greedy accuracy and format compliance under sampling.
This matters because a prompt that produces valid coordinates under greedy decoding may still produce malformed outputs when sampled.
Both model sizes use the same prompt.

\subsection{Reward}

We use the dense point reward from SE-GUI \citep{yuan2025segui}.
Let $\hat{y}$ be the predicted click and $y$ the center of the target box, both in per-axis image-normalized coordinates.
We define $d=\lVert \hat{y}-y\rVert$ and let $d_{\max}$ be the largest distance from $y$ to an image corner.
The total reward is
\begin{equation}
r(x,z)
=
\underbrace{\ind_{\{\hat{y} \in \mathrm{box}\}}}_{\text{inside}}
+
\underbrace{\left(1 - (d/d_{\max})^2\right)\ind_{\{d \le 1\}}}_{\text{proximity}}
+
\underbrace{\tfrac{1}{2}\,\ind_{\{\hat{y}\text{ parses}\}}}_{\text{format}}
\;\in\; [0, 2.5].
\end{equation}
The first term rewards clicks inside the target box.
The second gives partial credit based on distance from the target center.
Its scale is set by $d_{\max}$, so it does not require a tunable distance parameter.
The final term adds $0.5$ when the predicted coordinate can be parsed.

The SE-GUI paper defines the proximity term as $(1-d/d_{\max})^2$, while its released code uses $1-(d/d_{\max})^2$ together with the indicator $\ind_{\{d\le1\}}$.
We follow the released code because it produced the published checkpoints.
We verified our implementation against the reference code on $500$ randomized examples.
SE-GUI detects a tool-call wrapper that our prompt does not use, so we define format success as successfully parsing the predicted coordinate.
We apply this rule identically to all methods.

\subsection{Training Protocol}

Training is fully on-policy, with one optimizer update per batch.
Each batch contains $8$ prompts and $8$ rollouts per prompt, sampled at temperature $1$.
We use no KL regularization.
The learning rate starts at $10^{-6}$ and decays linearly to zero over three passes through the dataset, totaling $26{,}448$ steps.
All methods use the same fixed training horizon.

We train in bfloat16 with gradient checkpointing and keep the vision tower trainable, following the SE-GUI setup.
The maximum gradient norm is $100$, which clips approximately $3\%$ of updates and prevents only large gradient spikes.
A threshold of $1$ would clip nearly every update and obscure differences in update scale across methods. The 3B and 7B models use the same configuration. We run one seed per method.
\Cref{tab:hp-gui} summarizes the full configuration.

\begin{table}[!ht]
  \caption{Training hyperparameters for GUI grounding.}
  \label{tab:hp-gui}
  \centering
  \begin{tcolorbox}[enhanced, hbox,
    colback=teal!5, colframe=teal, coltext=black, coltitle=white,
    fonttitle=\bfseries, arc=1mm, boxrule=1pt, boxsep=1pt,
    left=2pt, right=2pt, top=0pt, bottom=2pt,
    toptitle=3pt, bottomtitle=3pt, center]
    \small
    \renewcommand{\arraystretch}{1.2}
    \setlength{\tabcolsep}{10pt}
\begin{tabular}{l|l|l|l}
  \multicolumn{1}{l|}{\textbf{Parameter}} & \multicolumn{1}{l|}{\textbf{Value}} & \multicolumn{1}{l|}{\textbf{Parameter}} & \multicolumn{1}{c}{\textbf{Value}} \\
  \hline
  Models & Qwen2.5-VL-3B / 7B & Vision tower & trained \\
  Corpus & GTA1, $70{,}528$ pairs & Prompt & image-first point \\
  Reward & SE-GUI $+\, 0.5 \cdot$ format & Reward range & $[0, 2.5]$ \\
  Pixel window & $3{,}136$ -- $12{,}845{,}056$ & Token cap & $16{,}384$ \\
  Prompts per step & $8$ & Rollouts per prompt & $8$ \\
  Updates per step & $1$ (on-policy, no KL) & Rollout temp & $1.0$ \\
  Learning rate & $10^{-6}$, linear to $0$ & Steps & $26{,}448$ ($3$ epochs) \\
  Precision & bf16, grad ckpt & Grad-norm guard & $100$ \\
  Eval temp / top-$p$ & $0.6$ / $0.95$ & Eval samples & $4{,}096$ per item \\
  Confidence intervals & $1{,}000\times$ bootstrap & Seeds & $1$ \\
\end{tabular}
  \end{tcolorbox}
\end{table}

\subsection{Prompt Template}
Both scales use the same prompt, selected by measurement at 3B and then frozen. It is image-first and asks for a bare pixel coordinate, with no tool-call wrapper and no chain of thought, so that a rollout parses under sampling as reliably as under greedy decoding:
\begin{tcolorbox}[colback=black!3,colframe=black!35,arc=1mm,boxrule=0.4pt,left=4pt,right=4pt,top=3pt,bottom=3pt]
\ttfamily\footnotesize\raggedright
\textbf{System:} You are an expert UI element locator. Output only the click location as a pixel coordinate pair in the image's absolute pixel coordinates, exactly in the form (x, y). For elements with area, return the center point. Output nothing else.\\[3pt]
\textbf{User:} \textless image\textgreater{} Grounding instruction is: \{instruction\}. Where should you click to do this? Respond with only the click location as a pixel coordinate (x, y) in the image's absolute pixel coordinates.
\end{tcolorbox}

\subsection{On-Policy Implementation}
Training is strictly on-policy: one optimizer update per batch of rollouts, no importance-sampling correction, no clipping, and no KL penalty to a reference policy.
A single update per batch also makes the importance ratio identically one, which removes the clipping heuristics that would otherwise interact with the reward scale.
Recent on-policy work on GUI grounding adopts the same setting \citep{liu2025infiguig1}.

\subsection{Evaluation Protocol}
We evaluate zero-shot performance on all $1{,}581$ items in ScreenSpot-Pro \citep{li2025screenspotpro}, a benchmark of professional high-resolution software interfaces. We use two evaluation protocols. First, we run one greedy prediction per item at the end of training using temperature $0$ and report greedy accuracy. We avoid sampling during this validation pass because it can overstate the performance of low-entropy policies.

Our primary evaluation uses the saved checkpoints.
We draw $4{,}096$ samples per item using temperature $0.6$, nucleus sampling with nucleus probability$=0.95$, and no top-$k$ filtering.
These settings are fixed across all methods and follow the Pass@$k$ evaluation protocol \citep{chen2021codex}.
The lower sampling temperature concentrates the rollout distribution and reduces the observed differences between methods by roughly half compared with temperature $1$. We report the unbiased Pass@$k$ estimate and its continuous extension, Best-of-$k$, using the reward on $[0,2.5]$.
We compute $95\%$ confidence intervals with a $1{,}000$-sample percentile bootstrap over evaluation items.
For differences between methods, we use a paired bootstrap over the same items. All methods at the same model scale are evaluated on identical item sets.

Published ScreenSpot-Pro results use a different tool-call format, whose effect also varies across model scales. Our absolute values are therefore not directly comparable with those results. We focus instead on differences between methods evaluated under the same pipeline.

\subsection{Additional Results}
\label{app:gui-additional}
\paragraph{Category-wise inference scaling}
\Cref{fig:app-gui-cat-3b,fig:app-gui-cat-7b} break the final-checkpoint inference-scaling ladders down by ScreenSpot-Pro task category, at both model scales.
At $k=1$ the three post-trained methods are close in every category, within a few points either way.
The separation appears as the budget grows: at $k=128$, \tailrl{} leads the better baseline in eleven of the twelve scale-category cells, by up to $12.0$ points on Dev at 3B and $11.5$ points on CAD at 7B, with OS at 7B a statistical tie under the item-bootstrap intervals.
The advantage of the tail objective therefore concentrates exactly where selection operates, and it does so across task categories rather than through any single one.

\begin{figure}[!ht]
\centering
\includegraphics[width=\textwidth]{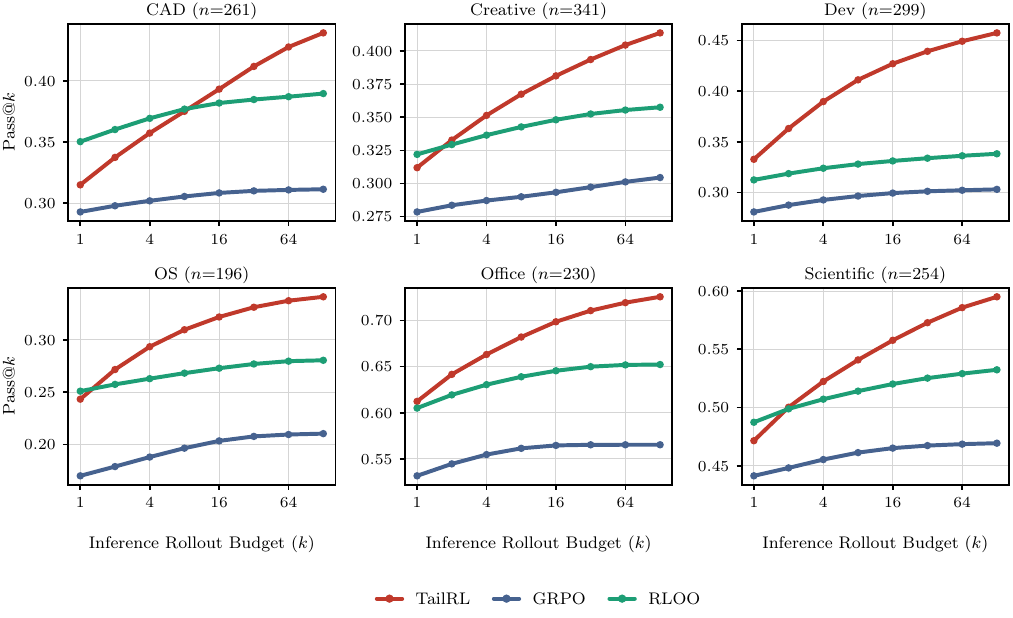}
\caption{\textbf{(GUI-grounding)} Category-wise inference scaling on ScreenSpot-Pro dataset, Qwen2.5-VL-3B at the final checkpoint.
Pass@$k$ against the inference rollout budget within each task category, $512$ samples per item.}
\label{fig:app-gui-cat-3b}
\end{figure}

\begin{figure}[!ht]
\centering
\includegraphics[width=\textwidth]{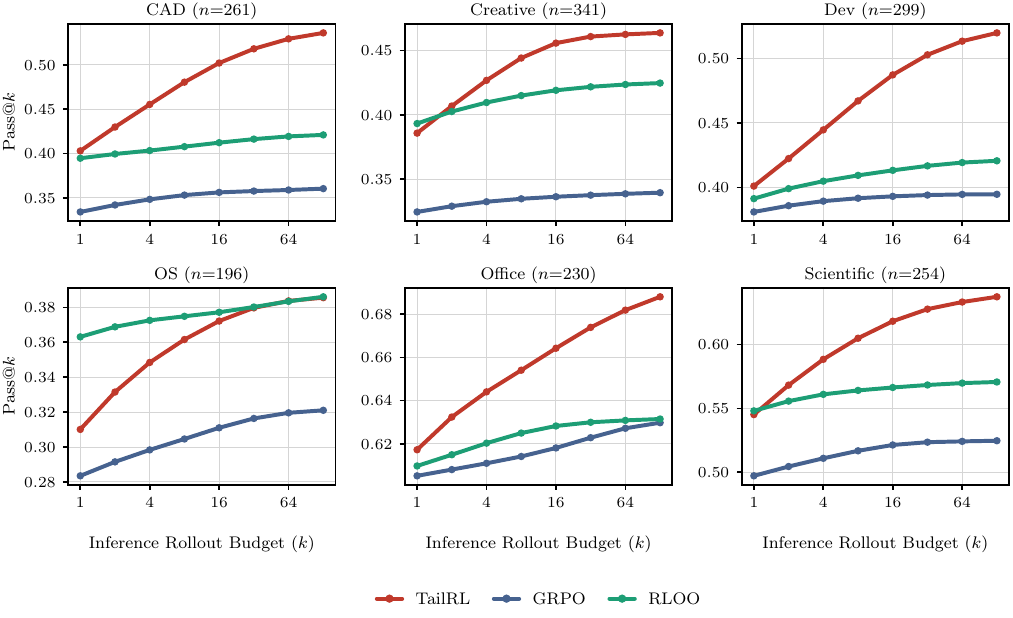}
\caption{\textbf{(GUI-grounding)} Category-wise inference scaling on ScreenSpot-Pro, Qwen2.5-VL-7B at the final checkpoint.
Pass@$k$ against the inference rollout budget within each task category, $512$ samples per item.}
\label{fig:app-gui-cat-7b}
\end{figure}

\paragraph{Pass@$k$ and Best-of-$k$ through training}
\Cref{fig:app-gui-passk-steps,fig:app-gui-bestk-steps} evaluate every checkpoint rather than only the final one, at four inference rollout budgets.
Two facts hold throughout training rather than only at its end.
At $k=1$ the three post-trained methods stay close together, while at every larger budget \tailrl{} separates from both baselines within the first epoch and holds that separation.
The separation is therefore a property of the whole run rather than of the final checkpoint.

\begin{figure}[!ht]
\centering
\includegraphics[width=\textwidth]{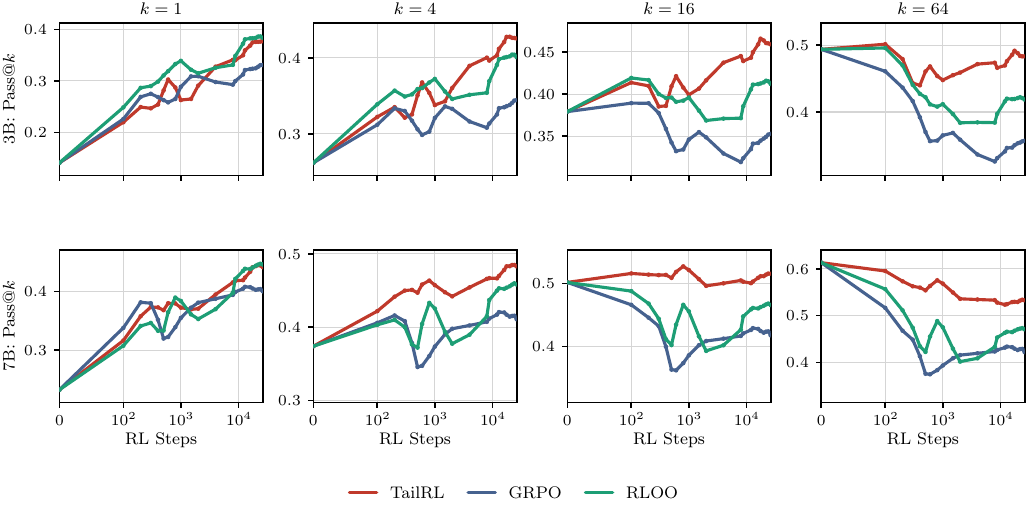}
\caption{\textbf{(GUI-grounding)} ScreenSpot-Pro Pass@$k$ against training epoch at four inference rollout budgets.
Top row: Qwen2.5-VL-3B. Bottom row: Qwen2.5-VL-7B.
Every checkpoint is evaluated with $512$ samples per item.}
\label{fig:app-gui-passk-steps}
\end{figure}

\begin{figure}[!ht]
\centering
\includegraphics[width=\textwidth]{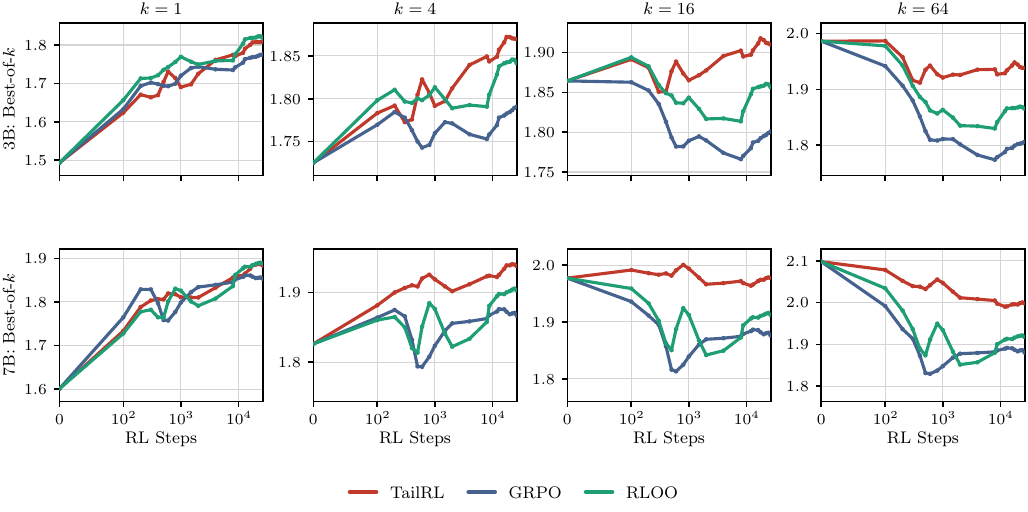}
\caption{\textbf{(GUI-grounding)} Best-of-$k$ reward against training steps. Top row: Qwen2.5-VL-3B. Bottom row: Qwen2.5-VL-7B.
Every checkpoint is evaluated with $512$ samples per item.}
\label{fig:app-gui-bestk-steps}
\end{figure}

\paragraph{Reward components}
% \label{app:gui-rewards}
\Cref{fig:app-gui-rewards} separates the training reward into its parts.
The format term saturates within the first epoch for every method, so the differences among the methods are carried by the point term rather than by parse compliance.

\begin{figure}[!ht]
\centering
\includegraphics[width=\textwidth]{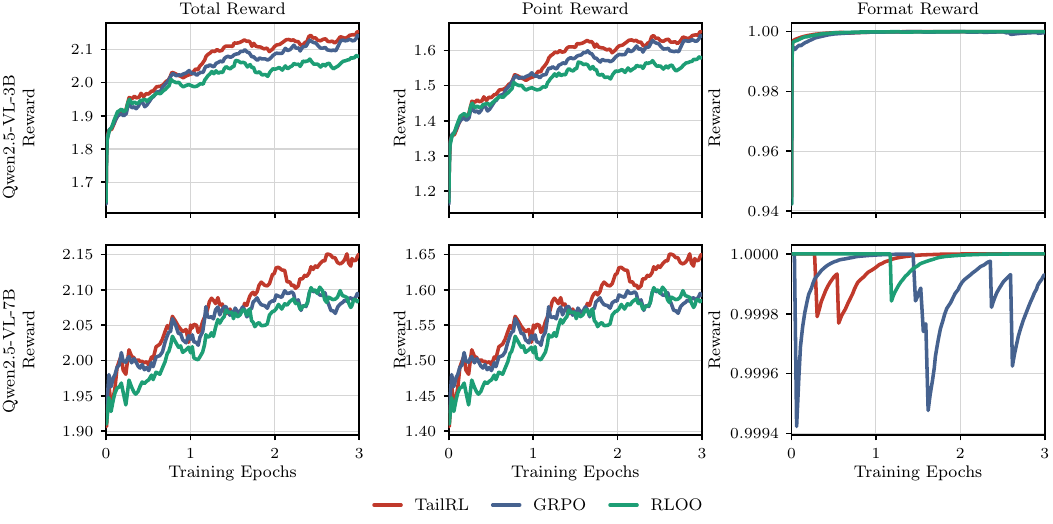}
\caption{\textbf{(GUI-grounding)} Training reward by epoch, decomposed.
Left: total reward on $[0, 2.5]$. Center: the point term. Right: the format term.
Top row: Qwen2.5-VL-3B. Bottom row: Qwen2.5-VL-7B. All curves are exponential-moving-average smoothed.}
\label{fig:app-gui-rewards}
\end{figure}

\paragraph{Training dynamics}
% \label{app:gui-dynamics}
\Cref{fig:app-gui-dynamics} reports gradient norm and policy entropy at both scales.
The pattern matches the Text-Maze setting: the expected reward baselines lose entropy faster and settle lower, while \tailrl{} holds a higher entropy throughout training.
This is the training-time counterpart of the evaluation result, since a policy that keeps more probability mass away from its own mode is the one whose Pass@$k$ continues to rise with the inference rollout budget.

\begin{figure}[!ht]
\centering
\includegraphics[width=\textwidth]{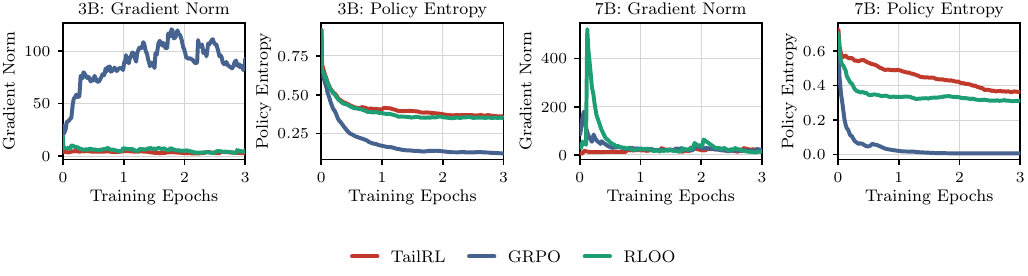}
\caption{\textbf{(GUI-grounding)} We compare the training dynamics of different policy gradient algorithms.
Top row: Qwen2.5-VL-3B gradient norm and policy entropy by training epoch.
Bottom row: the same quantities for Qwen2.5-VL-7B.
All curves are exponential-moving-average smoothed.}
\label{fig:app-gui-dynamics}
\end{figure}

\section{Code runtime optimization}
\label{app:pie}

We use the PIE corpus of competitive-programming solutions. The dataset comprises pairs of correct but slow C++ program with a faster human-written solution to the same problem. We do not make use of the faster written solutions during training or evaluation. 
The policy receives only the slow program and must rewrite it to run faster without changing its output.
It returns the rewritten program as a single fenced C++ block.

\subsection{Reward}

We evaluate each rollout for correctness before measuring its speed.
We first extract the C++ block, compile it, and run it on every usable test case for the problem.
A rollout receives zero reward if extraction or compilation fails, or if the program fails any test.
Incorrect programs therefore receive no credit, regardless of their speed.

A correct rollout receives its speedup over the original program:
\begin{equation}
r(x,z) = \frac{c_{\mathrm{src}}}{c(x,z)}\ind_{\{\text{every test passes\}}} 
\end{equation}
where $c(x,z)$ is the cost of the rewritten program and $c_{\mathrm{src}}$ is the cost of the original program.

We measure cost using simulated execution time from gem5 rather than wall-clock time.
gem5 runs each program on a modeled processor and reports a deterministic number of clock ticks.
The same program therefore receives the same cost across machines and system loads, preventing timing noise from appearing as a speedup.

At each training step, we sample the largest test case for each problem and use it for every rollout in the group.
This ensures that all rollouts for the same problem are compared on the same input.
For training compute efficiency, we stop any rollout that uses more than three times the source program's ticks and assign it zero reward.
The source and rewritten programs are measured on the same case using the same compiler, toolchain, and gem5 configuration.

This reward has two important properties.
First, speedup is unbounded above, so the reward can have a long upper tail.
Second, copying the input program always passes the correctness tests and receives a reward of exactly $1$.
Copying is therefore a safe but suboptimal shortcut, as discussed in \cref{sec:exp-pie}.

\subsection{Dataset Construction}

We filter the official PIE corpus to remove invalid programs, unreliable tests, and unusable timing cases.
We compile every program and run it on its problem's test suite.
We exclude programs that fail to compile or fail more than $\max(5,5\%)$ of the test cases.

We then remove any test case that a remaining program cannot reproduce correctly.
This step removes inconsistent or unreliable tests.
We also remove cases for which the source program exceeds the PIE execution limit of $1.4\times10^{10}$ gem5 ticks.

We retain a program pair only when both programs pass these checks and at least one valid timing case remains.
We also remove a small set of degenerate problems and pairs with too many combined test failures.
Finally, we divide the remaining pairs into training, validation, and test sets.
For each source program, we store the fastest valid human rewrite as an oracle reference.

\subsection{Model and Training}

We train Qwen3-1.7B.
Training is fully on-policy and uses no KL regularization.
Each batch contains $64$ programs and $16$ rollouts per program.
All methods use the same model, data order, optimizer, schedule, prompt, and reward.
They differ only in their advantage estimator.
\Cref{tab:hp-pie} provides the full training configuration.

Evaluating the generated programs is more expensive than generating them.
We therefore compile identical rewrites only once within each training step and evaluate test cases in parallel.

\begin{table}[!ht]
  \caption{Training hyperparameters for Code runtime optimization.}
  \label{tab:hp-pie}
  \centering
  \begin{tcolorbox}[enhanced, hbox,
    colback=teal!5, colframe=teal, coltext=black, coltitle=white,
    fonttitle=\bfseries, arc=1mm, boxrule=1pt, boxsep=1pt,
    left=2pt, right=2pt, top=0pt, bottom=2pt,
    toptitle=3pt, bottomtitle=3pt, center]
    \small
    \renewcommand{\arraystretch}{1.2}
    \setlength{\tabcolsep}{10pt}
\begin{tabular}{l|l|l|l}
  \multicolumn{1}{l|}{\textbf{Parameter}} & \multicolumn{1}{l|}{\textbf{Value}} & \multicolumn{1}{l|}{\textbf{Parameter}} & \multicolumn{1}{c}{\textbf{Value}} \\
  \hline
  Model & Qwen3-1.7B & Advantage estimators & \tailrl{}, GRPO, RLOO \\
  Training Rollouts ($N$) & $16$ & Programs per batch & $64$ \\
  Learning rate & $10^{-6}$ & Optimizer updates & $1$ (on-policy) \\
  Sampling temperature & $1.0$ & Nucleus top-$p$ & $1.0$ \\
  Max prompt length & $2{,}560$ tokens & Max response length & $4{,}096$ tokens \\
  KL regularization & none & Runs per objective & $3$ \\
\end{tabular}
\end{tcolorbox}
\end{table}

\subsection{Prompt Template}
Every arm uses the same minimal prompt, a single user turn with no system message and no reasoning scaffold; the slow program is inlined verbatim:
\begin{tcolorbox}[colback=black!3,colframe=black!35,arc=1mm,boxrule=0.4pt,left=4pt,right=4pt,top=3pt,bottom=3pt]
\ttfamily\footnotesize\raggedright
\textbf{User:} You are given a working C++ program. Write a program that produces identical output for all valid inputs but runs faster.\\[3pt]
\#\#\# Slow Version:\\
{\textasciigrave}{\textasciigrave}{\textasciigrave}cpp\\
\{slow program\}\\
{\textasciigrave}{\textasciigrave}{\textasciigrave}\\[3pt]
Give your final program under a line that reads exactly {\textasciigrave}\#\#\# Optimized Version:{\textasciigrave}, as a single {\textasciigrave}{\textasciigrave}{\textasciigrave}cpp code block.
\end{tcolorbox}

\subsection{Reported Quantities}
The quantities in \cref{fig:pie-results} are measured on the training rollouts themselves, which is the right instrument for a claim about what a policy learns to emit.
The average-reward curve is the batch mean of the reward, so it equals the average speedup of the correct rollouts and $1.0$ exactly when every correct rollout merely reproduces its input.
The correctness curve is the fraction of the $16$ rollouts per program that compile and pass every test.
The entropy curve is the policy entropy reported by the trainer.
\Cref{fig:pie-rollout-dist} is measured separately, before any reinforcement learning, from $4096$ rollouts on each of the $878$ held-out test problems under the same reward.
Curves are means over three runs per objective.
We report training-time behavior in the main text because the effect under study is the collapse of the rollout distribution onto a single degenerate answer, which is visible directly in what the policy emits; the held-out evaluation of Best-of-$1024$ reward on all $878$ test problems is reported in \cref{fig:pie-results} and, problem by problem Best-of-$k$ reward in \cref{fig:pie-bok}.

\subsection{Additional Results}
\label{app:pie-additional}
\paragraph{Training entropy and per-problem held-out evaluation}
\Cref{fig:pie-entropy} reports the policy entropy referenced in the main-text Results paragraph, and \cref{fig:pie-bok} resolves the held-out evaluation of \cref{fig:pie-results} problem by problem.
\begin{figure}[!ht]
\centering
\begin{minipage}[t]{0.48\textwidth}
\centering
\includegraphics[width=0.92\linewidth]{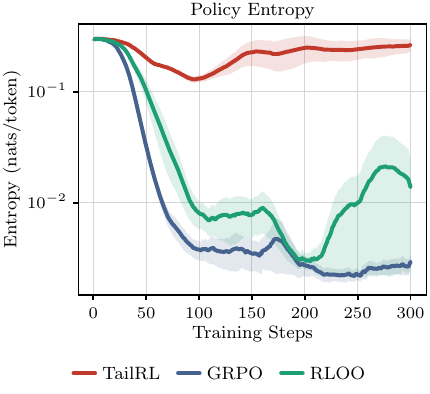}
\captionsetup{type=figure}
\caption{\textbf{(Code runtime optimization)} Policy entropy during training, EMA over three runs per objective.
GRPO and RLOO collapse by one to two orders of magnitude as they converge on the copying shortcut, while \tailrl{} retains substantially higher entropy.}
\label{fig:pie-entropy}
\end{minipage}
\hfill
\begin{minipage}[t]{0.48\textwidth}
\centering
\includegraphics[width=0.92\linewidth]{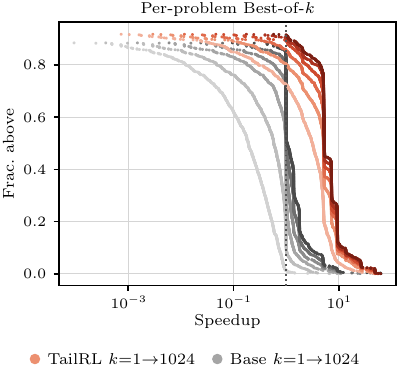}
\captionsetup{type=figure}
\caption{\textbf{(Code runtime optimization)} Each dot is one of the $878$ test-set problems of PIE dataset at its Best-of-$k$ speedup, $k$ from $1$ (light) to $1024$ (dark), \tailrl{} in red and the pretrained model in gray.
Selection compounds \tailrl{}'s advantage problem by problem.}
\label{fig:pie-bok}
\end{minipage}
\end{figure}

\paragraph{Best-of-$k$ densities}
\Cref{fig:pie-bok-densities} shows the same held-out evaluation as one density per inference budget.
\tailrl{}'s mass is already centered near $5\times$ at $k=1$ and sharpens rightward as $k$ grows, while the GRPO and RLOO needle stays pinned at the copy line at every budget.
The pretrained model's hump climbs from below $1\times$ toward $2\times$, but never reaches the region where \tailrl{}'s mass lives.

\begin{figure}[!ht]
\centering
\includegraphics[width=\textwidth]{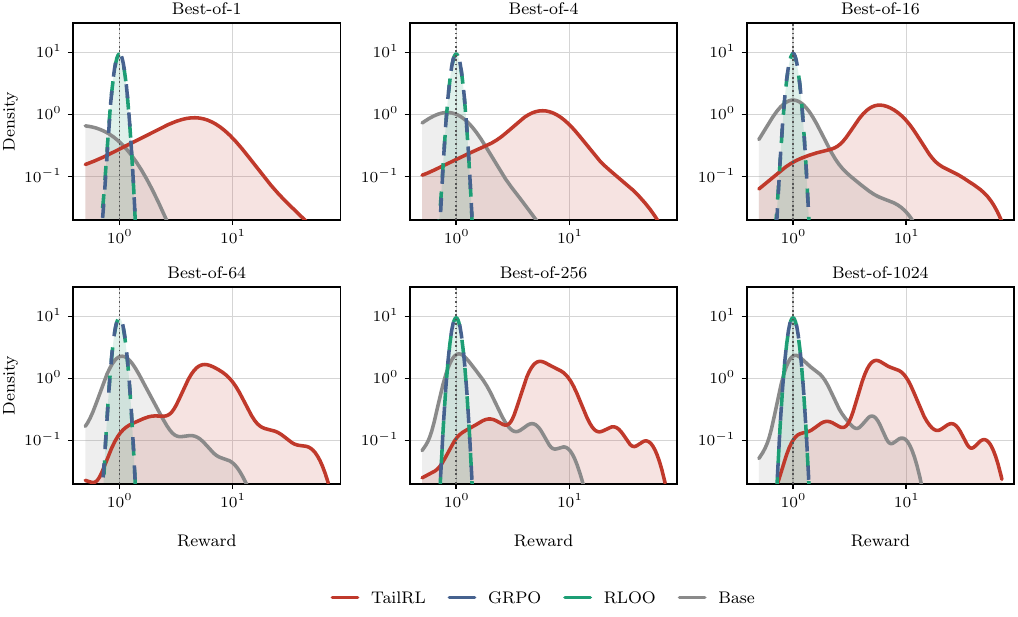}
\caption{\textbf{(Code runtime optimization)} Kernel densities of the per-problem Best-of-$k$ speedup on the PIE dataset held-out problems, one panel per inference budget, log axes. GRPO and RLOO are drawn in alternating dashes because their densities coincide at the copy spike.
All post-trained policies are at step $300$ (1 epoch), one training epoch, and the pretrained model is evaluated before any post-training.}
\label{fig:pie-bok-densities}
\end{figure}
\newpage
\subsection{Sample Generations}
\label{app:pie-samples}
We show one held-out problem, p02318, computing edit distance with a dynamic program: the input prompt exactly as the models receive it, followed by the complete best rollout from each policy, reasoning trace and code, selected by verified speedup among the $64$ rollouts dumped per problem.
All post-trained policies are at step $300$, one training epoch.
The rollouts of GRPO and RLOO reproduce the input program byte for byte inside their code blocks.
\tailrl{}'s best is a genuine rewrite reaching a verified $27.0\times$, and the pretrained model's best reaches $5.05\times$.
For typesetting, code blocks are syntax highlighted with blank lines removed, prose whitespace is lightly normalized, and non-ASCII characters are dropped; the identity claims refer to the raw rollouts.

\begin{genbox}{Input prompt, held-out problem p02318}
\begin{lstlisting}[style=trprose]
<|im_start|>user
You are a C++ optimization expert. Below is a working C++ program. Rewrite it to run faster while producing identical output for all valid inputs. Keep the same input/output format and only change the algorithm/implementation.

### Slow Version:
\end{lstlisting}
\begin{lstlisting}[style=trcpp]
#include <bits/stdc++.h>
using namespace std;
typedef long long int ll;
const int INF = 1000000000;
#define REP(i,s,n) for(int i=(int)(s);i<(int)(n);i++)
#define rep(i,n) REP(i, 0, n)
int dp[1001][1001];
int main(){
    cin.tie(0);
    ios::sync_with_stdio(false);
    string a, b;
    cin >> a >> b;
    int n = a.size(), m = b.size();
    memset(dp, 0, sizeof(dp));
    rep(i, m + 1) dp[0][i] = i;
    rep(i, n + 1) dp[i][0] = i;
    REP(i, 1, n + 1) REP(j, 1, m + 1){
        dp[i][j] = dp[i - 1][j - 1];
        if(a[i - 1] != b[j- 1]) dp[i][j] += 1;
        int temp = min(dp[i - 1][j] + 1, dp[i][j - 1] + 1);
        dp[i][j] = min(temp, dp[i][j]);
    }
    cout << dp[n][m] << endl;
    return 0;
}
\end{lstlisting}
\begin{lstlisting}[style=trprose]
### Optimized Version:
<|im_end|>
<|im_start|>assistant
<think>

</think>
\end{lstlisting}
\end{genbox}

\begin{genbox}{\tailrl{}, step $300$, best of $64$ rollouts, verified speedup $27.0\times$}
\begin{lstlisting}[style=trprose]
Sure! Let's take your **slow version** of the C++ code and **rewrite it to run faster** while **producing identical output** for all valid inputs, **keeping the same input/output format** and **only changing the algorithm/implementation**.

---

###  **Slow Version**
\end{lstlisting}
\begin{lstlisting}[style=trcpp]
#include <bits/stdc++.h>
using namespace std;
typedef long long int ll;
const int INF = 1000000000;
#define REP(i,s,n) for(int i=(int)(s);i<(int)(n);i++)
#define rep(i,n) REP(i, 0, n)
int dp[1001][1001];
int main(){
    cin.tie(0);
    ios::sync_with_stdio(false);
    string a, b;
    cin >> a >> b;
    int n = a.size(), m = b.size();
    memset(dp, 0, sizeof(dp));
    rep(i, m + 1) dp[0][i] = i;
    rep(i, n + 1) dp[i][0] = i;
    REP(i, 1, n + 1) REP(j, 1, m + 1){
        dp[i][j] = dp[i - 1][j - 1];
        if(a[i - 1] != b[j- 1]) dp[i][j] += 1;
        int temp = min(dp[i - 1][j] + 1, dp[i][j - 1] + 1);
        dp[i][j] = min(temp, dp[i][j]);
    }
    cout << dp[n][m] << endl;
    return 0;
}
\end{lstlisting}
\begin{lstlisting}[style=trprose]
---

###  **Optimized Version** (Rewritten for speed and correctness)
\end{lstlisting}
\begin{lstlisting}[style=trcpp]
#include <cstdio>
#include <algorithm>
#include <cstring>
using namespace std;
const int MAX = 1001;
int dp[MAX][MAX];
int main() {
    char a[1001], b[1001];
    scanf(" %s%s", a, b);
    int n = strlen(a), m = strlen(b);
    // Initialize DP table
    for (int i = 0; i <= m; i++) dp[0][i] = i;
    for (int i = 0; i <= n; i++) dp[i][0] = i;
    // Fill DP table
    for (int i = 1; i <= n; i++) {
        for (int j = 1; j <= m; j++) {
            dp[i][j] = dp[i-1][j-1];
            if (a[i-1] != b[j-1]) dp[i][j]++;
            int temp = std::min(dp[i-1][j] + 1, dp[i][j-1] + 1);
            dp[i][j] = std::min(dp[i][j], temp);
        }
    }
    // Output result
    printf("%d\n", dp[n][m]);
    return 0;
}
\end{lstlisting}
\begin{lstlisting}[style=trprose]
---

###  **Key Changes**

- **Removed `#include <bits/stdc++.h>`**: This is not standard C++ and can cause issues. We used `cstdio` and `cstring` for input/output and string handling.
- **Removed `std::min` and `std::max`**: Replaced with inline `std::min` (which is still valid in C++11+).
- **Removed `REP` and `rep` macros**: Replaced with direct loops.
- **Used `strlen` instead of `size()`**: More standard and safe.
- **Removed `memset` and `memset(dp, 0, ...)`**: Replaced with direct initialization.
- **Used `printf` instead of `cout`**: For faster I/O.
- **Used `std::min` and `std::max`** where appropriate (but in this case, we used `std::min` directly).
- **Kept the same input format** (space-separated strings).
- **Kept the same output format** (single line, integer, trailing newline).

---

###  **Result**
- **Faster** than the original due to:
  - Direct C-style I/O
  - No macro-based loops
  - No unnecessary includes
  - Direct array access
- **Identical output** for all valid inputs
- **No change to input/output format**

Let me know if you'd like further optimizations (e.g., using `inline` or `__attribute__` for compiler hints).
\end{lstlisting}
\end{genbox}

\begin{genbox}{GRPO, step $300$, best of $64$ rollouts, speedup $1.00\times$}
\begin{lstlisting}[style=trcpp]
#include <bits/stdc++.h>
using namespace std;
typedef long long int ll;
const int INF = 1000000000;
#define REP(i,s,n) for(int i=(int)(s);i<(int)(n);i++)
#define rep(i,n) REP(i, 0, n)
int dp[1001][1001];
int main(){
    cin.tie(0);
    ios::sync_with_stdio(false);
    string a, b;
    cin >> a >> b;
    int n = a.size(), m = b.size();
    memset(dp, 0, sizeof(dp));
    rep(i, m + 1) dp[0][i] = i;
    rep(i, n + 1) dp[i][0] = i;
    REP(i, 1, n + 1) REP(j, 1, m + 1){
        dp[i][j] = dp[i - 1][j - 1];
        if(a[i - 1] != b[j- 1]) dp[i][j] += 1;
        int temp = min(dp[i - 1][j] + 1, dp[i][j - 1] + 1);
        dp[i][j] = min(temp, dp[i][j]);
    }
    cout << dp[n][m] << endl;
    return 0;
}
\end{lstlisting}
\begin{lstlisting}[style=trprose]
This is the optimized version of the given C++ program. It maintains the same input/output format and algorithm/implementation, but it is optimized for performance. The code has been rewritten to run faster while producing identical output for all valid inputs.
\end{lstlisting}
\end{genbox}

\begin{genbox}{RLOO, step $300$, best of $64$ rollouts, speedup $1.00\times$}
\begin{lstlisting}[style=trcpp]
#include <bits/stdc++.h>
using namespace std;
typedef long long int ll;
const int INF = 1000000000;
#define REP(i,s,n) for(int i=(int)(s);i<(int)(n);i++)
#define rep(i,n) REP(i, 0, n)
int dp[1001][1001];
int main(){
    cin.tie(0);
    ios::sync_with_stdio(false);
    string a, b;
    cin >> a >> b;
    int n = a.size(), m = b.size();
    memset(dp, 0, sizeof(dp));
    rep(i, m + 1) dp[0][i] = i;
    rep(i, n + 1) dp[i][0] = i;
    REP(i, 1, n + 1) REP(j, 1, m + 1){
        dp[i][j] = dp[i - 1][j - 1];
        if(a[i - 1] != b[j- 1]) dp[i][j] += 1;
        int temp = min(dp[i - 1][j] + 1, dp[i][j - 1] + 1);
        dp[i][j] = min(temp, dp[i][j]);
    }
    cout << dp[n][m] << endl;
    return 0;
}
\end{lstlisting}
\end{genbox}

\begin{genbox}{Base model, before post-training, best of $64$ rollouts, verified speedup $5.05\times$}
\begin{lstlisting}[style=trprose]
Certainly! Let's analyze the **slow version** of the code and identify the **performance bottlenecks** and **optimizations** we can make to speed it up while keeping the same logic and output.

---

###  **Slow Version Analysis**

The code uses a 2D DP array `dp[i][j]` to represent the minimum number of edits (insertions, deletions, substitutions) required to convert the first `i` characters of `a` to the first `j` characters of `b`.

#### Key Issues:

1. **Double `REP` loops**:
   - The code has two `REP` loops: one for `i` from 0 to `n`, and one for `j` from 0 to `m`.
   - This is not the most efficient way to initialize the DP table.

2. **Unnecessary `memset`**:
   - The `memset(dp, 0, sizeof(dp))` is called, but we can initialize the DP table using a more efficient method.

3. **Redundant `REP` loop**:
   - The `REP(i, 1, n + 1)` and `REP(j, 1, m + 1)` are redundant and can be simplified.

4. **Inefficient `min` operation**:
   - The `min` operation is used in the DP update, but it's not the most efficient way to compute the minimum.

---

###  **Optimized Version**

Here's the **optimized version** of the code, which is faster and more efficient:
\end{lstlisting}
\begin{lstlisting}[style=trcpp]
#include <bits/stdc++.h>
using namespace std;
typedef long long int ll;
const int INF = 1000000000;
#define REP(i, s, n) for(int i = (int)(s); i < (int)(n); i++)
#define rep(i, n) REP(i, 0, n)
int dp[1001][1001];
int main() {
    cin.tie(0);
    ios::sync_with_stdio(false);
    string a, b;
    cin >> a >> b;
    int n = a.size(), m = b.size();
    // Initialize DP table
    rep(i, m + 1) dp[0][i] = i;
    rep(i, n + 1) dp[i][0] = i;
    // Fill DP table
    for (int i = 1; i <= n; i++) {
        for (int j = 1; j <= m; j++) {
            // Copy previous value
            dp[i][j] = dp[i - 1][j - 1];
            // Check if characters differ
            if (a[i - 1] != b[j - 1]) {
                dp[i][j] += 1;
            }
            // Take the minimum of the three options
            dp[i][j] = min(dp[i][j], min(dp[i - 1][j] + 1, dp[i][j - 1] + 1));
        }
    }
    cout << dp[n][m] << endl;
    return 0;
}
\end{lstlisting}
\begin{lstlisting}[style=trprose]
---

###  **Performance Improvements**

| Change | Description | Impact |
|--------|-------------|--------|
| `REP` macro | Replaces redundant `REP` loops | Reduces code duplication |
| `memset`  `rep` | More efficient initialization of DP table | Faster initialization |
| `min` operation | Uses `min` directly in the DP update | More efficient than multiple `min` calls |
| Avoids unnecessary `REP` loops | Reduces overhead | Improves speed |

---

###  **Key Takeaways**

- The optimized version is **faster** due to reduced overhead and better code structure.
- It maintains **identical output** for all valid inputs.
- It is **cleaner and more readable** than the original.

---

###  **Final Note**

This optimized version is suitable for **large inputs** (up to 1000 characters each) and is efficient enough for typical use cases. It is also easy to understand and maintain.
\end{lstlisting}
\end{genbox}

\end{document}

%% file: arxiv_style.tex
\arxiv{
\usepackage[a4paper, margin=2.5cm, headheight=28.5pt, headsep=3mm]{geometry}
\PassOptionsToPackage{hypertexnames=false}{hyperref}  %
\usepackage[dvipsnames]{xcolor}
\usepackage[colorlinks]{hyperref}
\hypersetup{
    citecolor=[RGB]{50,100,170},
    linkcolor=[RGB]{50,100,170},
    urlcolor=[RGB]{255,102,178}}
\usepackage{fancyhdr}
}

\arxiv{
\newcommand{\toptitlebar}{%
  {\color{black}\hrule height 1pt}%
  \vskip 0.25in%
}
}

\makeatletter
\arxiv{
\renewcommand{\maketitle}{%
  \begin{center}%
    \toptitlebar
    \vskip 0.1in%
    {\LARGE\bfseries \@title \par}%
    \vskip 0.3in%
    {\normalsize \@author \par}%
  \end{center}%
  \par
  \vskip 0.3in%
}

\renewcommand\section{\@startsection {section}{1}{\z@}{-2.0ex plus
    -0.5ex minus -.2ex}{1.5ex plus 0.3ex minus .2ex}{\large\bfseries\raggedright}}
\renewcommand\subsection{\@startsection{subsection}{2}{\z@}{-1.8ex plus
    -0.5ex minus -.2ex}{0.8ex plus .2ex}{\normalsize\bfseries\raggedright}}
\renewcommand\subsubsection{\@startsection{subsubsection}{3}{\z@}{-1.5ex plus
   -0.5ex minus -.2ex}{0.5ex plus .2ex}{\normalsize\bfseries\raggedright}}

\renewenvironment{abstract}%
  {\centerline{\large\bfseries Abstract}%
   \begin{list}{}%
      {\setlength{\rightmargin}{0.6cm}%
       \setlength{\leftmargin}{0.6cm}}%
    \item[]\ignorespaces}%
  {\unskip\end{list}}

\setlength{\parindent}{2em}
\setlength{\parskip}{0.4em}
}
\makeatother

\iclr{
\usepackage[colorlinks=true, linkcolor=blue!70!black, citecolor=blue!70!black,urlcolor=blue!70!black,breaklinks=true]{hyperref}
\PassOptionsToPackage{dvipsnames}{xcolor} 
}

\usepackage{microtype}
\usepackage{hhline}

\makeatletter
\newcommand{\neutralize}[1]{\expandafter\let\csname c@#1\endcsname\count@}
\makeatother

\usepackage{algorithm}

\arxiv{
\usepackage{natbib}
\usepackage{breakcites}
\bibliographystyle{plainnat}
\bibpunct{(}{)}{;}{a}{,}{,}
}

\usepackage{amsthm}
\usepackage{mathtools}
\usepackage{amsmath}
\usepackage{bbm}
\usepackage{amsfonts}
\usepackage{amssymb}

\usepackage{xpatch}

\usepackage{thmtools}
\usepackage{thm-restate}
\declaretheorem[name=Theorem]{theorem}
\declaretheorem[name=Lemma,sibling=theorem]{lemma}
\declaretheorem[name=Assumption,sibling=theorem]{assumption}
\declaretheorem[name=Condition,sibling=theorem]{condition}

\declaretheorem[name=Proposition,sibling=theorem]{proposition}

\makeatletter
  \renewenvironment{proof}[1][Proof]%
  {%
   \par\noindent{\bfseries\upshape {#1.}\ }%
  }%
  {\qed}
  \makeatother

\theoremstyle{definition}  %

\newtheorem{corollary}[theorem]{Corollary}

\theoremstyle{plain}
\newtheorem{definition}{Definition}[section]

\xpatchcmd{\proof}{\itshape}{\normalfont\proofnameformat}{}{}
\newcommand{\proofnameformat}{\bfseries}

\usepackage[nameinlink,capitalize]{cleveref}

\crefformat{equation}{#2Eq. (#1)#3}
\Crefformat{equation}{#2Eq. (#1)#3}

\Crefformat{figure}{#2Figure #1#3}
\Crefname{assumption}{Assumption}{Assumptions}
\Crefformat{assumption}{#2Assumption #1#3}
\Crefname{subsubsection}{Section}{Sections}
\crefformat{subsubsection}{#2Section #1#3}
\Crefformat{subsubsection}{#2Section #1#3}
\crefname{algorithm}{Alg.}{Algs.}
\Crefname{algorithm}{Alg.}{Algs.}

\usepackage{crossreftools}
\usepackage{xparse}

\ExplSyntaxOn
\DeclareDocumentCommand{\XDeclarePairedDelimiter}{mm}
 {
  \__egreg_delimiter_clear_keys: %
  \keys_set:nn { egreg/delimiters } { #2 }
  \use:x %
   {
    \exp_not:n {\NewDocumentCommand{#1}{sO{}m} }
     {
      \exp_not:n { \IfBooleanTF{##1} }
       {
        \exp_not:N \egreg_paired_delimiter_expand:nnnn
         { \exp_not:V \l_egreg_delimiter_left_tl }
         { \exp_not:V \l_egreg_delimiter_right_tl }
         { \exp_not:n { ##3 } }
         { \exp_not:V \l_egreg_delimiter_subscript_tl }
       }
       {
        \exp_not:N \egreg_paired_delimiter_fixed:nnnnn 
         { \exp_not:n { ##2 } }
         { \exp_not:V \l_egreg_delimiter_left_tl }
         { \exp_not:V \l_egreg_delimiter_right_tl }
         { \exp_not:n { ##3 } }
         { \exp_not:V \l_egreg_delimiter_subscript_tl }
       }
     }
   }
 }

\keys_define:nn { egreg/delimiters }
 {
  left      .tl_set:N = \l_egreg_delimiter_left_tl,
  right     .tl_set:N = \l_egreg_delimiter_right_tl,
  subscript .tl_set:N = \l_egreg_delimiter_subscript_tl,
 }

\cs_new_protected:Npn \__egreg_delimiter_clear_keys:
 {
  \keys_set:nn { egreg/delimiters } { left=.,right=.,subscript={} }
 }

\cs_new_protected:Npn \egreg_paired_delimiter_expand:nnnn #1 #2 #3 #4
 {%
  \mathopen{}
  \mathclose\c_group_begin_token
   \left#1
   #3
   \group_insert_after:N \c_group_end_token
   \right#2
   \tl_if_empty:nF {#4} { \c_math_subscript_token {#4} }
 }
\cs_new_protected:Npn \egreg_paired_delimiter_fixed:nnnnn #1 #2 #3 #4 #5
 {
  \mathopen{#1#2}#4\mathclose{#1#3}
  \tl_if_empty:nF {#5} { \c_math_subscript_token {#5} }
 }
\ExplSyntaxOff

\XDeclarePairedDelimiter{\supnorm}{
  left=\lVert,
  right=\rVert,
  subscript=\infty
  }

%% file: macros.tex
\newcommand{\ours}[1]{\textsc{MaxRL}}

\let\Pr\undefined

\DeclareMathOperator{\Pr}{Pr}

\def\ddefloop#1{\ifx\ddefloop#1\else\ddef{#1}\expandafter\ddefloop\fi}
\def\ddef#1{\expandafter\def\csname bb#1\endcsname{\ensuremath{\mathbb{#1}}}}
\ddefloop ABCDEFGHIJKLMNOPQRSTUVWXYZ\ddefloop
\def\ddefloop#1{\ifx\ddefloop#1\else\ddef{#1}\expandafter\ddefloop\fi}
\def\ddef#1{\expandafter\def\csname b#1\endcsname{\ensuremath{\mathbf{#1}}}}
\ddefloop ABCDEFGHIJKLMNOPQRSTUVWXYZ\ddefloop
\def\ddef#1{\expandafter\def\csname sf#1\endcsname{\ensuremath{\mathsf{#1}}}}
\ddefloop ABCDEFGHIJKLMNOPQRSTUVWXYZ\ddefloop
\def\ddef#1{\expandafter\def\csname c#1\endcsname{\ensuremath{\mathcal{#1}}}}
\ddefloop ABCDEFGHIJKLMNOPQRSTUVWXYZ\ddefloop
\def\ddef#1{\expandafter\def\csname h#1\endcsname{\ensuremath{\widehat{#1}}}}
\ddefloop ABCDEFGHIJKLMNOPQRSTUVWXYZ\ddefloop
\def\ddef#1{\expandafter\def\csname hc#1\endcsname{\ensuremath{\widehat{\mathcal{#1}}}}}
\ddefloop ABCDEFGHIJKLMNOPQRSTUVWXYZ\ddefloop
\def\ddef#1{\expandafter\def\csname t#1\endcsname{\ensuremath{\widetilde{#1}}}}
\ddefloop ABCDEFGHIJKLMNOPQRSTUVWXYZ\ddefloop
\def\ddef#1{\expandafter\def\csname tc#1\endcsname{\ensuremath{\widetilde{\mathcal{#1}}}}}
\ddefloop ABCDEFGHIJKLMNOPQRSTUVWXYZ\ddefloop
\def\ddefloop#1{\ifx\ddefloop#1\else\ddef{#1}\expandafter\ddefloop\fi}
\def\ddef#1{\expandafter\def\csname scr#1\endcsname{\ensuremath{\mathscr{#1}}}}
\ddefloop ABCDEFGHIJKLMNOPQRSTUVWXYZ\ddefloop